%% file: main.tex
\documentclass{article} 
\usepackage{iclr2023_conference_edited,times}
\input{preamble.tex}

\usepackage{hyperref}
\usepackage{url}

\title{Neural Causal Models for Counterfactual Identification and Estimation}
\author{%
  Kevin Xia \normalfont{and} \textbf{Yushu Pan} \text{and} \textbf{Elias Bareinboim}\\
  {\footnotesize Causal Artificial Intelligence Laboratory} \\
  {\footnotesize Columbia University, USA} \\
  {\footnotesize \texttt{\{kevinmxia,yushupan,eb\}@cs.columbia.edu}} \\
}

\PassOptionsToPackage{square, numbers}{natbib}

\begin{document}
\thispagestyle{empty}

\maketitle

\vspace{-0.1in}

\begin{abstract}
\input{section/00_abstract}
\end{abstract}

\input{section/01_introduction}
\input{section/01-1_preliminaries}
\input{section/02_data_structure}
\input{section/03_counterfactual_id}
\input{section/04_gan_counterfactual_estimation}
\input{section/05_experiments}

\input{section/06_conclusion}

\bibliography{references}
\bibliographystyle{iclr2023_conference}

\clearpage
\appendix
\input{section/A_proofs}

\clearpage
\input{section/B_experimental_details}

\clearpage
\input{section/C_related_works}
\clearpage
\input{section/D_examples}

\clearpage
\input{section/E_FAQ}

\end{document}

%% file: preamble.tex
\ifdefined\preamble\else\def\preamble{}

\usepackage[utf8]{inputenc} 
\usepackage[T1]{fontenc}    
\usepackage{url}            
\usepackage{booktabs}       
\usepackage{amsfonts}       
\usepackage{nicefrac}       
\usepackage{microtype}      
\usepackage{lipsum}
\usepackage{bbm}
\usepackage{enumitem}
\usepackage{optidef}
\usepackage{subcaption}
\usepackage{bm}
\usepackage{upgreek}
\usepackage{multicol}
\usepackage{wrapfig}
\usepackage{setspace}
\usepackage{multirow}

\usepackage{tikz}
\usetikzlibrary{shapes,decorations,arrows,calc,arrows.meta,fit,positioning}
\tikzset{
    -Latex,auto,node distance =1 cm and 1 cm,semithick,
    state/.style ={ellipse, draw, minimum width = 0.7 cm},
    point/.style = {circle, draw, inner sep=0.04cm,fill,node contents={}},
    bidirected/.style={Latex-Latex,dashed},
    el/.style = {inner sep=2pt, align=left, sloped}
}

\usepackage{amsmath,amsthm,amsfonts,amssymb}
\usepackage[linesnumbered, ruled, vlined]{algorithm2e}

\usepackage{pifont}
\definecolor{myRed}{rgb}{0.75,0,0}
\definecolor{myGreen}{rgb}{0,0.75,0}
\usepackage{thmtools,thm-restate}

\declaretheorem{lemma}
\declaretheorem{fact}

\declaretheorem{proposition}

\theoremstyle{definition}
\declaretheorem{definition}
\declaretheorem{example}

\newcommand{\M}{\mathcal{M}}

\newcommand{\Ll}{\mathcal{L}}

\newcommand{\Bernoulli}{\textsf{Bernoulli}}

\newcommand{\Parents}{Pa}

\DeclareSymbolFont{symbolsC}{U}{txsyc}{m}{n}
\DeclareMathSymbol{\boxright}{\mathrel}{symbolsC}{128}

\newcommand{\norm}[1]{\left\lVert#1\right\rVert}

\def\*#1{\mathbf{#1}}

\newcommand{\pai}[1]{\*{pa}_{#1}}
\newcommand{\Pai}[1]{\*{Pa}_{#1}}

\newcommand{\ui}[1]{\*{u}_{#1}}
\newcommand{\Ui}[1]{\*{U}_{#1}}

\theoremstyle{definition}

\DeclareMathOperator{\rcm}{\textsf{RCM}}

\DeclareMathOperator{\cbn}{\textsf{CBN}}

\DeclareMathOperator{\unif}{Unif}
\DeclareMathOperator{\bern}{Bernoulli}

\DeclareMathOperator{\mlp}{MLP}
\DeclareMathOperator{\gp}{GP}
\DeclareMathOperator{\kl}{KL}

\newcommand{\ncm}{\text{NCM}}

\newcommand{\xdasharrow}[2][->]{
\tikz[baseline=-\the\dimexpr\fontdimen22\textfont2\relax]{
\node[anchor=south,font=\scriptsize, inner ysep=1.5pt,outer xsep=8pt](x){#2};
\draw[shorten <=3.4pt,shorten >=3.4pt,dashed,#1](x.south west)--(x.south east);
}
}

\def\ddefloop#1{\ifx\ddefloop#1\else\ddef{#1}\expandafter\ddefloop\fi}
\def\ddef#1{\expandafter\def\csname bb#1\endcsname{\ensuremath{\mathbb{#1}}}}
\ddefloop ABCDEFGHIJKLMNOPQRSTUVWXYZ\ddefloop
\def\ddef#1{\expandafter\def\csname c#1\endcsname{\ensuremath{\mathcal{#1}}}}
\ddefloop ABCDEFGHIJKLMNOPQRSTUVWXYZ\ddefloop
\def\ddef#1{\expandafter\def\csname h#1\endcsname{\ensuremath{\widehat{#1}}}}
\ddefloop ABCDEFGHIJKLMNOPQRSTUVWXYZ\ddefloop
\def\ddef#1{\expandafter\def\csname v#1\endcsname{\ensuremath{\boldsymbol{#1}}}}
\ddefloop ABCDEFGHIJKLMNOPQRSTUVWXYZabcdefghijklmnopqrstuvwxyz\ddefloop
\def\ddef#1{\expandafter\def\csname v#1\endcsname{\ensuremath{\boldsymbol{\csname #1\endcsname}}}}
\ddefloop {alpha}{beta}{gamma}{delta}{epsilon}{varepsilon}{zeta}{eta}{theta}{vartheta}{iota}{kappa}{lambda}{mu}{nu}{xi}{pi}{varpi}{rho}{varrho}{sigma}{varsigma}{tau}{upsilon}{phi}{varphi}{chi}{psi}{omega}{Gamma}{Delta}{Theta}{Lambda}{Xi}{Pi}{Sigma}{varSigma}{Upsilon}{Phi}{Psi}{Omega}{ell}\ddefloop

\makeatletter
\newcommand*{\indep}{%
  \mathbin{%
    \mathpalette{\@indep}{}%
  }%
}
\newcommand*{\nindep}{%
  \mathbin{
    \mathpalette{\@indep}{\not}
  }%
}
\newcommand*{\@indep}[2]{%
  \sbox0{$#1\perp\m@th$}
  \sbox2{$#1=$}
  \sbox4{$#1\vcenter{}$}
  \rlap{\copy0}
  \dimen@=\dimexpr\ht2-\ht4-.2pt\relax
  \kern\dimen@
  {#2}%
  \kern\dimen@
  \copy0 
} 
\makeatother
\fi

%% file: section/00_abstract.tex
Evaluating hypothetical statements about how the world would be had a different course of action been taken is arguably one key capability expected from modern AI systems. Counterfactual reasoning underpins discussions in fairness, the determination of blame and responsibility, credit assignment, and regret. In this paper, we study the evaluation of counterfactual statements through neural models. Specifically, we tackle two causal problems required to make such evaluations, i.e., counterfactual identification and estimation from an arbitrary combination of observational and experimental data. First, we show that neural causal models (NCMs) are expressive enough and encode the structural constraints necessary for performing counterfactual reasoning. Second, we develop an algorithm for simultaneously identifying and estimating counterfactual distributions. We show that this algorithm is sound and complete for deciding counterfactual identification in general settings. Third, considering the practical implications of these results, we introduce a new strategy for modeling NCMs using generative adversarial networks. Simulations corroborate with the proposed methodology.

%% file: section/01_introduction.tex
\section{Introduction}

Counterfactual reasoning is one of human's high-level cognitive capabilities, used across a wide range of affairs, including determining how objects interact, assigning responsibility, credit and blame, and articulating explanations. Counterfactual statements underpin prototypical questions of the form "what if--" and "why--", which inquire about hypothetical worlds that have not necessarily been realized   \citep{pearl:mackenzie2018}.
If a patient Alice had taken a drug and died, one may wonder, "why did Alice die?"; "was it the drug that killed her?"; "would she be alive had she not taken the drug?".
In the context of fairness, why did an applicant, Joe, not get the job offer? Would the outcome have changed had Joe been a Ph.D.? Or perhaps of a different race? These are examples of fundamental questions about \textit{attribution} and \textit{explanation}, which evoke hypothetical scenarios that disagree with the current reality and which not even experimental studies can reconstruct.

We build on the semantics of counterfactuals based on a generative process called \textit{structural causal model} (SCM) \citep{pearl:2k}. 
A fully instantiated  SCM $\M^*$ describes a collection of causal mechanisms and distribution over exogenous conditions. Each $\M^*$ induces families of qualitatively different distributions
related to the activities of seeing (called observational), doing (interventional), and imagining (counterfactual), which together are known as the \textit{ladder of causation}  \citep{pearl:mackenzie2018,bareinboim:etal20}; also called the \textit{Pearl Causal Hierarchy} (PCH).  
The PCH is a containment hierarchy in which distributions can be put in increasingly refined layers: observational content goes into layer 1 ($\Ll_1$); experimental to layer 2 ($\Ll_2$); counterfactual to layer 3 ($\Ll_3$). 
It is understood that there are questions about layers 2 and 3 that cannot be answered (i.e. are underdetermined), even given all information in the world about layer 1; further, layer 3 questions are still underdetermined given data from layers 1 and 2
\citep{bareinboim:etal20, ibeling2020probabilistic}.

Counterfactuals represent the more detailed, finest type of knowledge encoded in the PCH, so naturally, having the ability to evaluate counterfactual distributions is an attractive proposition. In practice, a fully specified model $\M^*$ is almost never observable, which leads to the question -- how can a counterfactual statement, from $\mathcal{L}_3^*$, be evaluated using a combination of observational and experimental data (from $\mathcal{L}_1^*$ and $\mathcal{L}_2^*$)?  This question embodies the challenge of cross-layer inferences, which entail solving two challenging causal problems in tandem, identification and estimation. 

\begin{wrapfigure}{r}{0.45\textwidth}
    \vspace{-0.1in}
  \begin{center}
    \begin{tikzpicture}[-, scale=0.6, every node/.append style={transform shape}]
        \draw [line width=0.05mm] (-5.5,1.9) -- (5.5,1.9);
        \draw [line width=0.05mm, draw=gray] (0,3) -- (0,-2.7);
        
        \node [align=center, font=\fontsize{11}{0}\selectfont] (atext1) at (-3, 3) {(a)};
        \node [align=center, font=\fontsize{11}{0}\selectfont] (atext2) at (-3, 2.6) {Unobserved};
        \node [align=center, font=\fontsize{11}{0}\selectfont] (atext3) at (-3, 2.2) {Nature/Truth};
        \node [align=center, font=\fontsize{11}{0}\selectfont] (btext1) at (3, 3) {(b)};
        \node [align=center, font=\fontsize{11}{0}\selectfont] (btext2) at (3, 2.6) {Learned/};
        \node [align=center, font=\fontsize{11}{0}\selectfont] (btext3) at (3, 2.2) {Hypothesized};
        
        \filldraw [fill=gray!40,line width=0.01mm] (-4.5, -0.65) rectangle (-1.5, -1.35);
        
        \node (pch) at (-5, -1) {PCH:};
    
        \filldraw [fill=gray!20,line width=0.01mm] (-4.5, 1.5) rectangle (-1.5, 0);
        \node [align=center, font=\fontsize{11}{0}\selectfont] (Pu1) at (-3,1.05) {SCM $\cM^*$};
        \node [align=center, font=\fontsize{11}{0}\selectfont] (Pu) at (-3,0.45) {$=\langle \cF^*, P(\*U^*) \rangle$};
    
        \node [] (L1L2ref) at (-3.5,-1.25) {};
    
       	\node [fill=gray!20, draw=gray, line width=0.08mm, align=center, minimum width=0.9cm] (L1Pv) at (-4,-1) {$\Ll_1^*$};
       	\node [fill=gray!20, draw=gray, line width=0.08mm, align=center, minimum width=0.9cm] (L2Pv) at (-3,-1) {$\Ll_2^*$};
       	\node [fill=gray!20, draw=gray, line width=0.08mm, align=center, minimum width=0.9cm] (L3Pv) at (-2,-1) {$\Ll_3^*$};
       	
       	\path [-Latex, line width=0.1mm] (Pu) edge (L1Pv.north);
       	\path [-Latex, line width=0.1mm] (Pu) edge (L2Pv.north);
       	\path [-Latex, line width=0.1mm] (Pu) edge (L3Pv.north);

       	\filldraw [fill=gray!0,line width=0.01mm] (4.5, 1.5) rectangle (1.5, 0);
        \node [align=center, font=\fontsize{11}{0}\selectfont] (nPu1) at (3,1.05) {NCM $\widehat{M}$};
        \node [align=center, font=\fontsize{11}{0}\selectfont] (nPu) at (3,0.45) {$= \langle \widehat{\cF}, P(\widehat{\*U})\rangle$};
    
       	\node [fill=gray!20, draw=gray, line width=0.08mm, align=center, minimum width=0.9cm] (nL1Pv) at (2,-1) {$\Ll_1$};
       	\node [fill=gray!20, draw=gray, line width=0.08mm, align=center, minimum width=0.9cm] (nL2Pv) at (3,-1) {$\Ll_2$};
       	\node [draw=gray, line width=0.08mm, align=center, minimum width=0.9cm] (nL3Pv) at (4,-1) {$\Ll_3$};
    
       	\path [-Latex, line width=0.1mm] (nPu) edge (nL1Pv.north);
       	\path [-Latex, line width=0.1mm] (nPu) edge (nL2Pv.north);
       	\path [-Latex, line width=0.1mm] (nPu) edge (nL3Pv.north);
       	
       	\path [-Latex, line width=0.2mm] (L1L2ref.south) edge[densely dotted, out=-15, in=-135] (nL1Pv.south);
       	\node [align=center, font=\fontsize{8}{0}\selectfont, inner sep=0.5mm] (trainingtext) at (-2.0, -2.3) {Training ($\cL_1 = \cL_1 ^*,$};
       	\node [align=center, font=\fontsize{8}{0}\selectfont, inner sep=0.5mm] (trainingtext2) at (-1.5, -2.6) {$\cL_2 = \cL_2^*$)};
       	
       	\filldraw [fill=gray!0,line width=0.01mm] (0.6, 1.5) rectangle (-0.6, 0);
        \node [align=center, font=\fontsize{8}{0}\selectfont] (nG1) at (0,1.2) {Causal};
        \node [align=center, font=\fontsize{8}{0}\selectfont] (nG2) at (0,0.75) {Diagram};
        \node [align=center, font=\fontsize{11}{0}\selectfont] (nG3) at (0,0.3) {$\cG$};
        
        \node [inner sep=0] (scmright) at (-1.5,0.75) {};
        \node [inner sep=0] (cgleft) at (-0.6,0.75) {};
        \node [inner sep=0] (cgright) at (0.6,0.75) {};
        \node [inner sep=0] (ncmleft) at (1.5,0.75) {};
        \path [-Latex, line width=0.1mm] (scmright) edge (cgleft);
        \path [-Latex, line width=0.1mm] (cgright) edge (ncmleft);
        
        \node [align=center, font=\fontsize{8}{0}\selectfont, fill=white] (gconstrainttext) at (0.5, -0.3) {$\cG$-Constraint};
        
    \end{tikzpicture}
    \caption{The l.h.s. contains the true SCM $\cM^*$ that induces PCH's three layers. The r.h.s. contains a neural model $\hM$ constrained by inductive bias $\cG$ (entailed by $\cM^*$) and matching $\cM^*$ on $\cL_1$ and $\cL_2$ through training.}
    \label{fig:cht}
    \end{center}
    \vspace{-0.2in}
\end{wrapfigure}
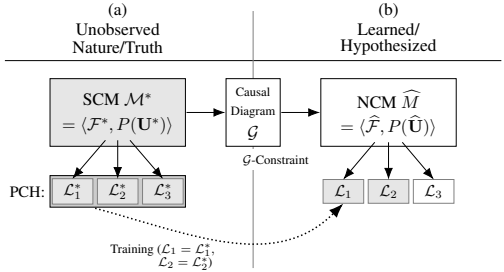

In the more traditional literature of causal inference, there are different symbolic methods for solving these problems in various settings and under different assumptions. In the context of identification, there exists an arsenal of results that includes celebrated methods such as Pearl's do-calculus \citep{pearl:95a}, and go through different algorithmic methods when considering inferences for $\Ll_2$-  \citep{tian:pea02-general-id,shpitser:pea06a,Huang2006,bareinboim:pea12-zid,lee:etal19,lee2020gidpo,lee2021matrix} and $\Ll_3$-distributions  \citep{heckman:92,pearl:01,avin:etal05,shpitser:pea09,shpitser:she18,zhang2018fairness,correa:etal21}. On the estimation side, there are various methods including the celebrated Propensity Score/IPW for the backdoor case  \citep{rubins:78,horvitz1952generalization,kennedy2019nonparametric,kallus20}, and some more relaxed settings \citep{fulcher2017robust,jung2020werm,jung2021dml}, but the literature is somewhat scarcer and less developed. 
In fact, there is a lack of estimation methods for $\cL_3$ quantities in most settings. 

On another thread in the literature, deep learning methods have achieved outstanding empirical success in solving a wide range of tasks in fields such as computer vision \citep{NIPS2012_c399862d}, speech recognition \citep{pmlr-v32-graves14}, and game playing \citep{volodymr:etal13}.
One key feature of deep learning is its ability to allow inferences to scale with the data to high dimensional settings.
We study here the suitability of the neural approach to tackle the problems of causal identification and estimation while trying to leverage the benefits of these new advances experienced in non-causal settings. 
\footnote{\label{ft:bd-models} One  of our motivations is that these methods showed great promise at estimating effects from observational data under backdoor/ignorability conditions  \citep{shalit2017estimating, louizos2017causal,  NIPS2017_b2eeb736, johansson2018learning, NEURIPS2018_a50abba8, yoon2018ganite,  kallus2018deepmatch, shi2019adapting,  du2020adversarial, Guo_2020}. } 
The idea behind the approach pursued here is illustrated in Fig.~\ref{fig:cht}. 
 Specifically, we  will search for a neural model $\hM$ (r.h.s.) that has the same generative capability of the true, unobserved SCM $\M^*$ (l.h.s.); in other words, $\hM$ should be able to generate the same observed/inputted data, i.e.,  $\mathcal{L}_1 = \mathcal{L}_1^*$ and $\mathcal{L}_2 = \mathcal{L}_2^*$.
\footnote{
This represents an extreme case where all $\cL_1$- and $\cL_2$-distributions are provided as data. In practice, this may be unrealistic, and our method takes as input any arbitrary subset of distributions from $\cL_1$ and $\cL_2$. 
}
To tackle this task in practice, we use an inductive bias for the neural model in the form of a \emph{causal diagram} \citep{pearl:2k,spirtes:etal00,bareinboim:pea16}, which is a parsimonious description of the mechanisms ($\mathcal{F}^*$) and exogenous conditions ($P(\*U^*)$) of the generating SCM. \footnote{When imposed on neural models, they enforce equality constraints connecting layer 1 and layer 2 quantities, defined formally through the \emph{causal Bayesian network} (CBN) data structure \citep[Def.~16]{bareinboim:etal20}.} 
The question then becomes: under what conditions can a model trained using this combination of qualitative inductive bias and the available data be suitable to answer questions about hypothetical counterfactual worlds, \textit{as if} we had access to the true $\M^*$? 

There exists a growing literature that leverages modern neural methods to solve causal inference tasks.\footnotemark[1] 
Our approach based on proxy causal models will answer causal queries by direct evaluation through a parameterized neural model $\hM$ fitted on the data generated by $\mathcal{M}^*$. \footnote{In general,  $\hM$ does not need to, and will not be equal to the true SCM $\M^*$.} 
For instance, some recent work solves the estimation of interventional ($\Ll_2$) or counterfactual ($\Ll_3$) distributions from observational ($\Ll_1$) data in Markovian settings,  implemented through architectures such as GANs, flows, GNNs, and VGAEs \citep{kocaoglu2018causalgan,pawlowski2020deep,zecevic2021relating,sanchezmartin2022}. 
In some real-world settings, Markovianity is a too stringent condition (see discussion in App.~\ref{app:markovianity}) and may be violated, which leads to the separation between layers 1 and 2, and, in turn, issues of causal identification. \footnote{Layer 3 differs from lower layers even in Markovian models; see \citet[Ex.~7]{bareinboim:etal20}. }
The proxy approach discussed above was pursued in \citet{xia:etal21} to solve the identification and estimation of interventional distributions  ($\Ll_2$) from observational data ($\Ll_1$) in non-Markovian settings. \footnote{\citet{witty2021} shows a related approach taking the Bayesian route; further details, see Appendix \ref{app:related-works}.} This work introduced an object we leverage throughout this paper called \textit{Neural Causal Model} (NCM, for short), which is a class of SCMs constrained to neural network functions and fixed distributions over the exogenous variables. While NCMs have been shown to be able to solve the identification and estimation tasks for $\cL_2$ queries, their potential for counterfactual inferences is still largely unexplored, and existing implementations have been constrained to low-dimensional settings.

Despite all the progress achieved so far, no practical methods exist for estimating counterfactual ($\Ll_3$) distributions in the general setting where an arbitrary combination of observational ($\Ll_1$) and experimental ($\Ll_2$) distributions is available, and unobserved confounders exist (i.e. Markovianity does not hold). Hence, in addition to providing the first neural method of counterfactual identification, this paper establishes the first general counterfactual estimation technique even among non-neural methods, leveraging the neural toolkit for scalable inferences.
Specifically, our contributions are: 
1. We prove that when fitted with a graphical inductive bias, the NCM encodes the $\cL_3$-constraints necessary for performing counterfactual inference (Thm.~\ref{thm:gl3-consistency}), and that they are still expressive enough to model the underlying data-generating model, which is not necessarily a neural network  (Thm.~\ref{thm:l3-g-expressiveness}). \\
2.  We show that counterfactual identification within a neural proxy model setting is equivalent to established symbolic approaches (Thm.~\ref{thm:ncm-ctfid-equivalence}). We leverage this duality to develop an optimization procedure (Alg.~\ref{alg:ncm-solve-ctfid}) for counterfactual identification and estimation that is both sound and complete (Corol.~\ref{thm:ncm-ctfid-correctness}). The approach is general in that it accepts any combination of inputs from $\cL_1$ and $\cL_2$, it works in any causal diagram setting, and it does not require the Markovianity assumption to hold.\\
3. We develop a new approach to modeling the NCM using generative adversarial networks (GANs) \citep{NIPS2014_5ca3e9b1}, capable of robustly scaling inferences to high dimensions (Alg.~\ref{alg:ncm-learn-pv}). We then show how GAN-NCMs can solve the challenging optimization problems in identifying and estimating counterfactuals in practice.
Experiments are provided in Sec.~5 and proofs in Appendix \ref{app:proofs}.

%% file: section/01-1_preliminaries.tex

\noindent \textbf{Preliminaries.}
We now introduce the notation and definitions used throughout the paper. We use uppercase letters ($X$) to denote random variables and lowercase letters ($x$) to denote corresponding values. Similarly, bold uppercase ($\*X$) and lower case ($\*x$) letters are used to denote sets of random variables and values respectively. We use $\cD_{X}$ to denote the domain of $X$ and $\cD_{\mathbf{X}} = \cD_{X_1} \times \dots \times \cD_{X_k}$ for the domain of $\mathbf{X} = \{X_1, \dots, X_k\}$. We denote $P(\*X = \*x)$ (which we will often shorten to $P(\*x)$) as the probability of  $\*X$ taking the values $\*x$ under the probability distribution $P(\*X)$.

We utilize the basic semantic framework of structural causal models (SCMs), as defined in  \citep[Ch.~7]{pearl:2k}.
An SCM $\cM$ consists of endogenous variables $\*V$, exogenous variables $\*U$ with distribution $P(\*U)$, and mechanisms $\cF$. $\cF$ contains a function $f_{V_i}$ for each variable $V_i$ that maps endogenous parents $\Pai{V_i}$ and exogenous parents $\Ui{V_i}$ to $V_i$. Each $\cM$  induces a causal diagram $\cG$, where every $V_i \in \*V$ is a vertex, there is a directed arrow $(V_j \rightarrow V_i)$ for every $V_i \in \*V$ and $V_j \in \Pai{V_i}$, and there is a dashed-bidirected arrow $(V_j  \dashleftarrow \dashrightarrow V_i)$ for every pair $V_i, V_j \in \*V$ such that $\Ui{V_i}$ and $\Ui{V_j}$ are not independent. For further details, see \citep[Def.~13/16,~Thm.~4]{bareinboim:etal20}. The exogenous $\*U_{V_i}$'s are not assumed independent (i.e.\ Markovianity is not required). Our treatment is constrained to \emph{recursive} SCMs, which implies acyclic causal diagrams, with finite domains over  $\mathbf{V}$. 

Each SCM $\cM$ assigns values to each counterfactual distribution as follows:
\begin{definition}[Layer 3 Valuation] 
\label{def:l3-semantics}
An SCM $\cM$ induces layer $\cL_3(\cM)$, a set of distributions over $\*V$, each with the form $P(\*Y_*) = P(\*Y_{1[\*x_1]}, \*Y_{2[\*x_2], \dots})$ such that 
\begin{align}
    \label{eq:def:l3-semantics}
    P^{\cM}(\*y_{1[\*x_1]}, \*y_{2[\*x_2]}, \dots) = 
    \int_{\cD_{\mathbf{U}}} \mathbbm{1}\left[\*Y_{1[\*x_1]}(\*u)=\*y_1, \*Y_{2[\*x_2]}(\*u) = \*y_2, \dots \right] dP(\*u),
\end{align}
where ${\*Y}_{i[\*x_i]}(\*u)$ is evaluated under 
 $\mathcal{F}_{\*x_i}\! :=\! \{f_{V_j}\! :\! V_j \in \*V \setminus \*X_i\} \cup \{f_X \leftarrow x\! :\! X \in \*X_i\}$.  \hfill $\blacksquare$

\end{definition}
Each $\*Y_i$ corresponds to a set of variables in a world where the original mechanisms $f_X$ are replaced with constants $\*x_i$ for each $X \in \*X_i$; this is also known as the mutilation procedure. This procedure corresponds to interventions, and we use subscripts to denote the intervening variables (e.g. $\*Y_{\*x}$) or subscripts with brackets when the variables are indexed (e.g. $\*Y_{1[\*x_1]}$). For instance, $P(y_x, y'_{x'})$ is the probability of the joint counterfactual event $Y=y$ had $X$ been $x$ and $Y=y'$ had $X$ been $x'$. 

SCM $\cM_2$ is said to be $P^{(\cL_i)}$-consistent (for short, $\cL_i$-consistent) with SCM $\cM_1$ if $\cL_i(\cM_1) = \cL_i(\cM_2)$.
We will use $\bbZ$ to denote a set of quantities from Layer 2 (i.e. $\bbZ = \{P(\*V_{\*z_k})\}_{k=1}^{\ell}$), and we use $\bbZ(\cM)$ to denote those same quantities induced by SCM $\cM$ (i.e. $\bbZ(\cM) = \{P^{\cM}(\*V_{\*z_k})\}_{k=1}^{\ell}$).

We use neural causal models (NCMs) as a substitute (proxy) model for the true SCM, as follows:

\begin{definition}[$\cG$-Constrained Neural Causal Model ($\cG$-NCM) {\citep[Def.~7]{xia:etal21}}]
    \label{def:gncm}
    Given a causal diagram $\cG$, a $\cG$-constrained Neural Causal Model (for short, $\cG$-NCM) $\widehat{M}(\bm{\theta})$ over variables $\*V$ with parameters $\bm{\theta} = \{\theta_{V_i} : V_i \in \*V\}$ is an SCM $\langle \widehat{\*U}, \*V, \widehat{\cF}, \widehat{P}(\widehat{\*U}) \rangle$ such that $\widehat{\*U} = \{\widehat{U}_{\*C} : \*C \in \bbC(\cG)\}$, where $\bbC(\cG)$ is the set of all maximal cliques over bidirected edges of $\cG$, and $\cD_{\widehat{U}} = [0, 1]$ for all $\widehat{U} \in \widehat{\*U}$; $\widehat{\cF} = \{\hat{f}_{V_i} : V_i \in \*V\}$, where each $\hat{f}_{V_i}$ is a feedforward neural network parameterized by $\theta_{V_i} \in \bm{\theta}$ mapping values of $\Ui{V_i} \cup \Pai{V_i}$ to values of $V_i$ for $\Ui{V_i} = \{\widehat{U}_{\*C} : \widehat{U}_{\*C} \in \widehat{\*U} \text{ s.t. } V_i \in \*C\}$ and $\Pai{V_i} = \Parents_{\cG}(V_i)$; $\widehat{P}(\widehat{\*U})$ is defined s.t.\ $\widehat{U} \sim \unif(0, 1)$ for each $\widehat{U} \in \widehat{\*U}$.
    \hfill $\blacksquare$
\end{definition}

%% file: section/02_data_structure.tex
\section{Neural Causal Models for Counterfactual Inference}
\label{sec:data-structure}

We first recall that inferences about higher layers of the PCH generated by the true SCM $\cM^*$ cannot be made in general through an NCM $\widehat{M}$ trained only from lower layer data \citep{bareinboim:etal20,xia:etal21}. 
This impossibility motivated the use of the inductive bias in the form of a causal diagram $\cG$ in the construction of the NCM in Def.~\ref{def:gncm}, which ascertains that the $\cG$-consistency  property holds. (See App.~\ref{sec:cg-assumption} for further discussion.) 
We next define consistency w.r.t. to each layer, 
which will be key for a more fine-grained discussion later on.

\begin{definition}[$\cG^{(\cL_i)}$-Consistency]
    \label{def:gli-consistency}
    Let $\cG$ be the causal diagram induced by the SCM $\cM^*$. For any SCM $\cM$, $\cM$ is said to be $\cG^{(\cL_i)}$-consistent (w.r.t.~$\cM^*$) if $\cL_i(\cM)$ satisfies all layer $i$ equality constraints implied by $\cG$.
    \hfill $\blacksquare$
\end{definition}

This generalization is subtle since regardless of which $\cL_i$ is used with the definition, the causal diagram $\cG$ generated by $\M^*$ is the same. 
The difference lies in the implied constraints.
For instance, if an SCM $\cM$ is $\cG^{(\cL_1)}$-consistent, that means that $\cG$ is a Bayesian network for the observational distribution of $\cM$, implying independences readable through d-separation \cite{pearl:88a}. 
If $\cM$ is $\cG^{(\cL_2)}$-consistent, that means that $\cG$ is a \emph{Causal Bayesian network} (CBN) \citep[Def.~16]{bareinboim:etal20} for the interventional distributions of $\cM$. 
While several SCMs could share the same d-separation constraints as $\cM^*$, there are fewer that share all $\cL_2$ constraints encoded by the CBN. $\cG$-consistency at higher layers imposes a stricter set of constraints, narrowing down the set of compatible SCMs. 
There also exist constraints of layer 3 that are important for counterfactual inferences.

To motivate the use of such constraints, consider an example inspired by the multi-armed bandit problem. A casino has 3 slot machines, labeled ``0", ``1", and ``2". Every day, the casino assigns one machine a good payout, one a bad payout, and one an average payout, with chances of winning represented by exogenous variables $U_+$, $U_-$, and $U_=$, respectively. A customer comes every day and plays a slot machine. $X$ represents their choice of machine, and $Y$ is a binary variable representing whether they win. Suppose a data scientist creates a model of the situation, and she hypothesizes that the casino predicts the customer's choice based on their mood ($U_M$) and will always assign the predicted machine the average payout  to maintain profits. Her model is described by the SCM $\cM'$:
\begin{align} 
\cM' =
\begin{cases}
    \*U &= \{U_M, U_+, U_=, U_-\}, U_M \in \{0, 1, 2\}, U_+, U_=, U_- \in \{0, 1\} \\
    \*V &= \{X, Y\}, X \in \{0, 1, 2\}, Y \in \{0, 1\} \\
    \cF &=
    \begin{cases}
        f_X(u_M) &= u_M \\
        f_Y(x, u_M, u_+, u_=, u_-) &=
        \begin{cases}
            u_= & x = u_M \\
            u_- & x = (u_M - 1) \% 3 \\
            u_+ & x = (u_M + 1) \% 3
        \end{cases}
    \end{cases} \\
    P(\*U): & P(U_M\!=\!i) = \frac{1}{3}, P(U_+\!=\!1) = 0.6, P(U_=\!=\!1) = 0.4, P(U_-\!=\!1) = 0.2
\end{cases}
\end{align}
It turns out that in this model $P(y_x) = P(y \mid x)$. For example, $P(Y = 1 \mid X = 0) = P(U_= = 1) = 0.4$, and $P(Y_{X = 0} = 1) = P(U_M = 0)P(U_= = 1) + P(U_M = 1)P(U_- = 1) + P(U_M = 2)P(U_+ = 1) = \frac{1}{3}(0.4) + \frac{1}{3}(0.2) + \frac{1}{3}(0.6) = 0.4$.

Suppose the true model $\cM^*$ employed by the casino (and unknown by the customers and data scientist) induces graph $\cG = \{ X \rightarrow Y \}$. 
 Interestingly enough, $\cM'$ would be $\cG^{(\cL_2)}$-consistent with $\cM^*$ since $\cM'$ is compatible with all $\Ll_2$-constraints, including $P(y_x) = P(y \mid x)$ and $P(x_y) = P(x)$. However, and perhaps surprisingly, it would fail to be $\cG^{(L_3)}$-consistent. A further constraint implied by $\cG$ on the third layer is that $P(y_x \mid x') = P(y_x)$, which is not true of $\cM'$. To witness, note that $P(Y_{X=0} = 1 \mid X = 2) = P(U_+ = 1) = 0.6$ in $\cM'$, which means that if the customer chose machine 2, they would have had higher payout had they chosen machine 0. This does not match $P(Y_{X=0} = 1) = 0.4$, computed earlier, so $\cM'$ fails to encode the $\cL_3$-constraints implied by $\cG$.

\begin{wrapfigure}{r}{0.22\textwidth}
\vspace{-0.0in}
\includegraphics[width=\linewidth]{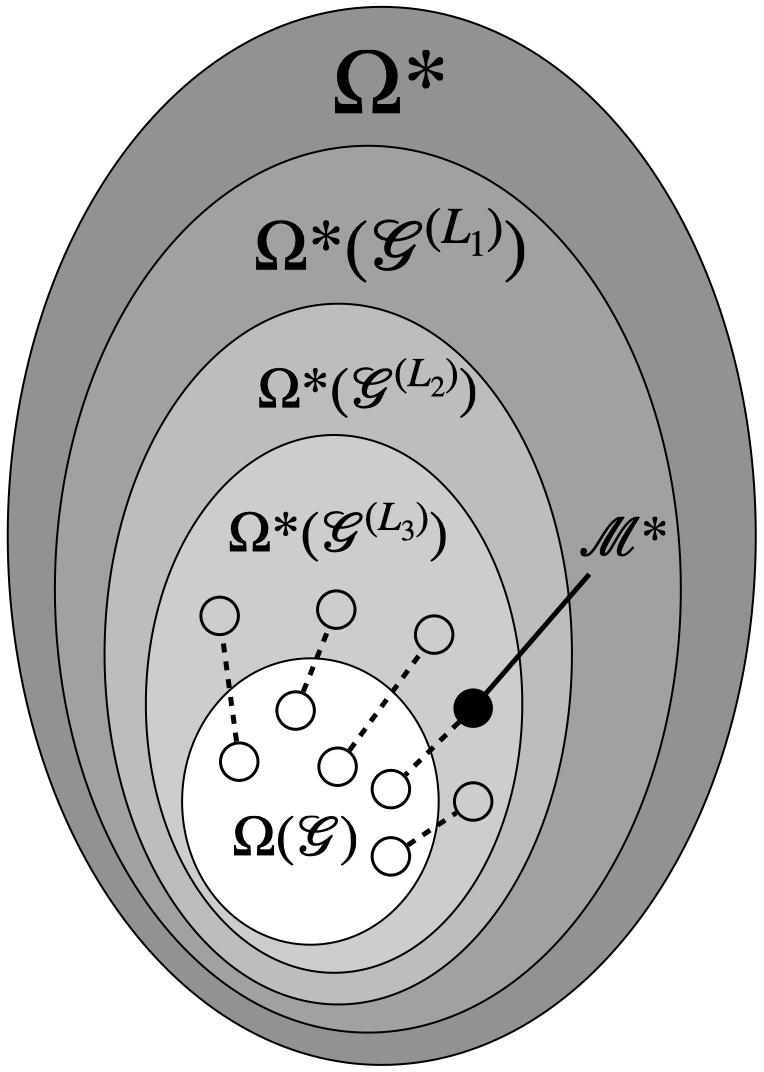}
\caption{
Model-theoretic visualization of  Thms.~\ref{thm:gl3-consistency} and \ref{thm:l3-g-expressiveness}. 
}
\label{fig:causal-criteria}
\vspace{-0.4in}
\end{wrapfigure}

In general, the causal diagram encodes a family of $\cL_3$-constraints which we leverage to make cross-layer inferences. A more detailed discussion can be found in  Appendix~\ref{app:examples}. We show next that NCMs encodes all of the equality constraints related to $\Ll_3$, in addition to the known $\Ll_2$-constraints. 

\begin{restatable}[NCM $\cG^{(\cL_3)}$-Consistency]{theorem}{glconsistency}
    \label{thm:gl3-consistency}
    Any $\cG$-NCM $\widehat{M}(\bm{\theta})$ is $\cG^{(\cL_3)}$-consistent.
    \hfill $\blacksquare$
\end{restatable}

This will be a key result for performing inferences at the counterfactual level. Similar to how constraints about layer 2 distributions help bridge the gap between layers 1 and 2, layer 3 constraints allow us to extend our inference capabilities into layer 3. (In fact, most of $\Ll_3$'s distributions are not obtainable through experimentation.)   While this graphical inductive bias is powerful, the set of NCMs constrained by $\cG$ is no less expressive than the set of SCMs constrained by $\cG$, as shown next.

\begin{restatable}[$\cL_3$-$\cG$ Expressiveness]{theorem}{lgexpressiveness}
    \label{thm:l3-g-expressiveness}
    For any SCM $\cM^*$ that induces causal diagram $\cG$, there exists a $\cG$-NCM  $\widehat{M}(\bm{\theta}) = \langle \widehat{\*U}, \*V, \widehat{\cF}, \widehat{P}(\widehat{\*U}) \rangle$ s.t. $\widehat{M}$ is $\cL_3$-consistent w.r.t. $\cM^*$.
    \hfill $\blacksquare$
\end{restatable}

This result ascertains that the NCM class is as expressive, and therefore, contains the same generative capabilities as the original generating model. More interestingly, even if the original SCM $\M^*$ does not belong to the NCM class, but from the higher space, there exists a NCM $\widehat{M}(\bm{\theta})$ that will be capable of expressing the collection of distributions from all layers of the PCH induced by it. 

A visual representation of these two results is shown in  Fig.~\ref{fig:causal-criteria}.
The space of all SCMs
is called $\Omega^*$, and the subspace that contains all SCMs $\cG^{((\cL_i)}$-consistent w.r.t. the true SCM $\cM^*$ (black dot) is called $\Omega^*(\cG^{(\cL_i)})$. Note that the $\cG^{(\cL_i)}$ space shrinks with higher layers, indicating a more constrained space with fewer SCMs. Thm.~\ref{thm:gl3-consistency} states that all $\cG$-NCMs ($\Omega(\cG)$) are within $\Omega^*(\cG^{(\cL_3)})$, and Thm.~\ref{thm:l3-g-expressiveness} states that all SCMs in $\Omega^*(\cG^{(\cL_3)})$ can be represented by a corresponding $\cG$-NCM on all three layers.

It may seem intuitive that the $\cG$-NCM has these two properties by construction, but these properties are nontrivial and, in fact, not enjoyed by many model classes. Examples can be found in Appendix \ref{app:examples}. Together, these two theorems ensure that the NCM has both the constraints and the expressiveness necessary for counterfactual inference, elaborated further in the next section.

%% file: section/03_counterfactual_id.tex
\section{Neural Counterfactual Identification}
\label{sec:ncm-ctf-id}

The problem of identification is concerned with determining whether a certain quantity is computable from a combination of assumptions, usually encoded in the form of a causal diagram, and a collection of distributions  \cite[p.~77]{pearl:2k}. 
This challenge stems from the fact that even though the space of SCMs (or NCMs) is constrained upon assuming a certain causal diagram, the quantity of interest may still be underdetermined. 
In words, there are many SCMs compatible with the same diagram $\cG$ but generate different answers for the target distribution. 
In this section, we investigate the problem of identification and decide whether counterfactual quantities (from $\Ll_3$) can be inferred from a combination of a subset of $\Ll_2$ and $\Ll_1$ datasets together with $\cG$, as formally defined next.

\begin{definition}[Neural Counterfactual Identification]
    \label{def:ncm-l3-id}
    Consider an SCM $\cM^*$ and the corresponding causal diagram $\cG$. Let $\bbZ = \{P(\*V_{\*z_k})\}_{k=1}^{\ell}$ be a collection of available interventional (or observational if $\*Z_k = \emptyset$) distributions from $\cM^*$.
    The counterfactual query $P(\*Y_* = \*y_* \mid \*X_* = \*x_*)$ is said to be neural identifiable (identifiable, for short) from the set of $\cG$-constrained NCMs $\Omega(\cG)$ and $\bbZ$ if and only if $P^{\widehat{M}_1}(\*y_* \mid \*x_*) = P^{\widehat{M}_2}(\*y_* \mid \*x_*)$ for every pair of models $\widehat{M}_1, \widehat{M}_2 \in \Omega(\cG)$ s.t. they match $\cM^*$ on all distributions in $\bbZ$ (i.e. $\bbZ(\cM^*) = \bbZ(\cM_1) = \bbZ(\cM_2) > 0$).
    \hfill $\blacksquare$
\end{definition}

From a symbolic standpoint, a counterfactual quantity $P(\*y_* \mid \*x_*)$ is identifiable from $\cG$ and $\bbZ$ if all SCMs that induce the distributions of $\bbZ$ and abide by the constraints of $\cG$ also agree on $P(\*y_* \mid \*x_*)$. This is illustrated in Fig.~\ref{fig:ctfid}. In the definition above, the search is constrained to the NCM subspace (shown in light gray) within the space of SCMs (dark gray). It may be concerning that the true SCM $\M^*$ might not be an NCM, as we alluded to earlier. The next result ascertains that identification within the constrained space of NCMs is actually equivalent to identification in the original SCM-space.

\begin{wrapfigure}{r}{0.42\textwidth}
    \vspace{-0.2in}
    \begin{tikzpicture}[-, scale=0.7, every node/.append style={transform shape}]
        \filldraw [fill=gray!50, line width=0.01mm] (-2.05,-0.05) ellipse (2.0 and 1.0);
        \filldraw [fill=gray!20, line width=0.01mm] (-1.6,0.1) ellipse (1.2 and 0.6);
      	\node [align=center, font=\fontsize{9}{0}\selectfont] at (-3.5,0.85) {$\Omega^*$};
      	\node [align=center, font=\fontsize{9}{0}\selectfont] at (-1.4,1.3) {$\Omega$};
      	\node [inner sep=0] (omg) at (-1.4,1.05) {};
      	\node [inner sep=0] (omgspace) at (-1.4,0.5) {};
      	\path [-] (omg) edge (omgspace);
      	
      	\node at (-2.85,-0.4) {$\cM^*$};
      	\draw [fill=black] (-2.75,-0.7) circle (0.1);
      	\node [inner sep=0] (ms) at (-2.75,-0.7) {\ };

      	\draw [fill=black] (-2.15,0.3) circle (0.1);
      	\node [inner sep=0] (mp) at (-2.15,0.3) {};
      	\node at (-2.45,0.1) {$\widehat{M}_1$};
      	\draw [fill=black] (-0.7,0) circle (0.1);
      	\node [inner sep=0] (mp2) at (-0.7,0) {};
      	\node at (-1,0.35) {$\widehat{M}_2$};
      	
      	\draw [densely dotted] (-1.8,0.35) circle (0.1);
      	\draw [densely dotted] (-1.85,-0.15) circle (0.1);
      	\draw [densely dotted] (-1.15,-0.15) circle (0.1);
      	\draw [densely dotted] (-1.4,0.1) circle (0.1);
      	
      	\draw [densely dotted] (-3.0,0.4) circle (0.1);
        \draw [densely dotted] (-3.5,0.25) circle (0.1);
        \draw [densely dotted] (-3.25,0.0) circle (0.1);
        \draw [densely dotted] (-3.8, 0.0) circle (0.1);
        \draw [densely dotted] (-3.4,-0.4) circle (0.1);
        \draw [densely dotted] (-2.2,-0.8) circle (0.1);
        \draw [densely dotted] (-1.7,-0.7) circle (0.1);

      	\draw (-2.05,-2.32) ellipse (1.9 and 0.8);
      	\node [align=center, font=\fontsize{8}{0}\selectfont] (l1) at (-2.95,-3.6) {($\Ll_1, \Ll_2$)};
      	\node [align=center, font=\fontsize{8}{0}\selectfont] (l1) at (-1.6,-3.6) {Data\\ Distributions};
      	\node at (-2.05,-2.4) {\small $\bbZ(\cM^*) \! = \! \bbZ(\hM_1)\! = \! \bbZ(\hM_2)$};
      	\draw [fill=black] (-2.05,-2.1) circle (0.05);
      	\node [inner sep=0.2em] (pl1) at (-2.05,-2.1) {\ };
    
      	\draw (2.7,-2.32) ellipse (1.77 and 0.8);
      	\node [align=center, font=\fontsize{8}{0}\selectfont] (l2) at (1.9,-3.6) {($\Ll_3$)};
      	\node [align=center, font=\fontsize{8}{0}\selectfont] (l2t) at (3.1,-3.6) {Counterfactual\\ Query};
      	\draw [densely dotted] (2.7,-2.32) ellipse (1.4 and 0.7);
      	\node [align=center] at (2.7,-2.4) {{\small $P^{\hM_1}\!(\*y_* | \*x_*)=$} \\ {\small $P^{\hM_2}(\*y_* | \*x_*)$}};
    
      	\draw [fill=black] (2.7,-1.9) circle (0.05);
      	\node [inner sep=0.2em] (pl21) at (2.7,-1.9) {\ };

      	\path [-Latex] (ms) edge (pl1);
      	\path [-Latex] (mp) edge (pl1);
      	\path [-Latex] (mp) edge (pl21);
      	\path [-Latex] (l1) edge node[above]  {?} (l2);
      	\path [-Latex] (mp2) edge (pl1);
      	\path [-Latex] (mp2) edge (pl21);
      	
      	\draw (2.9,-0.05) ellipse (1 and 0.6);
      	\node [align=center, font=\fontsize{9}{0}\selectfont] at (2.9,0.75) {Structural Assumptions};
      	\draw [densely dotted] (2.8,-0.05) ellipse (0.7 and 0.45);
      	\node [align=center] at (3.0,0) {$\cG$};
      	\draw [fill=black] (2.7,0) circle (0.05);
      	\node [inner sep=0.2em] (g1) at (2.7,0) {\ };
      	\path [-Latex] (mp) edge [bend left=25] (g1);
      	\path [-Latex] (ms) edge (g1);
      	\path [-Latex] (mp2) edge (g1);
    \end{tikzpicture}
    
    \caption[Neural Counterfactual ID]{
    $P(\*y_*)$ is identifiable from $\bbZ$ and $\Omega(\cG)$ if for any SCM $\cM^* \in \Omega^*$ and NCMs $\widehat{M}_1, \widehat{M}_2 \in \Omega$ (top left), $\widehat{M}_1, \widehat{M}_2, \cM^*$ match in $\bbZ$ (bottom left) and $\cG$ (top right), then the NCMs $\widehat{M}_1$, $\widehat{M}_2$ also match in $P(\*y_*)$ (bottom right).
    }
    \vspace{-0.3in}
    \label{fig:ctfid}
\end{wrapfigure}
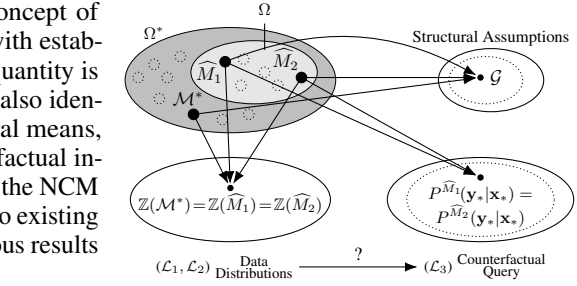

\begin{restatable}[Counterfactual Graphical-Neural Equivalence (Dual ID)]{theorem}{idequivalence}
    \label{thm:ncm-ctfid-equivalence}
    Let $\Omega^*, \Omega$ be the spaces including all SCMs and NCMs, respectively. Consider the true SCM $\cM^*$ and the corresponding causal diagram $\cG$. Let $Q = P(\*y_* \mid \*x_*)$ be the target query and $\bbZ$ the set of observational and interventional distributions available. Then, $Q$ is neural identifiable from $\Omega(\cG)$ and $\bbZ$ if and only if it is identifiable from $\cG$ and $\bbZ$.
    \hfill $\blacksquare$
\end{restatable}

Interestingly, this result connects the new concept of neural counterfactual identification (Def.~\ref{def:ncm-l3-id}) with established non-neural results. If a counterfactual quantity is determined to be neural identifiable, then it is also identifiable from $\cG$ and $\bbZ$ through other non-neural means, and vice versa.\footnote{We say identification from $\cG$ and $\bbZ$ instead of $\Omega^*(\cG)$ and $\bbZ$ because existing symbolic approaches (e.g. do-calculus) directly solve the identification problem on top of the graph instead of the space of SCMs.}  Practically speaking, counterfactual inference can be performed while constrained in the NCM space, and the obtained results will be faithful to existing symbolic approaches. This broadens the previous results connecting NCMs to classical identification. 

\begin{restatable}[Neural Counterfactual Mutilation (Operational ID)]{corollary}{opid}
    \label{cor:op-id}
    Consider the true SCM $\cM^* \in \Omega^*$, causal diagram $\cG$, a set of available distributions $\bbZ$, and a target query $Q$ equal to $P^{\cM^*}(\*y_* \mid \*x_*)$. Let $\hM \in \Omega(\cG)$ be a $\cG$-constrained NCM such that $\bbZ(\hM) = \bbZ(\cM^*)$. If $Q$ is identifiable from $\cG$ and $\bbZ$, then $Q$ is computable via Eq.~\ref{eq:def:l3-semantics} from $\hM$.
    \hfill $\blacksquare$
\end{restatable}

Corol.~\ref{cor:op-id} states that once identification is established, the counterfactual query can be inferred through the NCM $\hM$, as if it were the true SCM $\cM^*$, by directly applying layer 3's definition to 
$\hM$ (Def.~\ref{def:l3-semantics}). Remarkably, this result holds even if $\cM^*$ does not match $\hM$ in either the mechanisms $\cF$ or the exogenous dist. $P(\*U)$, and it only requires some specific properties: $\cG^{(\cL_3)}$-consistency, matching $\bbZ$, and identifiability. Without these properties, inferences performed on $\hM$ would bear no meaningful information about the ground truth. To understand this subtlety, refer to examples in App.~\ref{app:examples}. 

\begin{wrapfigure}{r}{0.53\textwidth}
\vspace{-0.2in}
\IncMargin{1em}
\begin{algorithm}[H]
    \scriptsize
    \setstretch{0.9}
    \renewcommand{\AlCapSty}[1]{\normalfont\scriptsize{\textbf{#1}}\unskip}
    
    \DontPrintSemicolon
    \SetKwData{ncmdata}{$\widehat{M}$}
    \SetKwData{graphdata}{$\cG$}
    \SetKwData{variabledata}{$\*V$}
    \SetKwData{pvdata}{$P(\*v)$}
    \SetKwData{thetamin}{$\bm{\theta}_{\min}^*$}
    \SetKwData{thetamax}{$\bm{\theta}_{\max}^*$}
    \SetKwFunction{ncmfunc}{NCM}
    \SetKwInOut{Input}{Input}
    \SetKwInOut{Output}{Output}
    
    \Input{ query $Q = P(\*y_* | \*x_*)$, $\cL_2$ datasets $\bbZ(\cM^*)$, and causal diagram \graphdata}
    \Output{ $P^{\cM^*}(\*y_* | \*x_*)$ if identifiable, \texttt{FAIL} otherwise.}
    \BlankLine
    \textls[-20]{
    $\ncmdata \gets \ncmfunc{\variabledata, \graphdata}$ \tcp*{from Def.\ \ref{def:gncm}}
    $\thetamin \! \gets \! \arg \min_{\bm{\theta}} P^{\ncmdata(\bm{\theta})}(\*y_* | \*x_*)$ s.t. $\bbZ(\ncmdata(\bm{\theta})) \! = \! \bbZ(\cM^*)$\;
    $\thetamax \! \gets \! \arg \max_{\bm{\theta}} P^{\ncmdata(\bm{\theta})}(\*y_* | \*x_*)$ s.t. $\bbZ(\ncmdata(\bm{\theta})) \! = \! \bbZ(\cM^*)$\;
    \eIf{$P^{\ncmdata(\thetamin)}(\*y_* | \*x_*) \neq P^{\ncmdata(\thetamax)}(\*y_* | \*x_*)$} {
        \Return \texttt{FAIL}
    }{
        \Return $P^{\ncmdata(\thetamin)}(\*y_* | \*x_*)$ \tcp*{choose min or max arbitrarily} 
    }
    }
    \caption{\scriptsize \textbf{NeuralID} -- Identifying/estimating counterfactual queries with NCMs.}
    \label{alg:ncm-solve-ctfid}
\end{algorithm}
\DecMargin{1em}
\vspace{-0.3in}
\end{wrapfigure}
Building on these results, we demonstrate through the procedure \textbf{NeuralID} (Alg.~\ref{alg:ncm-solve-ctfid}) how to decide the identifiability of counterfactual quantities. The specific optimization procedure searches explicitly in the space of NCMs for two models that respectively minimize and maximize the target query while maintaining consistency with the provided data distributions in $\bbZ$. If the two models match in the target query $Q$, then the effect is identifiable, and the value is returned; otherwise, the effect is non-identifiable. 

The implementation of how to enforce these consistency constraints in practice is somewhat challenging. We note two nontrivial details that are abstracted away in the description of Alg.~\ref{alg:ncm-solve-ctfid}. First, although training to fit a single, observational dataset is straightforward, it is not as clear how to simultaneously maintain consistency with the multiple datasets in $\bbZ$. Second, unlike with simpler interventional queries, it is not clear how to search the parameter space in a way that maximizes or minimizes a counterfactual query, which may be more involved due to nesting (e.g. $P(Y_{Z_{X=0}})$) or evaluating the same variable in multiple worlds (e.g. $P(Y_{X=0}, Y_{X=1})$). The details of how to solve these issues are discussed in Sec.~\ref{sec:ncm-estimation}.

Interestingly, this approach is qualitatively different than the case of classical, symbolic methods that avoid operating in the space of SCMs directly. Still, in principle, this alternative approach does not imply any loss in functionality, as evident from the next result. 

\begin{restatable}[Soundness and Completeness]{corollary}{completeness}
    \label{thm:ncm-ctfid-correctness}
    Let $\Omega^*$ be the set of all SCMs,  $\cM^* \in \Omega^*$ be the true SCM inducing causal diagram $\cG$, $Q = P(\*y_* \mid \*x_*)$ be a query of interest, and $\widehat{Q}$ be the result from running Alg.~\ref{alg:ncm-solve-ctfid} 
     with inputs $\bbZ(\cM^*) > 0$, $\cG$, and $Q$. Then $Q$ is identifiable from $\cG$ and $P^*(\*v)$ if and only if $\widehat{Q}$ is not \texttt{FAIL}. Moreover, if $\widehat{Q}$ is not \texttt{FAIL}, then $\widehat{Q} = P^{\cM^*}(\*y_* \mid \*x_*)$.
    \hfill $\blacksquare$
\end{restatable}

In words, the procedure \textbf{NeuralID} is both necessary and sufficient in this very general setting, implying that for any instances involving any arbitrary diagram, datasets, or queries, the identification status of the query is always classified correctly by this algorithm.

%% file: section/04_gan_counterfactual_estimation.tex
\section{Neural Counterfactual Estimation}
\label{sec:ncm-estimation}

\begin{wrapfigure}{r}{0.39\textwidth}
\vspace{-0.2in}
\IncMargin{1em}
\begin{algorithm}[H]
    \scriptsize
    \setstretch{0.9}
    \renewcommand{\AlCapSty}[1]{\normalfont\scriptsize{\textbf{#1}}\unskip}
    
    \DontPrintSemicolon
    \SetKwFunction{samp}{sample}
    \SetKwFunction{add}{add}
    
    \SetKwInOut{Input}{Input}
    \SetKwInOut{Output}{Output}
    
    \SetKwProg{Fn}{Function}{:}{}
    
    \Input{NCM $\widehat{M}(\bm{\theta}) = \langle \widehat{\*U}, \*V, \widehat{\cF}, P(\widehat{\*U}) \rangle$, counterfactual $\*Y_*$, conditional $\*X_* = \*x_*$, number of samples $m$}
    \Output{ $m$ samples from $P^{\hM}(\*Y_* | \*x_*)$}
    \BlankLine
    \textls[-20]{
    \Fn{\hM.\samp{$\*Y_*$, $\*x_*$, $m$}}{
        $S \gets \emptyset$\;
        \While{$|S| < m$}{
            $\widehat{\*u} \gets P(\widehat{\*U})$.\samp{}\;
            \If{$\*X_*^{\hM(\bm{\theta})}(\widehat{\*u}) = \*x_*$} {
                $S$.\add{$\*Y_*^{\hM(\bm{\theta})}(\widehat{\*u})$}\;
            }
        }
        \KwRet $S$;
    }
    }
    \caption{\scriptsize NCM Counterfactual Sampling}
    \label{alg:ncm-ctf-sample}
\end{algorithm}
\DecMargin{1em}
\vspace{-0.2in}
\end{wrapfigure}

We developed a procedure for identifying counterfactual quantities that is both sound and complete  under ideal conditions -- e.g., unlimited data, perfect optimization, which is encouraging. In this section, we build on these results and establish a more practical approach to solving these tasks under imperfect optimization and finite samples. 

Consider a $\cG$-NCM $\hM$ constructed as specified by Def.~\ref{def:gncm}.  Any counterfactual statement $\*Y_* = (\*Y_{1[\*x_1]}, \*Y_{2[\*x_2]}, \dots)$ can be evaluated from an NCM $\hM$ for a specific setting of $\widehat{\*U} = \widehat{\*u}$ by computing the corresponding values of $\*Y_i$ in the mutilated submodel $\hM_{\*x}$ for each $i$. That is,
\begin{equation}
    \label{eq:ncm-ctf-eval}
    \*Y_*^{\hM}(\*u) = (\*Y_{1[\*x_1]}^{\hM}(\*u), \*Y_{2[\*x_2]}^{\hM}(\*u), \dots)
\end{equation}
Then, sampling can be done following the natural approach delineated by Alg.~\ref{alg:ncm-ctf-sample}. In words, the distribution $P(\*Y_* \mid \*x_*)$ can be sampled from NCM $\hM$ by (1) sampling instances of $P(\widehat{\*U})$, (2) computing their corresponding value for $\*X_*$ (via Eq.~\ref{eq:ncm-ctf-eval}) while rejecting cases that do not match $\*x_*$, and (3) returning the corresponding value for $\*Y_*$ (via Eq.~\ref{eq:ncm-ctf-eval}) for the remaining instances.

Following this procedure, a counterfactual $P(\*Y_* = \*y_* \mid \*X_* = \*x_*)$ can be estimated from the NCM through a Monte-Carlo approach, instantiating Eq.~\ref{eq:ncm-mc-sampling} as follows: 
\begin{align}
     P^{\widehat M} \left( \*y_* \mid \*x_* \right) \approx \;
    &\frac{\sum_{j=1}^m \mathbbm{1}\left[\*Y^{\hM}_*(\widehat{\*u}_j) = \*y_*, \*X^{\hM}_*(\widehat{\*u}_j) = \*x_* \right]}{\sum_{j=1}^m \mathbbm{1}\left[\*X^{\hM}_*(\widehat{\*u}_j) = \*x_* \right]}, \label{eq:ncm-mc-sampling}
\end{align}
where $\{\widehat{\*u}_j\}_{j=1}^m$ are a set of $m$ samples from $P(\widehat{\*U})$.

Alg.~\ref{alg:ncm-learn-pv} demonstrates how to solve the challenging optimization task in lines 2 and 3 of Alg.~\ref{alg:ncm-solve-ctfid}. The first step is to learn parameters such that the distributions induced by the NCM $\hM$ match the true distributions in $\bbZ$. While Alg.~\ref{alg:ncm-solve-ctfid} describes the inputs in the form of $\Ll_2$-distributions, $\bbZ(\cM^*) = \{P^{\cM^*}(\*V_{\*z_k})\}_{k=1}^{\ell}$, in most settings, one has the empirical versions of such distributions in the form of finite datasets, $\{\hP^{\cM^*}(\*V_{\*z_k}) = \{\*v_{\*z_k, i}\}_{i=1}^{n_k}\}_{k=1}^{\ell}$.

One way to train $\hM$ to match $\cM^*$ in the distribution $P(\*V_{\*z_k})$ is to compare the distribution of data points in $\hP^{\cM^*}(\*V_{\*z_k})$ with the distribution of samples from $\hM$, $\hP^{\hM}(\*V_{\*z_k})$. The two empirical distributions can be compared using a divergence function $\bbD_P$, which returns a smaller value when the two distributions are ``similar''. The goal is then to minimize $\bbD_{P}(\hP^{\hM}(\*V_{\*z_k}), \hP^{\cM^*}(\*V_{\*z_k}))$ for each $k \in \{1, \dots, \ell\}$. In this work, a generative adversarial approach \citep{NIPS2014_5ca3e9b1} is taken to train the NCM, and $\bbD_P$ is computed using a discriminator network.

In addition to fitting the datasets, the second challenge of Alg.~\ref{alg:ncm-solve-ctfid} is to simultaneously maximize or minimize the query of interest $Q = P(\*y_* \mid \*x_*)$. This can be done by first computing samples of $P(\*Y_* \mid \*x_*)$ from $\hM$ via Alg.~\ref{alg:ncm-ctf-sample}, denoted $\hQ$, and then minimizing (or maximizing) the ``distance'' between $Q$ and $\hQ$. Essentially, samples from $\*Y_*$ are penalized based on how similar they are to the correct values $\*y_*$. For example, if the query to maximize is $P(Y = 1)$, and a value of 0.6 is sampled for $Y$ from $\hM$, then the goal could be to minimize squared error, $(1 - 0.6)^2$. In general, a distance metric $\bbD_Q$ is used to compute the distance between $\hQ$ and $Q$, and we use log loss for $\bbD_Q$ as our experiments involve binary variables. 
\begin{wrapfigure}{r}{0.52\textwidth}
\IncMargin{1em}
\vspace{-0.1in}
\begin{algorithm}[H]
    \scriptsize
    \renewcommand{\AlCapSty}[1]{\normalfont\scriptsize{\textbf{#1}}\unskip}

    \DontPrintSemicolon
    \SetKw{notsymbol}{not}
    \SetKwData{ncmdata}{$\widehat{M}$}
    \SetKwData{paramdata}{$\bm{\theta}$}
    \SetKwData{pdata}{$\hat{p}$}
    \SetKwData{qdata}{$\hat{q}$}
    \SetKwData{lossdata}{$\cL$}
    \SetKwFunction{ncmfunc}{NCM}
    \SetKwFunction{estimate}{Estimate}
    \SetKwFunction{consistent}{Consistent}
    \SetKwFunction{sample}{sample}
    \SetKwInOut{Input}{Input}
    \SetKwInOut{Output}{Output}
    
    \Input{ Data $\{\hP^{\cM^*}(\*V_{\*z_k}) = \{\*v_{\*z_k, i}\}_{i=1}^{n_k}\}_{k=1}^{\ell}$, query $Q = P(\*y_* | \*x_*)$, causal diagram $\cG$, number of Monte Carlo samples $m$, regularization constant $\lambda$, learning rate $\eta$, training epochs $T$}
    \BlankLine
    $\hM \gets$ \ncmfunc{$\*V, \cG$} \tcp*{from Def.~\ref{def:gncm}}
    Initialize parameters $\bm{\theta}_{\min}$ and $\bm{\theta}_{\max}$\;
    \For{$t \gets 1$ \KwTo $T$} {
        $L_{\min} \gets 0$, $L_{\max} \gets 0$\;
        \For{$k \gets 1$ \KwTo $\ell$}{
            \tcp{Sample via Alg.~\ref{alg:ncm-ctf-sample}}
            $\hP_{\min}(\*V_{\*z_k}) \gets \hM(\bm{\theta}_{\min}).$\sample{$\*V_{\*z_k}, n_k$}\;
            $\hP_{\max}(\*V_{\*z_k}) \gets \hM(\bm{\theta}_{\max}).$\sample{$\*V_{\*z_k}, n_k$}\;
            
            $L_{\min} \gets L_{\min} + \bbD_P \left( \hP_{\min}(\*V_{\*z_k}), \hP^{\cM^*} (\*V_{\*z_k}) \right)$\;
            $L_{\max} \gets L_{\max} + \bbD_P \left( \hP_{\max}(\*V_{\*z_k}), \hP^{\cM^*} (\*V_{\*z_k}) \right)$\;
        }
        $\hQ_{\min} \gets \hM(\bm{\theta}_{\min}).$\sample{$\*Y_*, m$}\;
        $\hQ_{\max} \gets \hM(\bm{\theta}_{\max}).$\sample{$\*Y_*, m$}\;
        
        \tcp{$L$ from Eq.~\ref{eq:ncm-div-loss}}
        
        $L_{\min} \gets L_{\min} - \lambda \bbD_Q \left( \hQ_{\min}, Q \right)$\;
        $L_{\max} \gets L_{\max} + \lambda \bbD_Q \left( \hQ_{\max}, Q \right)$\;
        
        $\bm{\theta}_{\min} \gets \bm{\theta}_{\min} - \eta \nabla L_{\min}$\;
        $\bm{\theta}_{\max} \gets \bm{\theta}_{\max} - \eta \nabla L_{\max}$\;
    }
    \caption{\scriptsize Training Model}
    \label{alg:ncm-learn-pv}
\end{algorithm}
\DecMargin{1em}
\vspace{-.9in}
\end{wrapfigure}
For this reason, an NCM trained with this approach will be referred as a \emph{GAN-NCM}. \footnote{Other choices of $\bbD_{P}$ include KL divergence, f-divergence, or  Maximum Mean Discrepancy (MMD) \citep{JMLR:v13:gretton12a}.} More details about architecture and hyperparameters used throughout this work can be found in Appendix \ref{app:experiments}.

Putting $\bbD_P$ and $\bbD_Q$ together, we can write that the objective $L \left(\widehat{M}, \{\hP^{\cM^*}(\*V_{\*z_k})\}_{k=1}^{\ell} \right)$ is 
\begin{equation}
    \label{eq:ncm-div-loss}
    \left(\sum_{k=1}^{\ell} \bbD_P \left(\widehat{P}^{\widehat{M}}_{\*z_k}, \widehat{P}^{\cM^*}_{\*z_k} \right)\right) \pm \lambda \bbD_Q \left( \hQ, Q \right),
\end{equation}
where $\lambda$ is initially set to a high value, and decreases during training.  Optimization may be done using gradient descent. After training, the two values of $Q$ induced by $\hM(\bm{\theta}_{\min})$ and $\hM(\bm{\theta}_{\max})$ are compared with a hypothesis testing procedure to decide identifiability. 
 Eq.~\ref{eq:ncm-mc-sampling} is used as $Q$'s estimate, whenever identifiable.

%% file: section/05_experiments.tex
\section{Experimental Evaluation} \label{sec:experiments}

\begin{figure*}
    \begin{center}
    \includegraphics[width=\textwidth,keepaspectratio]{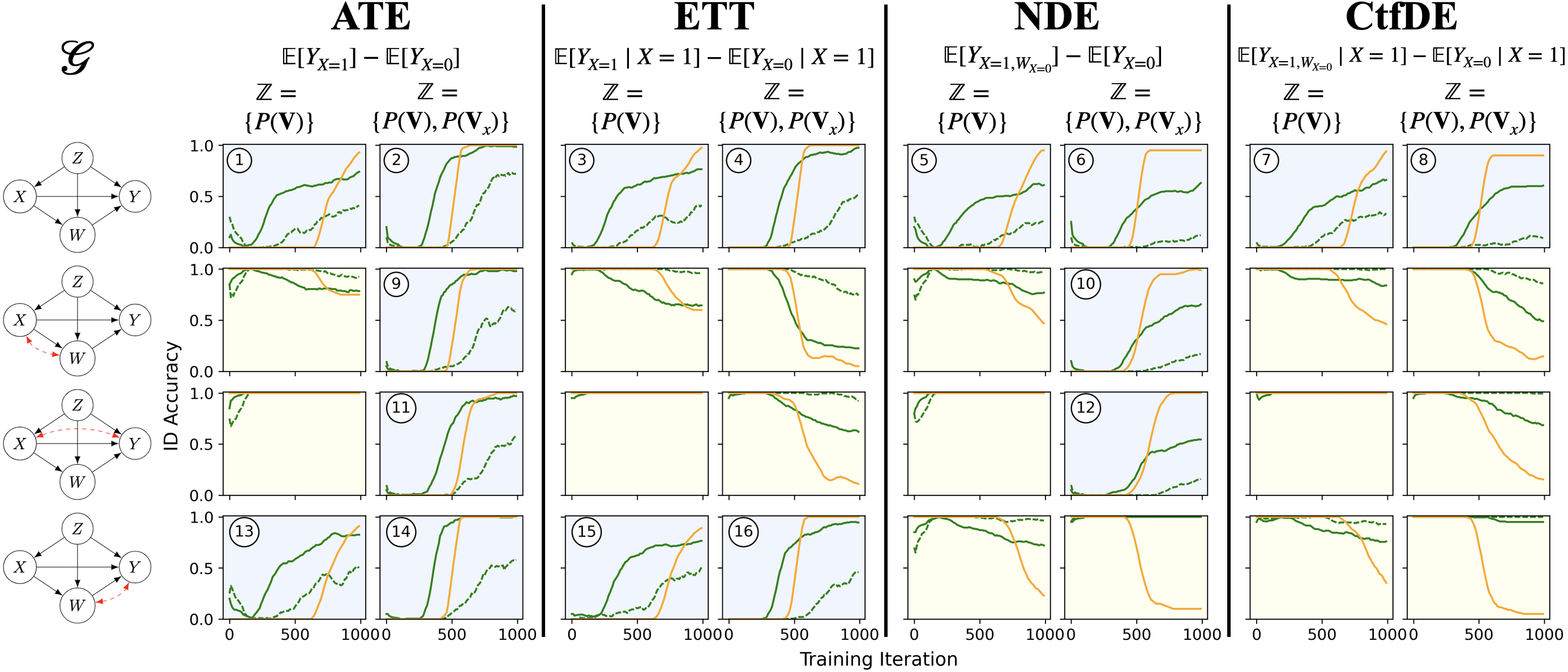}
    \caption{Experimental results on deciding identifiability on counterfactual queries with NCMs. GAN-NCM (green) is compared with MLE-NCM (orange) for settings with $d=1$. GAN-NCM performance is also shown for $d=16$ (dashed green). Blue (resp. yellow) backgrounds on plots correspond to a ground truth of ID (resp. non-ID). ID cases are numbered for reference in later plots.}
    \label{fig:gan-id-expl-results}
    \end{center}
    \vspace{-0.3in}
\end{figure*}

We first evaluate the NCM's ability to identify counterfactual distributions through Alg.~\ref{alg:ncm-learn-pv}. 
Each setting consists of a target query ($Q$), a causal diagram ($\cG$), and a set of input distributions ($\bbZ$). In total, we test 32 variations.
Specifically, we evaluate the identifiability of four queries $Q$: 
(1) Average Treatment Effect (ATE), (2) Effect of Treatment on the Treated (ETT) \citep[Eq.~8.18]{pearl:2k}, (3) Natural Direct Effect (NDE) \citep[Eq.~6]{pearl:01}, and (4) Counterfactual Direct Effect (CtfDE) \citep[Eq.~3]{zhang2018fairness}; each expression is shown on the top of Fig.~\ref{fig:gan-id-expl-results}. 
The four graphs used are shown on the figure's left side, and represent general structures found throughout the mediation and fairness literature \citep{pearl:01,zhang2018fairness}. 
The variable $X$ encodes the treatment/decision, $Y$ the outcome, $Z$ observed features, and $W$ mediating variables. 
Lastly, we consider a setting in which only the observational data is available ($\bbZ = \{P(\mathbf{V})\}$) and another in which additional experimental data on $X$ is available ($\bbZ = \{P(\mathbf{V}), P(\mathbf{V}_x)\}$). 
In the experiments shown, all variables are 1-dimensional binary variables except $Z$, whose dimensionality $d$ is adjusted in experiments.
The background color of each setting indicates that the query $Q$ is identifiable (blue) or is not identifiable (yellow) from the inputted $\cG$ and $\bbZ$. Given the sheer volume of variations, we summarize the experiments below and provide further discussion and details in Appendix  \ref{app:experiments}.

\begin{wrapfigure}{r}{0.3\textwidth}
\vspace{-0.2in}
\includegraphics[width=\linewidth]{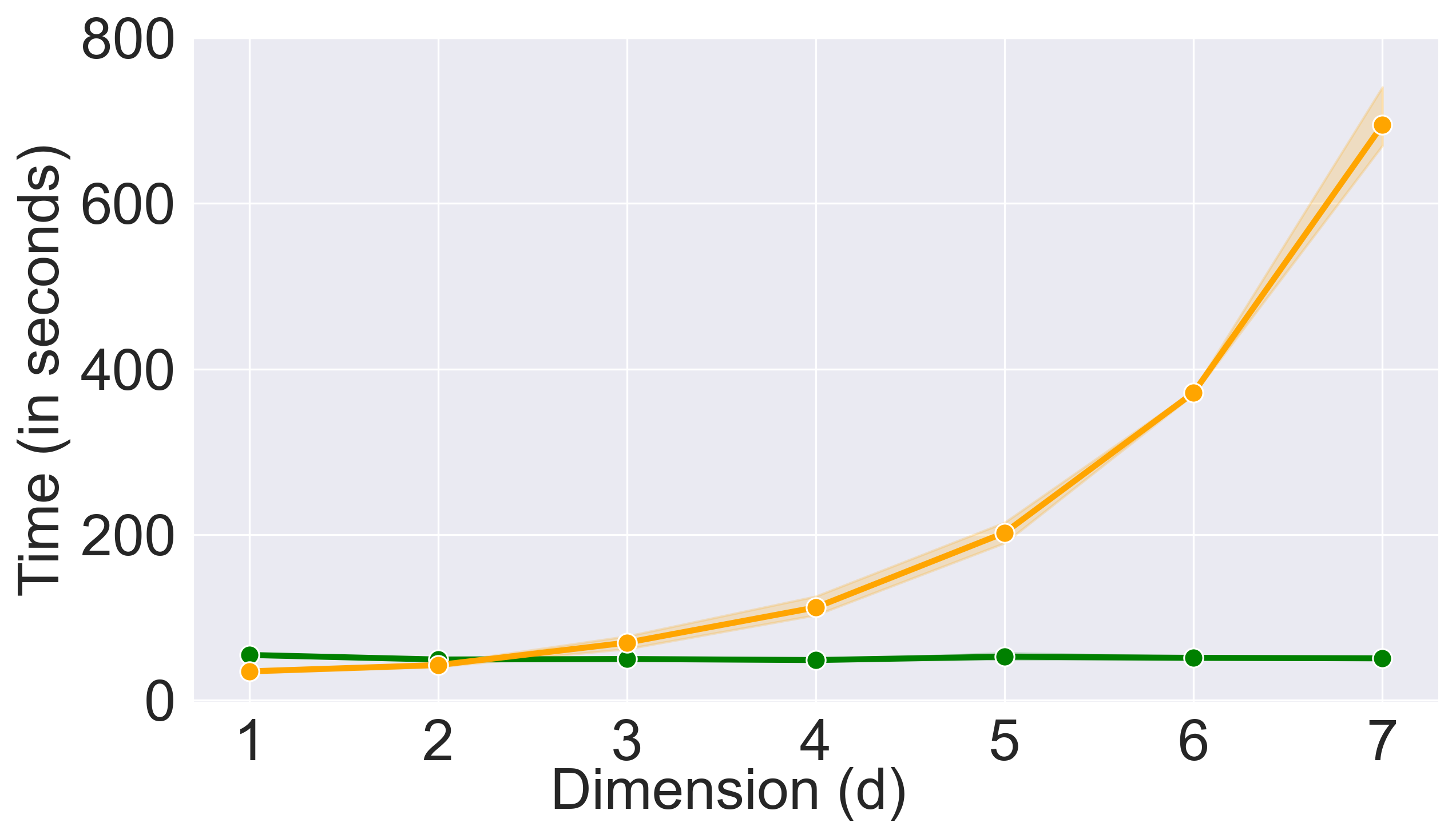}
\caption{Results comparing run times of 100 epochs of training between GAN-NCM (green) and MLE-NCM (orange) in the first graph of Fig.~\ref{fig:gan-id-expl-results} as the dimensionality $d$ of $Z$ scales higher.}
\label{fig:runtime-results}
\vspace{-0.2in}
\end{wrapfigure}

We implement two approaches to ground the discussion around NCMs, one based on GANs (\textit{GAN-NCM}), and another based on maximum likelihood (\textit{MLE-NCM}). The former was discussed in the previous section and the latter is quite natural in statistical settings. 
The experiments (Fig.~\ref{fig:gan-id-expl-results}) show that GAN-NCM has on average higher accuracy. The MLE-NCM performs slightly better in ID cases (blue), but the performance drops significantly for non-ID cases (yellow), suggesting it may be biased in returning ID for all cases. The GAN-NCM is also shown to achieve decent performance in 16-d, where the MLE-NCM fails to work. 
We plot the run time of these two approaches in Fig.~\ref{fig:runtime-results}, which shows that the MLE-NCM
scales poorly compared to the GAN-NCM; this pattern is observed in all settings. Intuitively,  this is not surprising since the MLE-NCM explicitly computes a likelihood for every value in every variable domain, the size of which grows exponentially w.r.t. the dimensionality ($d$), while the GAN-NCM avoids this by implicitly fitting distributions through $P(\widehat{\*U})$ and $\widehat{\cF}$ and directly outputting samples.

\begin{wrapfigure}{r}{0.45\textwidth}
\vspace{-0.2in}
\includegraphics[width=\linewidth]{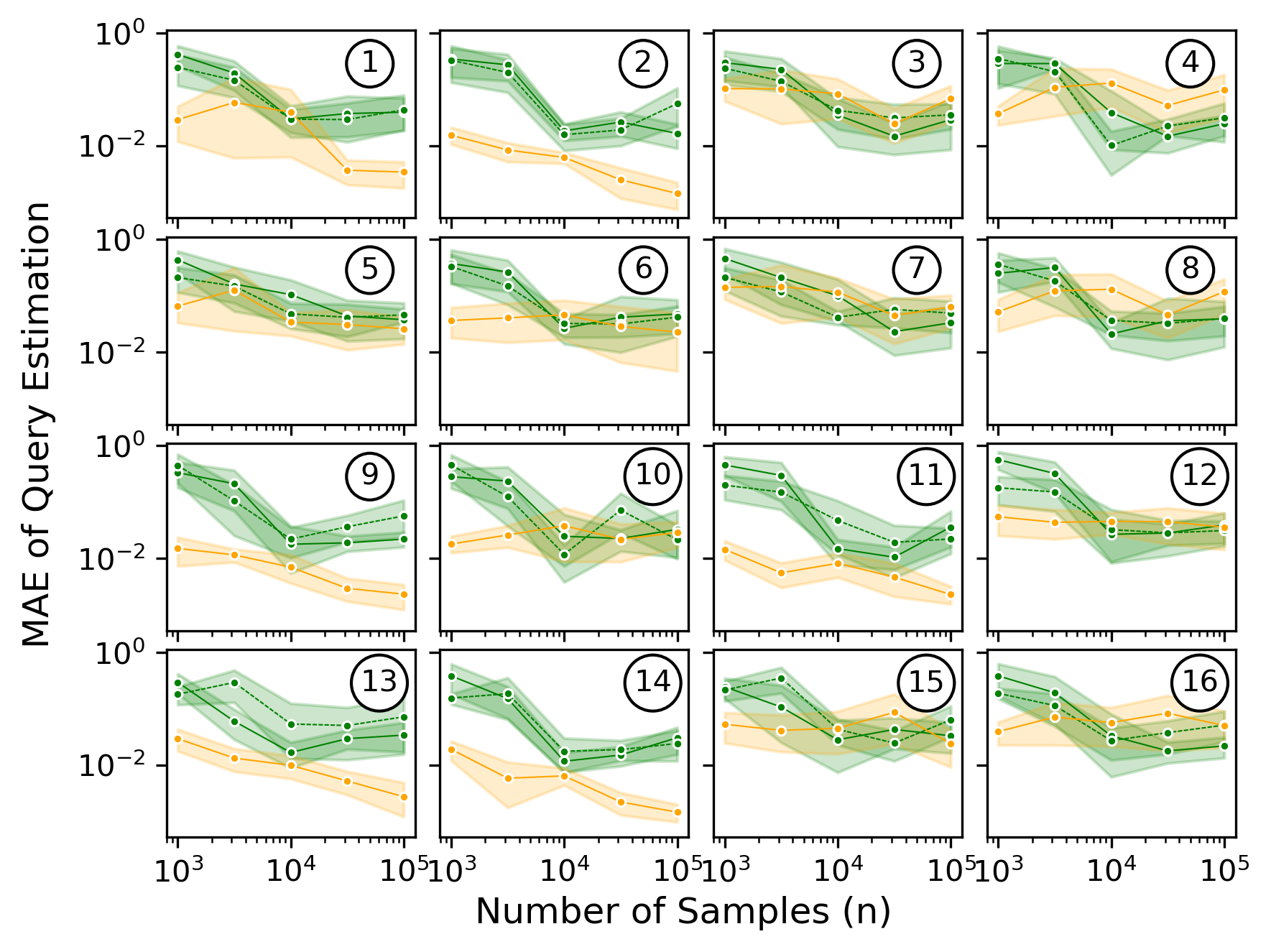}
\vspace{-0.25in}
\caption{
Results on estimating identifiable cases from Fig.~\ref{fig:gan-id-expl-results} (corresponding numbers shown on the right). Mean Absolute Error (MAE) is plotted (with 95\% confidence) for each setting for varying sample sizes. Results are shown for GAN-NCM (solid green) and MLE-NCM (orange) with $d=1$ and also GAN-NCM with $d=16$ (dashed green).}
\label{fig:gan-est-expl-results}
\vspace{-0.3in}
\end{wrapfigure}

For the identifiable cases  (blue background), the target $Q$ is estimated through Eq.~\ref{eq:ncm-mc-sampling} after training. Results are shown in Fig.~\ref{fig:gan-est-expl-results}. The MLE-NCM serves as a benchmark for 1-dimensional cases since, intuitively, the data distributions can be learned more accurately when modeled explicitly. Still, even when $d=1$, the GAN-NCM achieves competitive results in most settings and consistently achieves an error under 0.05 with more samples. The GAN-NCM is able to maintain this consistency even at $d=16$, demonstrating its robustness scaling to higher dimensions. The code will be made available after publication. 

After all, the GAN-NCM is shown to be effective at identifying and estimating counterfactual distributions even in high dimensions. As expected, the MLE-NCM may achieve lower error in some 1-d settings, but the GAN-NCM may be preferred for scalability. Moreover, an incorrect ID conclusion in a non-ID case may be dangerous for downstream decision-making as the resulting estimation will likely be incorrect or misleading. The GAN-NCM is evidently more robust in such non-ID cases while still performing competitively in ID cases. Further experiments and discussions are provided in App.~\ref{app:experiments}.

%% file: section/06_conclusion.tex
\section{Conclusions}

We developed in this work a neural approach to the problems of counterfactual identification and estimation using neural causal models (NCMs).
Specifically, we first showed that with the  graphical inductive bias, NCMs are capable of encoding  counterfactual ($\cL_3$) constraints  and are still expressive so as to represent any generating SCM (Thms.~\ref{thm:gl3-consistency}, \ref{thm:l3-g-expressiveness}). 
We then showed that NCMs have the ability of  solving any counterfactual identification instance  (Thm.~\ref{thm:ncm-ctfid-equivalence}, Corol~\ref{cor:op-id}). 
Given these theoretical properties, we introduced a sound and complete algorithm (Alg.~\ref{alg:ncm-solve-ctfid}, Corol.~\ref{thm:ncm-ctfid-correctness}) for identifying and estimating counterfactuals in  general non-Markovian settings given arbitrary datasets from $\cL_1$ and $\cL_2$. We developed an approach 
 based on GANs
to implement this algorithm in practice 
(Alg.~\ref{alg:ncm-learn-pv}) and empirically demonstrated its ability to scale inferences. From a neural perspective, counterfactual reasoning under a causal inductive bias allows for deep models to be trained with an improved understanding of interpretability and generalizability. From a causal perspective, neural nets can now provide tools to solve counterfactual inference problems previously only understood in theory.

\section{Acknowledgements}
This research was supported in part by the NSF, ONR, AFOSR, DoE, Amazon, JP Morgan,
and The Alfred P. Sloan Foundation.

%% file: section/A_proofs.tex
\section{Proofs} \label{app:proofs}

\subsection{Proof of Theorem 1}
\label{app:proof-G-cons}

First, we bring forth the longer and more formal definition of causal diagrams.

\begin{definition}[Causal Diagram {\citep[Def.~13]{bareinboim:etal20}}]
    \label{def:scm-cg}
    Consider an SCM $\cM = \langle \*U, \*V, \cF, P(\*U)\rangle$. We construct a graph $\cG$ using $\cM$ as follows:
    \begin{enumerate}[label=(\arabic*)]
        \item add a vertex for every variable in $\*V$,
        \item add a directed edge ($V_j \rightarrow V_i$) for every $V_i, V_j \in \*V$ if $V_j$ appears as an argument of $f_{V_i} \in \cF$,
        \item add a bidirected edge ($V_j \xdasharrow[<->]{} V_i$) for every $V_i, V_j \in \*V$ if the corresponding $\Ui{V_i}, \Ui{V_j} \subseteq \*U$ are not independent or if $f_{V_i}$ and $f_{V_j}$ share some $U \in \*U$ as an argument.
    \end{enumerate}
    We refer to $\cG$ as the causal diagram induced by $\cM$ (or ``causal diagram of $\cM$'' for short).
    \hfill $\blacksquare$
\end{definition}

Def.~\ref{def:gli-consistency} states that an SCM is $\cG^{(\cL_i)}$-consistent if its Layer $i$ distributions satisfy all Layer $i$ equality constraints implied by $\cG$. To clarify, we formalize this notion in the following definitions.

\begin{definition}[Equality Constraints]
    Let $\bbP$ be a collection of distributions. We denote $C(\bbP) = \{P_1 = P_2 : P_1, P_2 \in \bbP^*\}$ (where $\bbP^*$ is the set of all marginal and conditional distributions of $\bbP$) as the set of all equality constraints that hold between two distributions obtained from $\bbP$.
    \hfill $\blacksquare$
\end{definition}

\begin{definition}[$\cL_i$ Equality Constraints]
    For a graph $\cG$, let $\Omega^*(\cG)$ be the set of all SCMs that induce $\cG$ as a causal diagram. Let $\cC^{(\cL_i)} = \bigcap_{\cM \in \Omega^*(\cG)} C(\cL_i(\cM))$ be the set of all equality constraints that hold for the $\cL_i$ distributions of every SCM in $\Omega^*(\cG)$. Then, for any collection of distributions $\bbP$, we say that $\bbP$ satisfies all Layer $i$ equality constraints implied by $\cG$ if and only if $\cC^{(\cL_i)} \subseteq C(\bbP)$.
    \hfill $\blacksquare$
\end{definition}

A more detailed discussion on the intuition of these constraints can be found in Appendix \ref{app:cli-discussion}. Naturally, it is intuitive that if an SCM induces $\cG$ as a causal diagram, it must satisfy all of its constraints, as shown in the following Lemma.

\begin{lemma}
    \label{lem:scm-cons-connection}
    Let $\cM$ be an SCM and $\cG$ be its induced causal diagram. Then $\cL_3(\cM)$ satisfies all Layer 3 equality constraints implied by $\cG$.
    \hfill $\blacksquare$
\end{lemma}

\begin{proof}
    If $\Omega^*(\cG)$ is the set of all SCMs that induce $\cG$ as a causal diagram, then $\cM$ must be an element of $\Omega^*(\cG)$. Since $\cC^{(\cL_3)} = \bigcap_{\cM' \in \Omega^*(\cG)} C(\cL_3(\cM'))$ is the set of all equality constraints that hold for the $\cL_3$ distributions of every SCM in $\Omega^*(\cG)$, there will not be any constraints in $\cC^{(\cL_3)}$ that is not in $C(\cL_3(\cM))$. Hence, $\cC^{(\cL_3)} \subseteq C(\cL_3(\cM))$, so $C(\cL_3(\cM))$ satisfies all Layer $i$ equality constraints implied by $\cG$.
\end{proof}

Moreover, the graph induced by any $\cG$-constrained NCM is simply $\cG$ itself.

\begin{fact}[{\citep[Lem.~6]{xia:etal21}}]
    \label{fact:graph-equality}
    Let $\widehat{M}(\bm{\theta}) = \langle \widehat{\*U}, \*V, \widehat{\cF}, P(\widehat{\*U}) \rangle$ be a $\cG$-constrained NCM. Let $\widehat{\cG}$ be the causal diagram induced by $\hM$. Then $\widehat{\cG} = \cG$.
\end{fact}

These two results prove that $\cG$-constrained NCMs satisfy all Layer 3 constraints of $\cG$.

\glconsistency*

\begin{proof}
    This follows directly from Fact \ref{fact:graph-equality} and Lem.~\ref{lem:scm-cons-connection}.
\end{proof}

\subsection{Proof of Theorem 2}

For this proof, we expand on the technical results provided in \citet{zhang:bareinboim21b}. The paper works with a subclass of SCMs known as \emph{discrete SCMs} and proves strong results about its expressiveness. Similar to this paper, we will assume that in any SCM, the variables in $\mathbf{U}$ are independent, and unobserved confounding is modeled by sharing the same variable from $\mathbf{U}$ in the functions of multiple variables in $\mathbf{V}$.

\begin{definition}[Discrete SCM {\citep[Def.~2]{zhang:bareinboim21b}}]
    An SCM $\cM = \langle \*U, \*V, \cF, P(\*U)$ is said to be a discrete SCM if $\cD_U$ is discrete for all $U \in \*U$ and $\cD_V$ is both discrete and finite for all $V \in \*V$.
    \hfill $\blacksquare$
\end{definition}

\begin{fact}[{\citep[Thm.~1]{zhang:bareinboim21b}}]
    \label{fact:scm-to-discrete}
    Let $\Omega^*$ be the set of all SCMs and $\Omega'$ the set of discrete SCMs. For any SCM $\cM^* \in \Omega^*$ inducing causal diagram $\cG$, there exists a discrete SCM $\cM' \in \Omega'$ with finite $|\cD_{U}|$ for all $U \in \*U$ such that $\cM'$ is $\cG^{(\cL_3)}$-consistent and is $\cL_3$-consistent with $\cM^*$.
    \hfill $\blacksquare$
\end{fact}

With this result, we need only show that the $\cG$-NCM can represent any discrete SCM inducing $\cG$ with finite domains in $\*U$. We will focus on feedforward neural networks, specifically multi-layer perceptrons (MLPs) with the binary step activation function.

\begin{definition}[Multi-layer Perceptron]
    \label{def:ff-nn-binary-step}
    A neural network node is a function defined as
    $$\hat{f}(\mathbf{x}; \mathbf{w}, b) = \sigma \left(\sum_i \mathbf{w}_i \mathbf{x}_i + b\right),$$
    where $\mathbf{x}$ is a vector of real-valued inputs, $\mathbf{w}$ and $b$ are the real-valued learned weights and bias respectively, and $\sigma$ is an activation function. For this work, we will often denote $\sigma$ as the binary step function for our activation function:
    \[\sigma(z) = 
    \begin{cases}
        1 & z \geq 0 \\
        0 & z < 0.
    \end{cases}
    \]
    This is simply one choice of activation function which always outputs a binary result.
    
    A neural network layer of width $k$ is comprised of $k$ neural network nodes with the same input vector, together outputting a $k$-dimensional output:
    $$\hat{f}(\mathbf{x}; \*W, \*b) = \left( \hat{f}_1(\mathbf{x}; \mathbf{w}_1, b_1), \dots, \hat{f}_k(\*x;\mathbf{w}_k, b_k)\right),$$
    where $\*W = \{\*w_1, \dots, \*w_k\}$ and $\*b = \{b_1, \dots, b_k\}$.
    An MLP is defined as a function comprised of several neural network layers $\hat{f}_1, \dots, \hat{f}_{\ell}$, with each layer taking the previous layer's output as its input:
    $$\hat{f}_{\mlp}(\*x) = \hat{f}_{\ell}(\dots \hat{f}_1(\*x; \*W_1, \*b_1) \dots; \*W_{\ell}, \*b_{\ell}).$$
    This means that a neural network is a function that is a composition of the functions of the individual layers, where the input is the input to the first layer, and the output is the output of the last layer.
    \hfill $\blacksquare$
\end{definition}

We note that while Def.~\ref{def:ff-nn-binary-step} is undefined for non-numerical inputs and outputs, any kind of categorical data can be considered if first converted into a numerical representation. We utilize the following results about the expressivity of MLPs to demonstrate the overall expressiveness of an NCM constructed with MLPs.

\begin{fact}[MLP Representation {\citep[Lem.~4]{xia:etal21}}]
    \label{fact:discrete-to-binary}
    For any function $f: \mathbf{X} \rightarrow Y$ mapping set of variables $\mathbf{X}$ to variable $Y$, all from finite numerical domains, there exists an equivalent MLP $\hat{f}$ using binary step activations.
    \hfill $\blacksquare$
\end{fact}

\begin{fact}[Neural Inverse Probability Integral Transform (Discrete) {\citep[Lem.~5]{xia:etal21}}]
    \label{fact:unif-to-pmf}
    For any probability mass function $P(\mathbf{X})$ over finite domains, there exists an MLP $\hat{f}$ which maps $\unif(0, 1)$ to $P(\mathbf{X})$.
    \hfill $\blacksquare$
\end{fact}

We can now combine these results to achieve an expressiveness result on the counterfactual level for NCMs.

\lgexpressiveness*

\begin{proof}
    By Fact \ref{fact:scm-to-discrete}, there must exist a discrete SCM $\cM' = \langle \*U' , \*V, \cF', P(\*U') \rangle$ such that $\cM'$ is $\cG$-consistent, is $\cL_3$-consistent with $\cM^*$, and $|\cD_U|$ is finite for all $U \in \*U$. Let $\widehat{M} = \langle \widehat{\*U}, \*V, \widehat{\cF}, P(\widehat{\*U})\rangle$ be a $\cG$-NCM such that each $\hat{f}_{V} \in \widehat{\cF}$ is an MLP composed of smaller MLPs as defined next.
    
    For $U \in \*U'$, let $\*V_{U} \subseteq \*V$ denote the set of endogenous variables such that for all $V \in \*V_{U}$, $\hat{f}_V$ takes $U$ as an argument. Let $\bbC = \{\*C_1, \dots, \*C_k\}$, with $\*C_1, \dots \*C_k \subseteq \*V$ denote the maximal confounded cliques of $\cG$, that is, for all $V_1, V_2 \in \*C_i$, there is a bidirected edge between $V_1$ and $V_2$ in $\cG$, and there is no $\*C_j$ such that $\*C_i \subseteq \*C_j$. Let $\*U'_{\*C_1}, \dots, \*U'_{\*C_k}$ be a partition over $\*U'$ such that if $U \in \*U'_{\*C_i}$, then $\*V_{U} \subseteq \*C_i$. If for any $U \in \*U'$, there exist multiple cliques such that this applies, one can be chosen arbitrarily. For each $V$, let $\bbC(V) \subseteq \bbC$ denote the set of cliques that contain $V$. Then we note that
    $$\*U'_{V} \subseteq \bigcup_{\*C \in \bbC(V)}\*U'_{\*C},$$
    where $\*U'_{V} \subseteq \*U'$ denotes the exogenous parents of $V$ in $\cM'$.
    
    By construction of the $\cG$-NCM in Def.~\ref{def:gncm}, $\widehat{\*U}$ consists of a $\unif(0, 1)$ random variable $\widehat{U}_{\*C_i}$ for each maximal confounded clique $\*C_i$. By Fact \ref{fact:unif-to-pmf}, we can construct MLP $\hat{f}^{(\*U')}_{\*C_i}$ mapping $\widehat{U}_{\*C_i}$ to $\*U'_{\*C_i}$ for each $\*C_i$. Then by Fact \ref{fact:discrete-to-binary}, we can construct MLP $\hat{f}_{V}^{(\cF')}$ to map $\*U'_{V}$ and $\Pai{V}$ to $V$, matching $f'_V$. Combining these two results, $\hat{f}_V(\pai{V}, \widehat{\*u}_V)$ is defined as $\hat{f}_V^{(\cF')}(\pai{V}, \hat{f}^{(\*U')}_{{\*C_{i_1}}}(\widehat{u}_{\*C_{i_1}}), \dots , \hat{f}^{(\*U')}_{\*C_{i_{\ell}}}(\widehat{u}_{\*C_{i_{\ell}}}))$, where $\bbC(V) = \{\*C_{i_1}, \dots, \*C_{i_{\ell}}\}$. \footnote{There are a few subtleties here to align with Def.~\ref{def:ff-nn-binary-step}. First, although $\bigcup_{\*C \in \bbC(V)} \*U'_{\*C}$ may contain elements of $\*U'$ not in $\*U'_V$, $\hat{f}_{V}^{(\cF')}$ can be constructed to accept them as input and not use them (weight of 0, identity activation function). Secondly, although Def.~\ref{def:ff-nn-binary-step} does not allow additional inputs inbetween layers, $\pai{V}$ can simply be provided as an input to $\hat{f}^{(\*U')}_{\*C}$ and passed forward to the next layer without modification (weight of 1, identity activation function). Thirdly, while Def.~\ref{def:ff-nn-binary-step} is not defined to have multiple nested MLPs in the same layer, the same result can be achieved by nesting them iteratively and passing the relevant outputs through the nested layers without modification. The presentation in the proof is made for the sake of clarity.}
    
    Altogether, $\widehat{\cF}^{(\*U')} = \{\hat{f}^{(\*U')}_{\*C} : \*C \in \bbC\}$ collectively forms a neural mapping from $\widehat{\*U}$ to $\*U'$, and $\widehat{\cF}^{(\cF')} = \{\hat{f}^{(\cF')}_V : V \in \*V\}$ collectively forms a neural mapping from $\*U'$ to $\*V$. We use the notation $\cF_{\*x_i}(\*U) \models \*y_i$ to be equivalent to the statement $\*Y_{i[\*x_i]}(\*U) = \*y_i$, which is the random event (w.r.t. $\*U$) that the variables of $\*Y_i$ under intervention $\*X_i = \*x_i$ takes value $\*y_i$.
    
    Then, for any counterfactual query $Q = P(\*y_*) = P \left(\bigwedge_i \*Y_{i[\*x_i]} = \*y_i \right)$, we have
    \begin{align*}
        P^{\widehat{M}}(\*y_*) &= P\left( \bigwedge_i \widehat{\*Y}_{i[\*x_i]}(\widehat{\*U}) = \*y_i \right) \\
        &= P\left(\bigwedge_i \widehat{\cF}_{\*x_i}(\widehat{\*U}) \models \*y_i \right) \\
        &= P\left(\bigwedge_i \widehat{\cF}_{\*x_i}^{(\cF')} \left( \widehat{\cF}^{(\*U')}(\widehat{\*U}) \right) \models \*y_i \right) \\
        &= P\left(\bigwedge_i \widehat{\cF}_{\*x_i}^{(\cF')} \left( \*U' \right) \models \*y_i \right) \\
        &= P\left(\bigwedge_i \cF'_{\*x_i} \left( \*U' \right) \models \*y_i \right) \\
        &= P\left( \bigwedge_i \*Y'_{i[\*x_i]}(\*U') = \*y_i \right) \\
        &= P^{\cM'}(\*y_*) \\
        &= P^{\cM^*}(\*y_*).
    \end{align*}
    
    Hence, we have shown that for any SCM $\cM^*$ inducing causal diagram $\cG$, we can construct a $\cG$-NCM that matches $\cM^*$ on the third layer.
\end{proof}

\subsection{Proof of Theorem 3 and Corollaries 1 and 2}

Thms.~\ref{thm:gl3-consistency} and \ref{thm:l3-g-expressiveness} provide the key properties that result in the proof of Thm.~\ref{thm:ncm-ctfid-equivalence}.

\idequivalence*

\begin{proof}
    If $Q$ is identifiable from $\cG$ and $\bbZ$, it must also be identifiable from $\Omega(\cG)$ and $\bbZ$ simply because $\Omega \subset \Omega^*$. If all pairs of SCMs $\cM_1, \cM_2 \in \Omega^*$ that induce $\cG$ and match $\cM^*$ in $\bbZ$ (i.e.~$\bbZ(\cM_1) = \bbZ(\cM_2) = \bbZ(\cM^*)$) also match in $Q$ (i.e.~$P^{\cM_1}(\*y_* \mid \*x_*) = P^{\cM_2}(\*y_* \mid \*x_*)$), this should still hold if $\cM_1$ and $\cM_2$ are $\cG$-NCMs, as by Thm.~\ref{thm:gl3-consistency}, the set of all $\cG^{(\cL_3)}$-consistent SCMs includes the set of all $\cG$-NCMs.
    
    If $Q$ is not identifiable from $\cG$ and $\bbZ$, there must exist SCMs $\cM_1, \cM_2 \in \Omega^*$ such that $\cM_1$ and $\cM_2$ both induce $\cG$, $\bbZ(\cM_1) = \bbZ(\cM_2) = \bbZ(\cM^*)$, but $P^{\cM_1}(\*y_* \mid \*x_*) \neq P^{\cM_2}(\*y_* \mid \*x_*)$. By Thm.~\ref{thm:l3-g-expressiveness}, there must exist corresponding $\cG$-NCMs $\widehat{M}_1, \widehat{M}_2 \in \Omega(\cG)$ such that $\cL_3(\widehat{M}_1) = \cL_3(\cM_1)$ and $\cL_3(\widehat{M}_2) = \cL_3(\cM_2)$. This implies that $\bbZ(\hM_1) = \bbZ(\hM_2) = \bbZ(\cM^*)$ and that $P^{\hM_1}(\*y_* \mid \*x_*) \neq P^{\hM_2}(\*y_* \mid \*x_*)$, since $\hM_1$ and $\hM_2$ share all layer 3 distributions with $\cM_1$ and $\cM_2$ respectively. However, this provides an example of two NCMs that match in $\cG$ and the correct $\bbZ$ but do not match in $Q$. This proves that $Q$ is not identifiable from $\Omega(\cG)$ and $\bbZ$ if it is not identifiable from $\cG$ and $\bbZ$
    
    Therefore, since both directions hold, it must be the case that $Q$ is identifiable from $\cG$ and $\bbZ$ if and only if it is identifiable from $\Omega(\cG)$ and $\bbZ$.
\end{proof}

Consequently, any identifiable query can be computed from any proxy NCM $\widehat{M}$ trained to match $\cM^*$ in $\bbZ$ as if it were $\cM^*$ itself.

\opid*

\begin{proof}
    By Thm.~\ref{thm:l3-g-expressiveness}, there must exist $\cG$-NCM $\hM' \in \Omega(\cG)$ such that $\cL_3(\hM') = \cL_3(\cM^*)$, which implies that $\hM'$ matches $\cM^*$ in both $\bbZ$ and $Q$. By Thm.~\ref{thm:ncm-ctfid-equivalence}, since $Q$ is identifiable from $\cG$ and $\bbZ$, it must also be identifiable from $\Omega(\cG)$ and $\bbZ$, which implies that all $\cG$-NCMs that match in $\bbZ$ with $\cM^*$ must also match $\hM'$ in $Q$. In other words, for any $\cG$-NCM $\hM \in \Omega(\cG)$ such that $\bbZ(\hM) = \bbZ(\cM^*)$, it must be the case that $P^{\hM}(\*y_* \mid \*x_*) = P^{\hM'}(\*y_* \mid \*x_*)$. Hence, all such NCMs will match $\cM^*$ in $Q$. Finally, $P^{\widehat{M}}(\*y_* \mid \*x_*) = \frac{P^{\widehat{M}}(\*y_*, \*x_*)}{P^{\widehat{M}}(\*x_*)}$ can be computed from $\widehat{M}$ using Def.~\ref{def:l3-semantics}.
\end{proof}

These two results lead to the soundness and completeness of Alg.~\ref{alg:ncm-solve-ctfid}.

\completeness*

\begin{proof}
    Thm.~\ref{thm:l3-g-expressiveness} states that $Q$ is identifiable from $\cG$ and $\bbZ$ if and only if $Q$ is identifiable from $\Omega(\cG)$ and $\bbZ$. By definition, $Q$ is identifiable from $\Omega(\cG)$ and $\bbZ$ if and only if for all pairs of $\cG$-NCMs, $\hM_1, \hM_2 \in \Omega(\cG)$ with $\bbZ(\hM_1) = \bbZ(\hM_2) = \bbZ(\cM^*)$, we have $P^{\widehat{M}_1}(\*y_* \mid \*x_*) = P^{\widehat{M}_2}(\*y_* \mid \*x_*)$. This holds if and only if $P^{\widehat{M}(\bm{\theta}^*_{\min})}(\*y_* \mid \*x_*) = P^{\widehat{M}(\bm{\theta}^*_{\max})}(\*y_* \mid \*x_*)$. If they are not equal, then $\hM(\bm{\theta}^*_{\min})$ and $\hM(\bm{\theta}^*_{\max})$ are examples of two NCMs that match in $\cG$ and the correct $\bbZ$ but not in $Q$. Otherwise, all $\cG$-NCMs that match in $\bbZ$ must also match in $Q$, since any NCM that induces a result for $Q$ that is less than $P^{\hM(\bm{\theta}^*_{\min})}(\*y_* \mid \*x_*)$ or greater than $P^{\hM(\bm{\theta}^*_{\min})}(\*y_* \mid \*x_*)$ would contradict the optimality of $\bm{\theta}^*_{\min}$ and $\bm{\theta}^*_{\max}$.
    
    If $Q$ is identifiable, then Corol.~\ref{cor:op-id} guarantees that any $\cG$-NCM $\hM$ that induces the correct $\bbZ$ will match $\cM^*$ in $Q$ (i.e.~$P^{\hM}(\*y_* \mid \*x_*) =  P^{\cM^*}(\*y_* \mid \*x_*)$).
\end{proof}

%% file: section/B_experimental_details.tex
\section{Experimental Details} \label{app:experiments}

This section provides the detailed information about our experiments.
Our models are primarily written in PyTorch \citep{paszke2017automatic}, and training is facilitated using PyTorch Lightning \citep{falcon2020framework}.

\subsection{Additional Results} \label{app:more-experiments}

The accuracies from Fig.~\ref{fig:gan-id-expl-results} are computed from the max-min gaps (i.e.~the difference between the induced query of the maximizing model and the minimizing model) via a hypothesis testing procedure (see Sec.~\ref{fig:gan-id-expl-results}). The actual values of these gaps are plotted below.

\begin{figure}[h]
    \begin{center}
    \includegraphics[width=\textwidth,keepaspectratio]{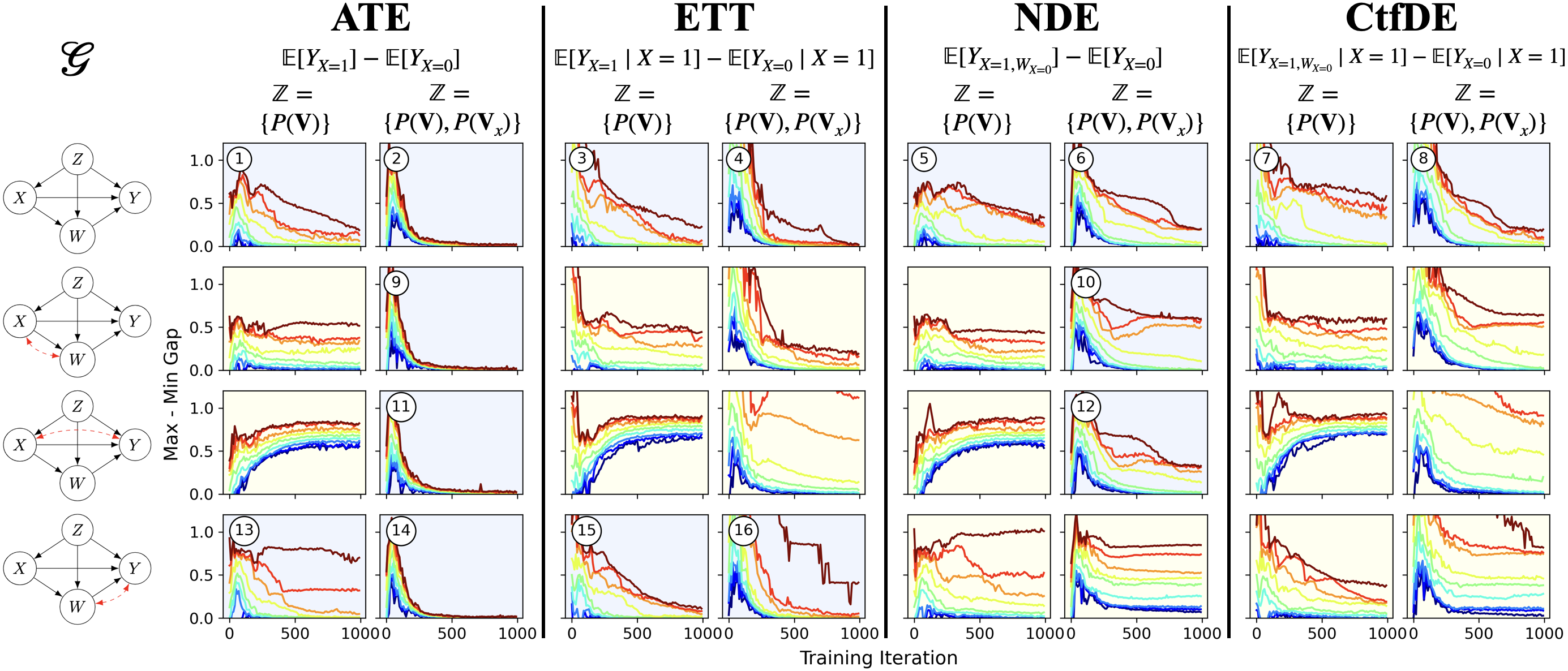}
    \caption{(1, 5, 10, 25, 50, 75, 90, 95, 99)-percentiles of max-min gaps of 1-dimensional (1D) GAN-NCM in ID experiments shown in Fig.~\ref{fig:gan-id-expl-results}.}
    \label{fig:expl-gan1d-gaps}
    \end{center}
\end{figure}
\begin{figure}[h]
    \begin{center}
    \includegraphics[width=\textwidth,keepaspectratio]{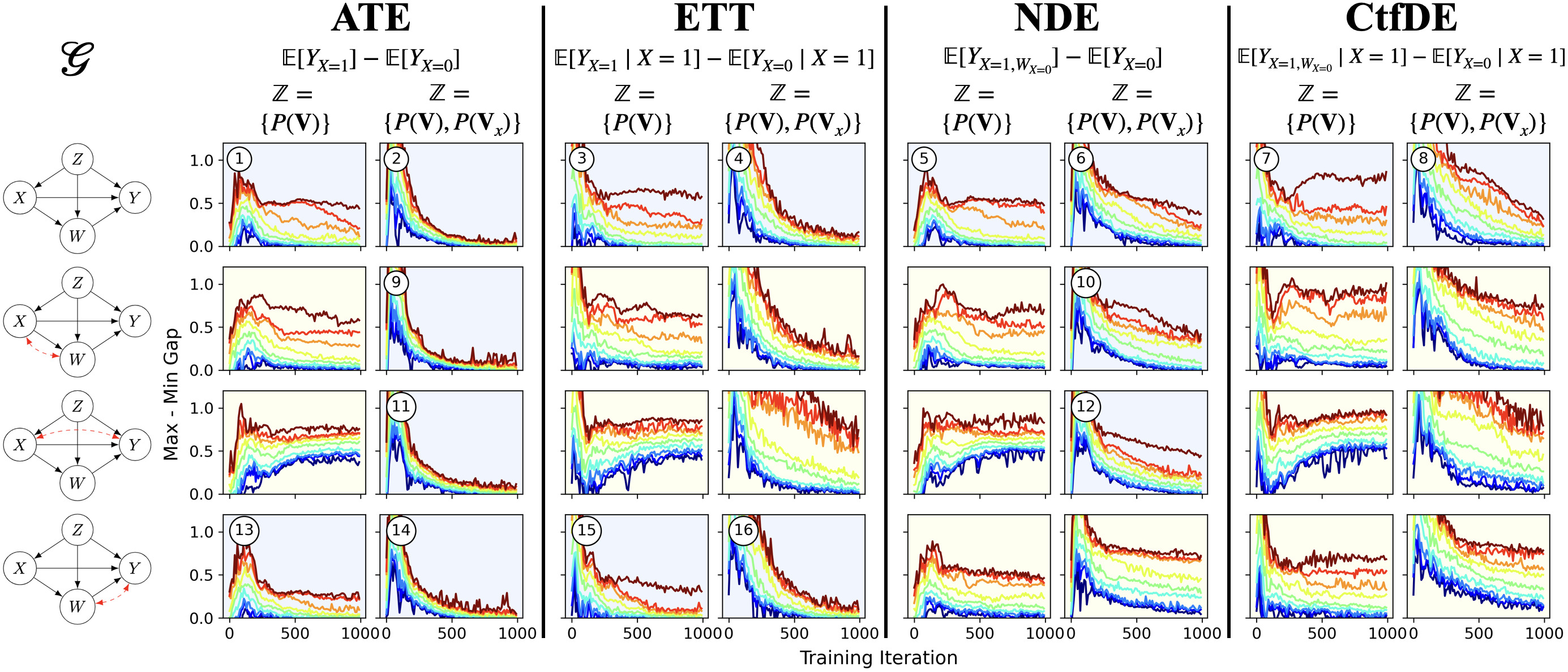}
    \caption{(1, 5, 10, 25, 50, 75, 90, 95, 99)-percentiles of max-min gaps of  16-dimensional (16D) GAN-NCM in ID experiments shown in Fig.~\ref{fig:gan-id-expl-results}.}
    \label{fig:expl-gan16d-gaps}
    \end{center}
\end{figure}
\begin{figure}[h]
    \begin{center}
    \includegraphics[width=\textwidth,keepaspectratio]{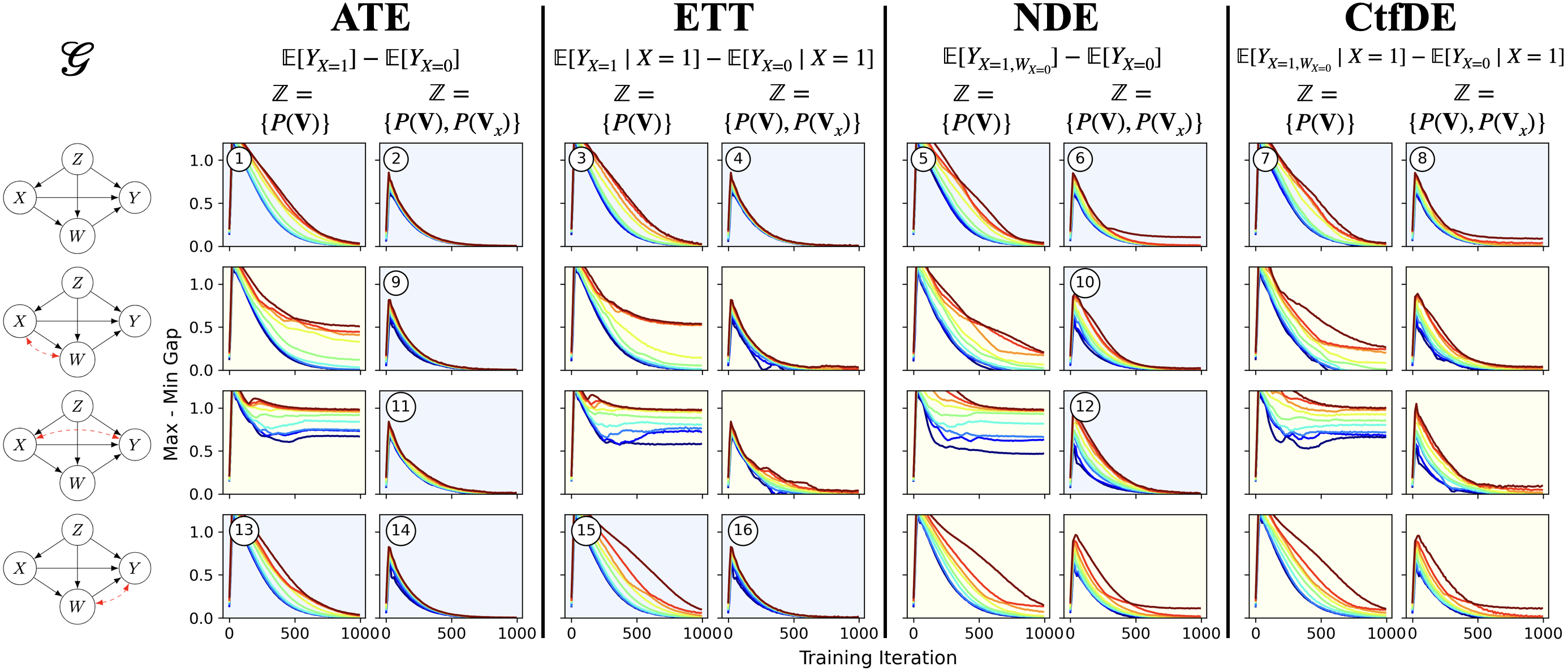}
    \caption{(1, 5, 10, 25, 50, 75, 90, 95, 99)-percentiles of max-min gaps of 1-dimensional (1D) MLE-NCM in ID experiments shown in Fig.~\ref{fig:gan-id-expl-results}.}
    \label{fig:expl-mle-gaps}
    \end{center}
\end{figure}

Note from the GAN-NCM results in Fig.~\ref{fig:expl-gan1d-gaps} that in both ID and non-ID cases, the gaps are initially large due to the large value of $\lambda$. As training progresses and $\lambda$ decreases, the gaps tighten in ID cases but remain wide in non-ID cases. In terms of the accuracy curves in Fig.~\ref{fig:gan-id-expl-results}, the result is that accuracy starts low but slowly increases in ID cases, while accuracy remains high the entire time in non-ID cases. Fig.~\ref{fig:expl-gan16d-gaps} shows similar patterns, but gaps tend to be larger in general due to the higher complexity of having higher dimensions. Positivity of data distributions may be more difficult to achieve with a smaller dataset, so the GAN-NCM leans on the safer side and predicts ``non-ID'' more often.

On the other hand, when looking at the MLE-NCM gaps from Fig.~\ref{fig:expl-mle-gaps}, that gaps often tighten during training even in non-ID cases, suggesting that MLE-NCM may have trouble finding the proper bounds in non-trivial non-ID cases.

\begin{wrapfigure}{r}{0.6\textwidth}
\vspace{-0.2in}
\includegraphics[width=\linewidth]{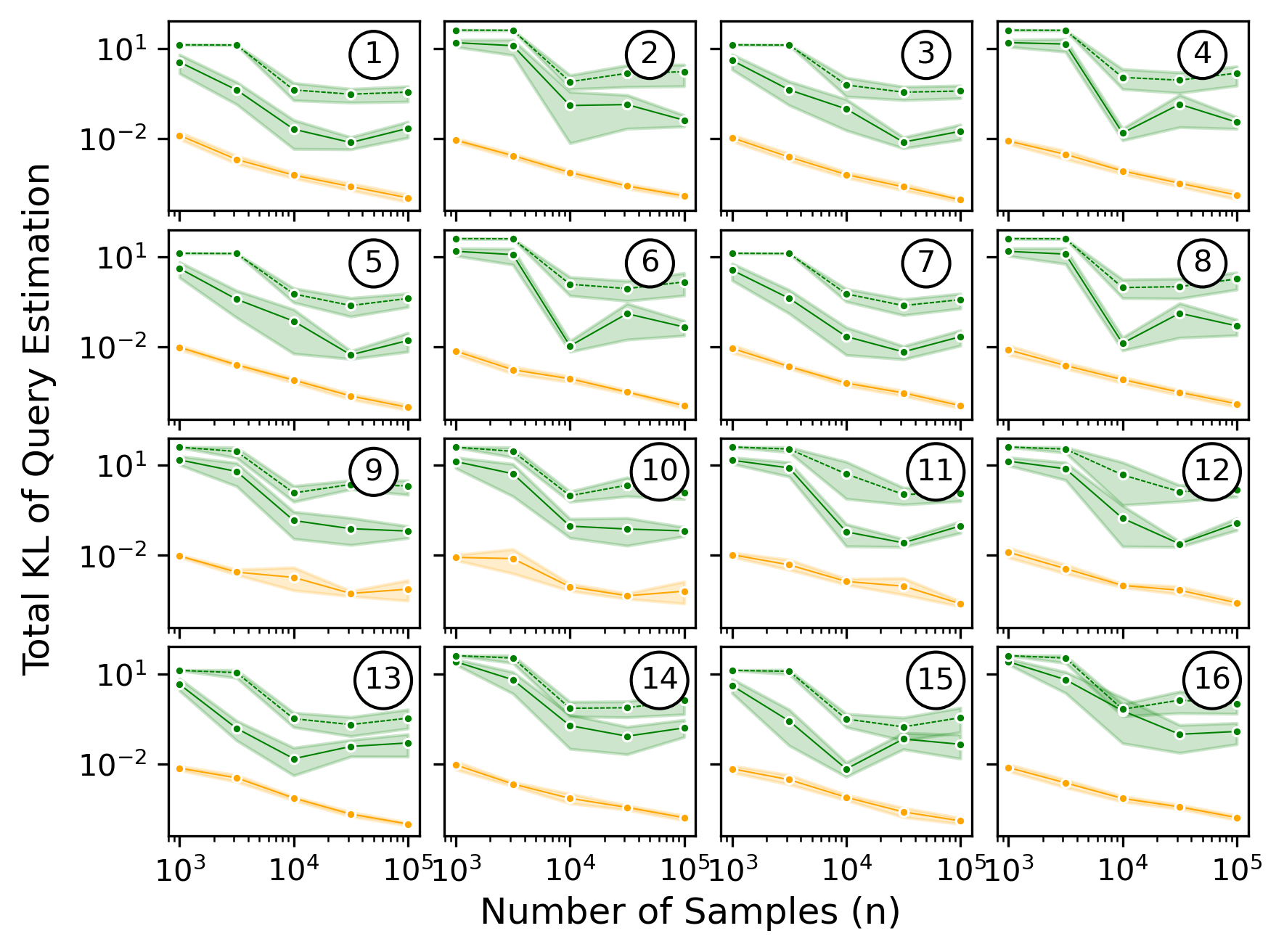}
\caption{Total KL Divergence of estimation experiments shown in Fig.~\ref{fig:gan-est-expl-results}. Results are shown for GAN-NCM (solid green) and MLE-NCM (orange) with $d=1$ and also GAN-NCM with $d=16$ (dashed green).}
\label{fig:est-expl-kl-results}
\end{wrapfigure}

For the estimation experiments shown in Fig.~\ref{fig:gan-est-expl-results}, we plot the total KL divergence comparing the induced distributions of the trained models to the true data distributions in Fig.~\ref{fig:est-expl-kl-results}. Note the downward trend in all three cases, suggesting that GAN-NCM and MLE-NCM improve in modeling the true distribution with more data (a good sanity check). MLE-NCM is more consistent and achieves slightly lower KL for the same amount of data, which also reflects in its accuracy of the induced query. GAN-NCM is evidently not as statistically efficient due to its implicit design and less stable training method, but it is capable of learning the data distribution even in high dimensions.

\begin{figure}[h]
    \begin{center}
    \includegraphics[width=\textwidth,keepaspectratio]{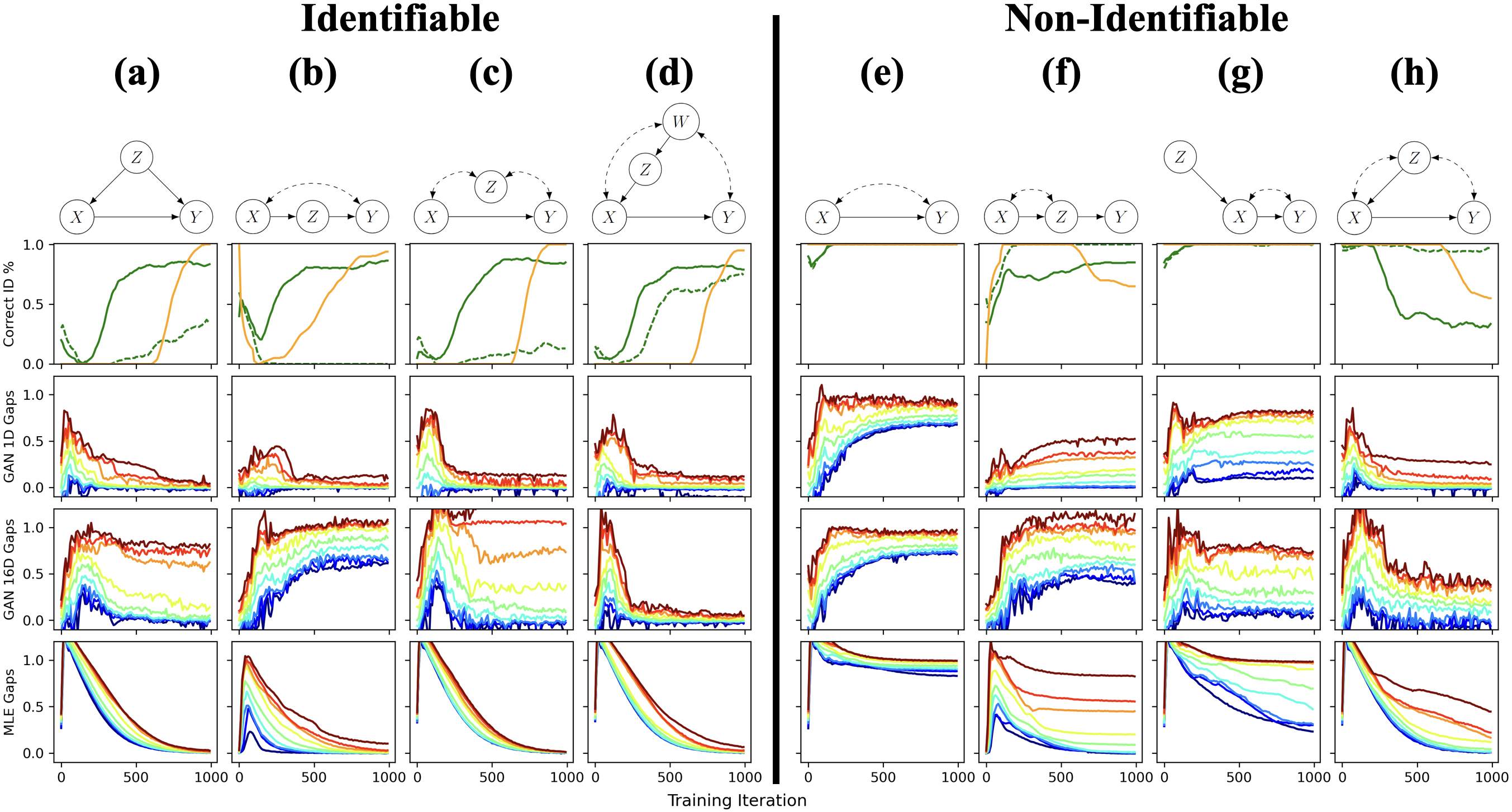}
    \caption{Identification results of the ATE query in the 8 settings from \citep{xia:etal21}. The graphs at the top from left to right are: (ID cases) back-door, front-door, M, napkin; (not ID cases) bow, extended bow, IV, bad M. In the top row of plots, the running average of classification accuracy over 1000 epochs of both GAN-NCM (green) and MLE-NCM (orange) are shown. The second, third, and fourth rows of plots show the (1, 5, 10, 25, 50, 75, 90, 95, 99)-percentiles of max-min gaps of 1D GAN-NCM, 16D GAN-NCM, and MLE-NCM respectively.}
    \label{fig:id-r80}
    \end{center}
\end{figure}

To demonstrate that this work subsumes results from layer 2, we demonstrate the ability of the GAN-NCM to solve the ID and estimation problems in the same 8 settings from Fig.~4 of \citep{xia:etal21}. Results are shown in Fig.~\ref{fig:id-r80}, compared alongside the MLE-NCM which is the original approach used. Recall that the query tested is the average treatment effect (ATE) equal to $P(Y_{X=1} = 1) - P(Y_{X=0} = 1)$. Results are comparable to results from Fig.~\ref{fig:gan-id-expl-results}. Like the MLE-NCM, the GAN-NCM starts with large gaps, which slowly tighten in ID cases but stay wide in non-ID cases. Accuracy is also similar.

Fig.~\ref{fig:est-r80} shows the estimation results from estimating the four identifiable cases from Fig.~\ref{fig:id-r80}. Estimation results are similar to that of the counterfactual cases shown in Fig.~\ref{fig:gan-est-expl-results}. As expected, KL divergence decreases with data for all three settings, but training is slightly more efficient with the MLE-NCM. This is reflected in the slightly reduced error on the query under the same amount of data. Still, GAN-NCM is a viable method for estimating ATE in these cases, even in high dimensions.

\begin{wrapfigure}{r}{0.6\textwidth}
\vspace{-0.2in}
\includegraphics[width=\linewidth]{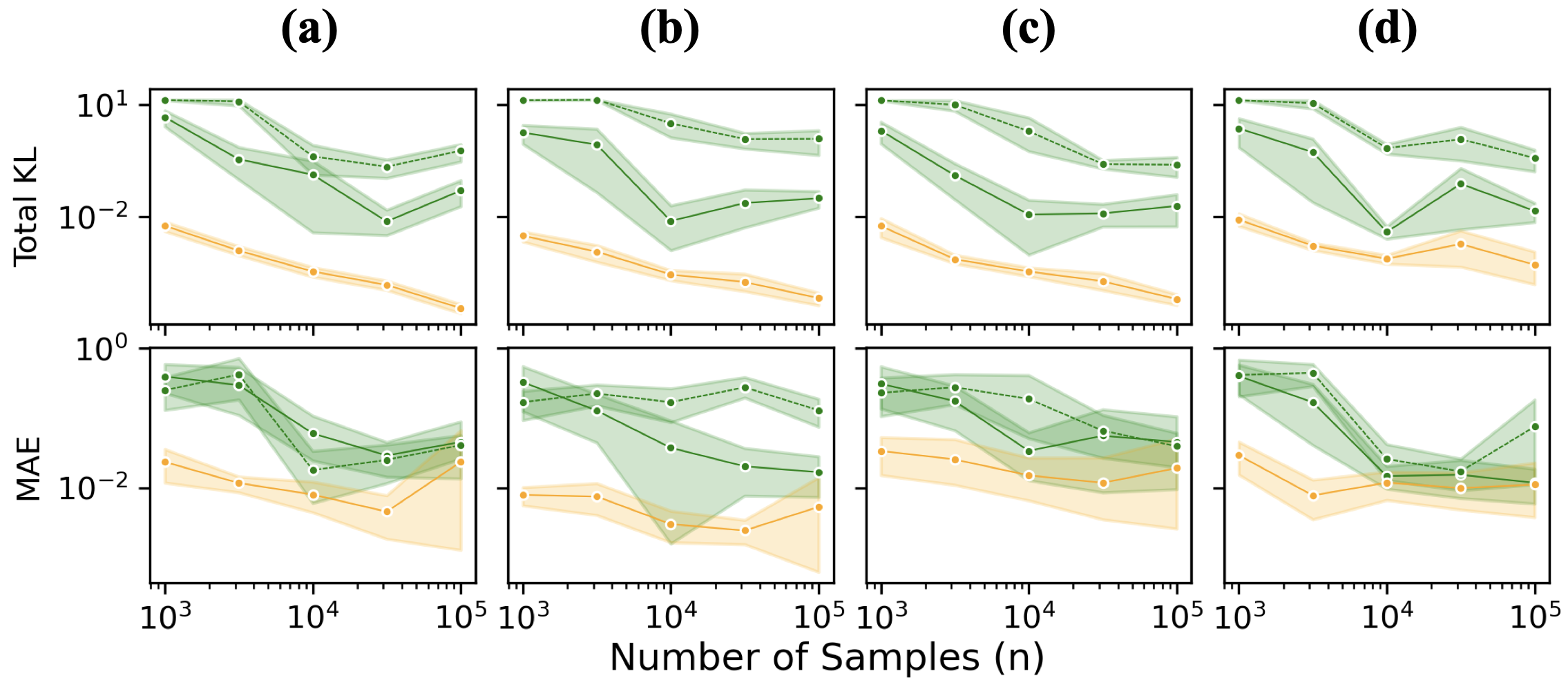}
\caption{Estimation results for identifiable cases shown in Fig.~\ref{fig:id-r80}. Results are shown for GAN-NCM (solid green) and MLE-NCM (orange) with $d=1$ and also GAN-NCM with $d=16$ (dashed green). Top row is KL divergence between the true $P(\*V)$ and the one induced by the trained NCM. Bottom row is mean absolute error (MAE) of the induced query, ATE.}
\label{fig:est-r80}
\vspace{-0.2in}
\end{wrapfigure}

As mentioned earlier, the MLE-NCM may be problematic for solving the ID problem due to its apparent bias for returning ID conclusions even in non-ID cases. This is obviously undesirable, especially since attempting to estimate a quantity in a non-ID case will likely produce a misleading result. This behavior is not evident in Fig.~\ref{fig:id-r80}, so one may be tempted to conclude that this is only an issue for complicated counterfactual quantities and not simple layer 2 quantities like ATE. However, this is not the case.

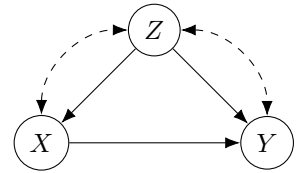
\begin{wrapfigure}{r}{0.3\textwidth}
\vspace{-0.2in}
    \centering
    \begin{tikzpicture}[xscale=1.5, yscale=1.5]
		\node[draw, circle] (X) at (-1, 0) {$X$};
		\node[draw, circle] (Z) at (0, 1) {$Z$};
		\node[draw, circle] (Y) at (1, 0) {$Y$};
		
		\path [-{Latex[length=2mm]}] (Z) edge (X);
		\path [-{Latex[length=2mm]}] (Z) edge (Y);
		\path [-{Latex[length=2mm]}] (X) edge (Y);
		\path [dashed, {Latex[length=2mm]}-{Latex[length=2mm]}, bend left, out=45, in=135] (X) edge (Z);
		\path [dashed, {Latex[length=2mm]}-{Latex[length=2mm]}, bend left, out=45, in=135] (Z) edge (Y);
	\end{tikzpicture}
	\caption{Back-door + M graph.}
	\label{fig:bdm}
\vspace{-0.2in}
\end{wrapfigure}

To witness, consider the back-door + M graph in Fig.~\ref{fig:bdm}, created by combining the edges of the back-door and M graphs (graphs a and c in Fig.~\ref{fig:id-r80}). It is also the bad M graph (graph h in Fig.~\ref{fig:id-r80}) with an additional edge from $Z$ to $Y$. In this setting, $P(y_x)$ (and therefore ATE) is not identifiable from $P(\*V)$ and $\cG$. We test the GAN-NCM and MLE-NCM to identify this setting in Fig.~\ref{fig:bdm-id-results}.

\begin{figure}[h]
    \centering
	\begin{subfigure}[t]{0.24\textwidth}
	    \centering
	    \includegraphics[width=\textwidth,height=0.5\textheight,keepaspectratio]{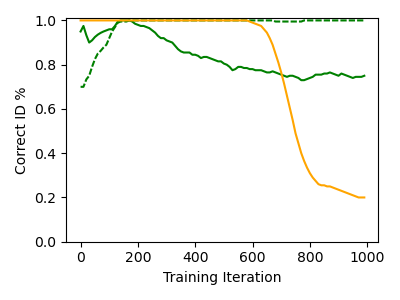}
        \caption{ID accuracy.}
        \label{fig:bdm-acc}
	\end{subfigure}
	\begin{subfigure}[t]{0.24\textwidth}
	    \centering
	    \includegraphics[width=\textwidth,height=0.5\textheight,keepaspectratio]{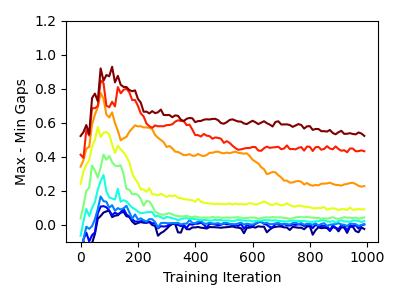}
        \caption{GAN-NCM 1D gaps.}
        \label{fig:bdm-gan1d-gaps}
	\end{subfigure}
	\begin{subfigure}[t]{0.24\textwidth}
	    \centering
	    \includegraphics[width=\textwidth,height=0.5\textheight,keepaspectratio]{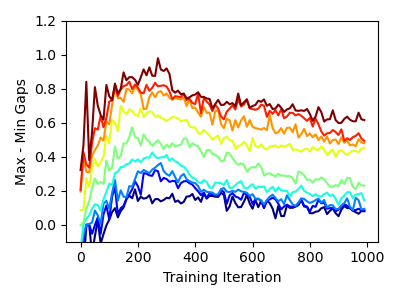}
        \caption{GAN-NCM 16D gaps.}
        \label{fig:bdm-gan16d-gaps}
	\end{subfigure}
	\begin{subfigure}[t]{0.24\textwidth}
	    \centering
	    \includegraphics[width=\textwidth,height=0.5\textheight,keepaspectratio]{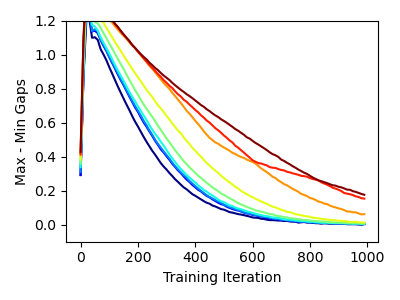}
        \caption{MLE-NCM 1D gaps.}
        \label{fig:bdm-mle-gaps}
	\end{subfigure}
	\caption{Identification results for the ATE query from $P(\*V)$ and the graph $\cG$ in Fig.~\ref{fig:bdm}. The four plots shown from left to right are ID accuracy and (1, 5, 10, 25, 50, 75, 90, 95, 99)-percentile max-min gaps for 1D GAN-NCM (green), 16D GAN-NCM (dashed green), and 1D MLE-NCM (orange) over 1000 training epochs.}
	\label{fig:bdm-id-results}
\end{figure}

Evidently, the same bias of the MLE-NCM of incorrectly returning an identifiable result occurs here. As training progresses, the gaps begin to close (shown in Fig.~\ref{fig:bdm-mle-gaps}) despite the query being non-identifiable. This results in a very poor accuracy towards the end of training (shown through orange line in Fig.~\ref{fig:bdm-acc}). The GAN-NCM handles this robustly in both 1 and 16 dimensions, having large gaps all the way throughout training (Figs.~\ref{fig:bdm-gan1d-gaps} and \ref{fig:bdm-gan16d-gaps}) and maintaining a high accuracy (green lines in Fig.~\ref{fig:bdm-acc}).

We do not make any theoretical claim that explains this phenomenon. From experimentation, it seems that MLE-NCM is more easily able to model associations between variables through directed edges instead of bidirected edges. That is, if $X \rightarrow Y$ and $X \leftrightarrow Y$, the MLE-NCM may be biased to modeling the association between $X$ and $Y$ by directly using the value of $x$ in $\hat{f}_Y$ as opposed to training $\hat{f}_X$ and $\hat{f}_Y$ to use the shared $U_{XY}$ to model this association. As a result, the MLE-NCM may be treating Fig.~\ref{fig:bdm} as a simple back-door case, and the optimization term to maximize/minimize the query may not be strong enough to override this bias.

\subsection{Nested Counterfactuals} \label{app:nested-ctf}

A more general definition of counterfactuals can be defined as shown below (see \citep{correa:etal21} for a more detailed discussion).

\begin{definition}[Layer 3 Valuation (General)] 
\label{def:l3-semantics-nested}
An SCM $\cM$ induces layer $\cL_3(\cM)$, a set of distributions over $\*V$, each with the form $P(\*Y_*) = P(\*Y_{1[\widetilde{\*X}_1]}, \*Y_{2[\widetilde{\*X}_2], \dots})$ such that 

\begingroup\abovedisplayskip=0.5em\belowdisplayskip=0pt

 \begin{align}
    \label{eq:def:l3-semantics-nested}
    P^{\cM}(\*y_{1[\widetilde{\*X}_1]}, \*y_{2[\widetilde{\*X}_2]}, \dots) = 
    \int_{\cD_{\mathbf{U}}} \mathbbm{1}\left[\*Y_{1[\widetilde{\*X}_1]}(\*u)=\*y_1, \*Y_{2[\widetilde{\*X}_2]}(\*u) = \*y_2, \dots \right] dP(\*u),
\end{align}
where $\widetilde{\*X}_i$ is a collection of functions defined over variables $\*X_i$ (such that we have $\tilde{f}_{X}$ mapping from some $\widetilde{\*U}_X \subseteq \*U$ to $\cD_X$ for each $\tilde{f}_X \in \widetilde{\*X}_i$), and ${\*Y}_{i[\widetilde{\*X}_i]}(\*u)$ (also called \emph{potential response} \citep[Def.~7.4.1]{pearl:2k}) is the solution for $\*Y$ after evaluating 
 $\mathcal{F}_{\widetilde{\*X}_i} := \{f_{V_j} : V_j \in \*V \setminus \*X\} \cup \{f_X \leftarrow \tilde{f}_X : \tilde{f}_X \in \widetilde{\*X}_i\}$ with $\*u$.

\endgroup
\hfill $\blacksquare$
\end{definition}
 
$\*Y_{\widetilde{\*X}}$ is considered a natural intervention if each $\tilde{f}_X \in \widetilde{\*X}$ is a potential response for $X$ in $\cM$. If $\tilde{f}_X$ for variable $X$ is set to some fixed value $x$, then that corresponds to an atomic intervention, which is what is considered in Def.~\ref{def:l3-semantics-nested}. However, one could also have a natural intervention where $\tilde{f}_X := X_{\widetilde{\*Z}}(\*u)$ for some $\widetilde{\*Z}$, defined recursively. In other words, the counterfactual terms may be \emph{nested}.

For example, consider the term $P(Y_{X = 1, W_{X = 0}} = 1)$ in the natural direct effect (NDE) query. This would be evaluated as
\begin{align*}
    P(Y_{X = 1, W_{X = 0}} = 1) &= \int_{\cD_{\*U}} \mathbbm{1} \left[ Y_{X=1, W_{X=0}} (\*u) = 1\right] dP(\*u) \\
    &= \int_{\cD_{\*U}} \mathbbm{1} \left[ Y_{X=1, W = W_{X=0}(\*u)} (\*u) = 1\right] dP(\*u).
\end{align*}
Intuitively, we are aggregating all settings of $\*u$ in which $Y = 1$ under the intervention where $X$ is fixed to $1$ and $W$ is fixed to the value that it would take under the same setting of $\*u$ under the intervention where $X$ is fixed to $0$.

It turns out that all nested counterfactuals can be unrolled into an unnested form via the Counterfactual Unnesting Theorem (CUT).

\begin{fact}[{Counterfactual Unnesting Theorem, \citep[Thm.~1]{correa:etal21}}]
    Let $\widetilde{\*X}$, $\widetilde{\*Z}$ be any natural interventions on disjoint sets $\*X, \*Z \subseteq \*V$. Then for $\*Y \subseteq \*V$ disjoint from $\*X$ and $\*Z$, we have
    \begin{equation}
        \label{eq:cut}
        P(\*Y_{\widetilde{\*Z}, \widetilde{\*X}} = \*y) = \sum_{\*x \in \cD_{\*X}} P(\*Y_{\widetilde{\*Z}, \*x} = \*y, \widetilde{\*X} = \*x)
    \end{equation}
\end{fact}

This theorem simplifies the scope of theoretical studies on counterfactual identification, and indeed, one could perform inferences on such queries using the NCM by performing it on the terms of the unnested version. However, the proxy model nature of the NCM allows it to directly evaluate nested counterfactuals according to the definition. In the case of $P(Y_{X=1, W_{X=0}} = 1)$, for example, this can be estimated by sampling several instances of $\widehat{\*U}$, evaluating $W$ in those instances under the intervention $X=0$, and then subsequently evaluating $Y$ under the intervention $X = 1$ and with $W$ set to the corresponding values obtained previously. This dramatically simplifies the evaluation of nested counterfactuals, which otherwise would have required an analytical understanding to unnest the query manually, followed by requiring an evaluation on all terms in the sum in the r.h.s. of Eq.~\ref{eq:cut}. In practice, this sum takes a long time (possibly intractable in continuous cases) to compute, with compounding error in each term. A visualization of the improvement is shown in Fig.~\ref{fig:natural-vs-unnest}.

\begin{figure}[h]
    \centering
	\begin{subfigure}[t]{0.4\textwidth}
	    \centering
	    \includegraphics[width=\textwidth,height=0.5\textheight,keepaspectratio]{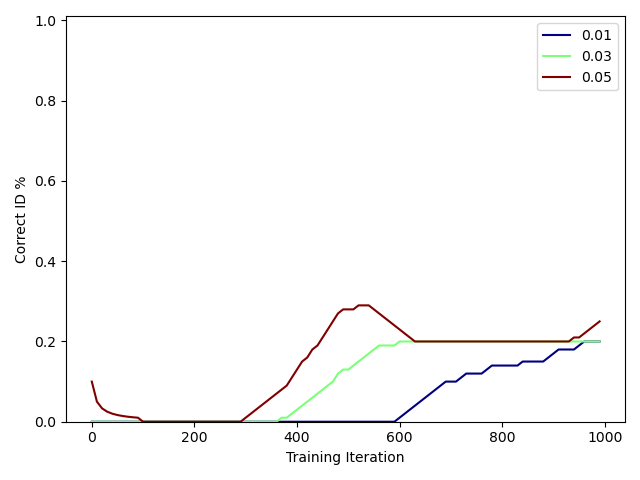}
        \caption{Unnested Evaluation}
        \label{fig:nde-unnested}
	\end{subfigure}
	\quad
	\begin{subfigure}[t]{0.4\textwidth}
	    \centering
	    \includegraphics[width=\textwidth,height=0.5\textheight,keepaspectratio]{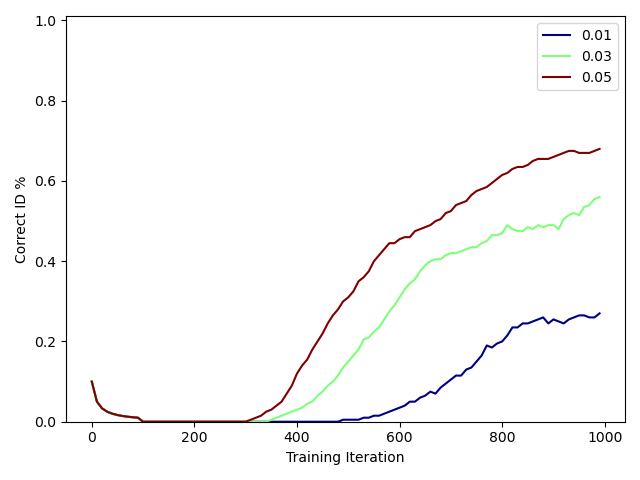}
        \caption{Natural Evaluation}
        \label{fig:nde-nested}
	\end{subfigure}
	\caption{A comparison of identifying NDE from $\bbZ = \{P(\*V), P(\*V_x)\}$ and $\cG$, with $\cG$ being the second graph of Fig.~\ref{fig:gan-id-expl-results}. Three different choices of $\tau$ are shown. Left figure shows results from evaluating the unnested form of NDE, while right figure shows results from performing the natural evaluation of NDE by definition.}
	\label{fig:natural-vs-unnest}
\end{figure}

\subsection{Data Generation}
\label{app:data-gen}

To generate data for our experiments, we use an implementation of what we call a \emph{regional canonical model}, inspired by canonical SCMs \citep{zhang:bareinboim21b}.

\begin{definition}[Regional Canonical Model (RCM)]
    Given a causal diagram $\cG$, a \emph{regional canonical model} $\cM_{\rcm}$ associated with $\cG$ is an SCM $\langle \*V, \*U_{\rcm}, \cF_{\rcm}, P(\*U_{\rcm}) \rangle$ where:
    \begin{enumerate}
        \item $\*U_{\rcm} = \{R_{\*C} : \forall \*C \in \bbC(\cG)\}$, where $\bbC(\cG)$ is the set of maximally confounded cliques of $\cG$ and $R_{\*C} \sim \unif(0, 1)$ for all $R_{\*C} \in \*U_{\rcm}$.
        
        \item For each $V \in \*V$, denote $\*R_V = \{R_{\*C} : \forall \*C \text{ s.t. } V \in \*C\}$. Then, each $V$ is associated with a set of regions $\*I_V = \{\*I_{V, i}\}_{i=1}^{r}$ such that each $\*I_{V, i} = \{[a, b]_{\*C, i} : R_{\*C} \in \*R_V\}$ with $a, b \in [0, 1]$, $a < b$. In words, each $\*I_{V, i}$ is a collection of intervals within 0 to 1 for each of the exogenous parents in $\*R_V$. Furthermore, each $\*I_{V, i}$ is associated with a corresponding function $h_{V, i}: \cD_{\Pai{V}} \rightarrow \cD_{V}$ mapping to values of $V$ from its parents (in $\cG$). Here, $r$ is a hyperparameter which decides the complexity of the model.
        
        \item Each $f^{\rcm}_V \in \cF_{\rcm}$ is defined as
        \begin{equation*}
            f^{\rcm}_V(\pai{V}, \*r_V) =
            \begin{cases}
                h_{V, i^*}(\pai{V}) & i^* \text{ is the largest } i \text{ s.t. } r_{\*C} \in [a, b]_{\*C, i} \quad \forall r_{\*C} \in \*r_{V} \\
                h_{V, \textsc{default}}(\pai{V}) & \text{ if no such } i^* \text{ exists.}
            \end{cases}
        \end{equation*}
    \end{enumerate}
\end{definition}

Explaining this model in words, a uniformly random exogenous variable is assigned to each confounded clique ($R_{\*C}$), and all variables within that clique will take that variable as an exogenous input to its functions. Each variable $V$ is then associated with a set of exogenous variables $\*R_V$ depending on the cliques of which it is a member. The space of the domains of $\*R_V$ form a unit hypercube with dimension $|\*R_V|$, and each $V$ is assigned $r$ rectangular ``regions'' in this hypercube ($\*I_V$). Each region is associated with a function for $V$ ($h_{V, i}$) given its parents, which is the behavior that describes $f^{\rcm}_V$ when the inputted $\*r_V$ falls within that region. If $\*r_V$ does not fall into any of the $r$ regions, then $f^{\rcm}_V$ will take a default behavior described by $h_{V, \textsc{default}}$ instead.

For our experimental settings, we randomly sample an RCM given the graph $\cG$ to be the ground truth, $\cM^*$, with $r=20$. Each interval $[a, b]_{\*C, i}$ is sampled by sampling two numbers from $\unif(0, 1)$ and sorting them. Since all variables are binary vectors, $h_{V, \textsc{default}}$ and each $h_{V, i}$ map each possible input of $\pai{V}$ to a randomly sampled binary vector (each bit distributed $\Bernoulli(0.5)$) corresponding to a value of $V$. All variables in the simulations are single-dimensional binary variables except for $Z$, which is a $d$-dimensional binary vector depending on the setting of $d$ in the experiment. To simulate the ground truth, samples can be drawn from the canonical SCM using the same procedure from Def.~\ref{def:l3-semantics}.

Similar to the canonical SCM in \citet{zhang:bareinboim21b}, the RCM assigns different functional behavior to each variable given its endogenous parents depending on the value of their exogenous parents. As implicated in the proof of Thm.~2.3 in \citet{zhang:bareinboim21b}, such a model is very expressive with sufficiently large variation in the space of exogenous parents, eventually being able to express any SCM on all three layers. In our case, this can be achieved with a large enough choice of $r$. This provides a richer experimental environment when compared to choosing a fixed ground truth $\cM^*$ or sampling $\cM^*$ from a more restricted parametric family of models (which may produce biased results depending on the family).

\subsection{Hypothesis Testing} \label{app:hyp-test}
Alg.~\ref{alg:ncm-solve-ctfid} is a search within the $\Omega(\cG)$ space for two \ncm{} parameterizations, $\bm{\theta}^*_{\min}$ and $\bm{\theta}^*_{\max}$, that minimizes/maximizes the desired query, respectively. After optimization is finished, the corresponding $P(\*y_* \mid \*x_*)$ of the two models can be compared to determine whether an effect is identifiable. With perfect optimization and unbounded resources, these two quantities will be equal in identifiable cases. In practice, the two quantities are unlikely to be exactly equal due to noise in the training procedure, so we rely on a hypothesis testing step such as 
\begin{equation}
    \label{eq:gap-test}
|f(\widehat{M}(\bm{\theta}_{\mathrm{max}})) - f(\widehat{M}(\bm{\theta}_{\mathrm{min}}))| < \tau
\end{equation}
for quantity of interest $f$ and a certain threshold $\tau$. This threshold is similar to a significance level in statistics and can be adjusted depending on the severity of implications of false positives or negatives. In general, $\tau$ can be determined by analyzing empirical results in which the identification ground truth is known, and we choose $\tau=0.05$ in our experiments for all cases.

\subsection{GAN-NCM Experiments}

In this section, we describe the architecture of the implicit GAN-NCM as described in Sec.~\ref{sec:ncm-estimation}.

\subsubsection{Model Architecture} \label{app:architecture}

The implementation of the GAN-NCM used in the experiments specifically utilizes Wasserstein GANs with gradient penalty \citep{gulrajani2017improved}. This choice was made to avoid the problem of mode collapse, a frequent problem that arises when training GANs. Wasserstein GANs (WGANs) \citep{arjovsky2017wasserstein} minimize Earth-Mover distance via the  Kantorovich-Rubinstein duality \citep{villani2009optimal}, producing the objective:
\begin{equation}
    \min_G \max_{D \in \cD_D} \bbE_{\*x \sim P_r}[D(\*x)] - \bbE_{\tilde{\*x} \sim P_g}[D(\tilde{\*x})],
\end{equation}
where $G$ is a generating model, $D$ is a discriminatory model (called a critic in this instance), $\cD_D$ is the set of 1-Lipshitz functions, $P_r$ is a real distribution (i.e. from $\cM^*$), and $P_g$ is the distribution induced by $G$.

Wasserstein GAN with gradient penalty (WGAN-GP) removes the hard 1-Lipschitz constraint over $\cD_D$ and enforces this constraint softly by using an additional penalty term on the gradient of the critic in the objective, namely
\begin{equation}
    L = \bbE_{\tilde{\*x} \sim P_g}[D(\tilde{\*x})] - \bbE_{\*x \sim P_r}[D(\*x)] + \lambda_{\gp}\bbE_{\hat{\*x} \sim P_{\gp}} \left[ \left( \norm{\nabla_{\hat{\*x}} D(\hat{\*x})}_2 - 1 \right)^2 \right],
\end{equation}
where data points from $P_{\gp}$ is sampled uniformly along straight lines between pairs of points sampled from $P_g$ and $P_r$, and $\lambda_{\gp}$ is a hyperparameter regulating the strength of this penalty. This objective is minimized in the optimization of the critic model. In our case, we use a generalized version of this objective:
\begin{equation}
    \label{eq:modified-wgangp}
    L = \bbE_{\tilde{\*x} \sim P_g}[D(\tilde{\*x})] - \bbE_{\*x \sim P_r}[D(\*x)] + \lambda_{\gp}\bbE_{\hat{\*x} \sim P_{\gp}} \left[ \left( \norm{\max \left(0, \nabla_{\hat{\*x}} D(\hat{\*x})\right)}_2 - c \right)^2 \right],
\end{equation}
where the desired Lipschitz constant $c$ can be adjusted, and only gradients larger than $c$ are penalized. For our experiments, we use $c=0.01$ and $\lambda_{\gp} = 10$, which produced the best results while hyperparameter tuning.

In a GAN-NCM $\hM = \langle \widehat{\*U}, \*V, \widehat{\cF}, P(\widehat{\*U}) \rangle$, the noise $P(\widehat{\*U})$ and functions $\cF$ collectively form a generative model over $\*V$. With the mutilation procedure, the same noise and functions form a generative model over $\*V_{\*x}$, where the functions for variables in $\*X$ are set to fixed values. Hence, we treat the parameters of $\widehat{\cF}$ as the parameters of the generator $G$ in the GAN objective. Another fully-connected multi-layer perceptron (MLP) is used as the discriminator/critic, $D$.

When fitting a specific data distribution $P(\*V_{\*x})$, the generator ($\widehat{\cF}_{\*x}$) attempts to maximize $\bbE_{\tilde{\*x} \sim P^{\hM}(\*V_{\*x})}[D(\tilde{\*x})]$, the expected score assigned by the critic for generated samples. The critic ($D$) attempts to minimize the objective in Eq.~\ref{eq:modified-wgangp}.

In cases where multiple datasets are fitted ($|\bbZ| > 1$), one could create a critic network for each dataset, but for the sake of computational efficiency, we use the same critic network $D$ for all datasets, with an additional input of an index clarifying which dataset is being evaluated. We did not find this to have any significant impact to performance.

In our experiments, each function $\hat{f}_{V_i} \in \widehat{\cF}$ is a fully connected multi-layer-perceptron (MLP) with 2 hidden layers of size 64, layer normalization \citep{ba2016layer} applied after each layer, and ReLU activations \citep{nair2010rectified} after each hidden layer. These choices were made as best practices for implementing WGAN-GP. More layers of larger size can be used for more complicated data. Since we work with binary data, a sigmoid activation function is placed at the end to constrain values between 0 and 1. Values are rounded at evaluation time. To simplify training, we increase the dimensionality of each $\widehat{U} \in \widehat{\*U}$ to 2. We recommend using larger dimensions for cases where bidirected cliques are very large, and a single $\widehat{U}$ variable is inputted to several functions.

For the critic $D$, we use an MLP with 2 hidden layers of size $64 \times |\*V|$, which scales depending on the number of variables in the dataset. We also apply layer normalization after each hidden layer and use ReLU activations. We do not use an activation after the last layer.

Weights of all networks are initialized using Glorot initialization \citep{glorot2010understanding}. The models are trained using the Adam optimizer \citep{KingmaB14}.

\subsubsection{Query Loss} \label{app:query-loss}

Since variables in our experiments are binary vectors, we compute $\bbD_Q$ as in Eq.~\ref{eq:ncm-div-loss} by using log loss. For example, if $Q = P(Y = 1)$, let $\hQ = \{\hat{y}_j\}_{j=1}^m$ be $m$ samples of $Y$ from the NCM $\hM$. Note that all outputs from $\hM$ are real numbers between 0 and 1 as a result of the sigmoid activation. If the goal was to maximize the query, we would compute
\begin{equation*}
    \bbD_Q(\hQ, Q) = \bbD_Q(\{\hat{y}_j\}_{j=1}^m, 1) = \sum_{j=1}^m -\log(\hat{y}_j)
\end{equation*}
which would approach 0 as $\hat{y}_j \rightarrow 1$ and approach $+\infty$ as $\hat{y}_j \rightarrow 0$. Penalizing this term would encourage the optimization process to generate samples of $\hat{y}_j$ closer to 1. In the case of minimization, we would replace $-\log(\hat{y}_j)$ with $-\log(1 - \hat{y}_j)$.

More generally, denote $\hat{y}_{X = x \mid Z = z}$ as a sample from $Y^{\widehat{M}}_{X = x}(\widehat{\*U})$ under the condition that $Z^{\widehat{M}}(\widehat{\*U}) = z$. In other words, $\hat{y}_{X = x \mid Z = z}$ is a sample from the NCM under training, $\widehat{M}$, under the intervention $X = x$ such that under the same sample of $\widehat{\*U}$, $Z$ evaluates to $z$ (obtained via rejection sampling). Then, for each Monte Carlo sample of the query, we add the corresponding loss term in Table \ref{tab:query-losses} depending on the query and whether we are minimizing or maximizing.

\begin{table}[h]
    \centering
    \begin{tabular}{l|l|l|l}
    \hline \hline
    Query & Calculation & Min Loss & Max Loss \\ \hline \hline
    ATE        & $P(Y_{X = 1} = 1)$          & $-\log(1 - \hat{y}_{X = 1})$        & $-\log(\hat{y}_{X = 1})$        \\ 
               & $-P(Y_{X = 0} = 1)$         & $-\log(\hat{y}_{X = 0})$            & $-\log(1 - \hat{y}_{X = 0})$      \\ \hline
    ETT        & $P(Y_{X = 1} = 1 \mid X = 1)$          & $-\log(1 - \hat{y}_{X = 1 \mid X = 1})$        & $-\log(\hat{y}_{X=1 \mid X=1})$        \\
               & $-P(Y_{X = 0} = 1 \mid X = 1)$ & $-\log(\hat{y}_{X = 0 \mid X = 1})$ & $-\log(1 - \hat{y}_{X=0 \mid X=1})$ \\ \hline
    NDE        & $P(Y_{X = 1, W_{X = 0}} = 1)$ & $-\log(1 - \hat{y}_{X=1, W_{X=0}})$        & $-\log(\hat{y}_{X=1, W_{X=0}})$        \\ 
               & $-P(Y_{X = 0} = 1)$ & $-\log(\hat{y}_{X=0})$ & $-\log(1 - \hat{y}_{X = 0})$ \\ \hline
    CtfDE        & $P(Y_{X = 1, W_{X = 0} = 1} \mid X = 1)$         & $-\log(1 - \hat{y}_{X = 1, W_{X = 0} \mid X = 1})$    & $-\log(\hat{y}_{X = 1, W_{X = 0} \mid X = 1})$        \\
               & $-P(Y_{X = 0} = 1 \mid X = 1)$ & $-\log(\hat{y}_{X = 0 \mid X = 1})$ & $-\log(1 - \hat{y}_{X = 0 \mid X = 1})$ \\ \hline
    \end{tabular}
    \caption{Log loss terms for $\bbD_Q$ in Eq.~\ref{eq:ncm-div-loss} for each of the queries in Fig.~\ref{fig:gan-id-expl-results}. Note that if $Y$ is a binary variable, then $\bbE[Y] = P(Y = 1)$.}
    \label{tab:query-losses}
\end{table}

\subsubsection{Identification Experiment Procedure}

We ran 20 trials for each setting, each with 4 reruns for the sake of hypothesis testing (see Sec.~\ref{app:hyp-test}). One trial of the experiment begins with by sampling $n=10^4$ data points from a randomly generated RCM (see Sec.~\ref{app:data-gen}). In the 1D case, we enforce positivity to accurately reflect our assumptions. In the 16D case, positivity is unreasonable for low sample sizes, so this is not enforced. This may be reflected in a lower accuracy in ID cases, since some identifiable cases are no longer identifiable when positivity does not hold.

For each rerun, one model is initialized and trained to fit the data while simultaneously maximizing the query, then a second model is initialized and trained to fit the data with simultaneously minimizing the query. Results are computed afterwards from the two models.

Each training epoch follows the procedure in Alg.~\ref{alg:ncm-learn-pv}. For each dataset in $\bbZ$, a batch is sampled from the real dataset alongside a batch from the corresponding distribution from the NCM, $\hM$. These batches are used to train the critic, $D$, as discussed in Sec.~\ref{app:architecture}. Following this, another batch is sampled from $\hM$ and evaluated by the critic. The resulting loss is summed across all datasets and added to the query loss explained in Sec.~\ref{app:query-loss}, multiplied by $\lambda$. This loss is used to train the generator, or in other words, the functions of $\widehat{\cF}$.

We use a learning rate of $\eta = 2 \cdot 10^{-5}$ and a batch size of 256 chosen from hyperparameter tuning results, but performance is relatively stable with other choices of learning rate and batch size when using the Adam optimizer. We run for 1000 epochs, but more epochs can be used for better results. We start $\lambda$ at $10^{-4}$ and exponentially decrease to $10^{-5}$ by the end of training. A wider range can be used for $\lambda$ given more epochs, but using values that are too extreme may disrupt the GAN training process. We found that with more datasets in $\bbZ$, the total sum of losses from $\bbD_P$ overpowers the query loss from $\bbD_Q$ compared to when $\bbZ$ has fewer datasets. We account for this by multiplying $\lambda$ by a factor of $|\bbZ|^2$. We use $m=10^4$ Monte Carlo samples for estimating Eq.~\ref{eq:ncm-ctf-eval} during training, but more can be used for a more accurate result at the expense of computational time.

\subsubsection{Estimation Experiment Procedure}
For estimation experiments, the procedure is largely the same as the identification experiments except only one model is trained, the query loss $\bbD_Q$ is not computed, and optimization is purely to fit the datasets in $\bbZ$. For each setting in Fig.~\ref{fig:gan-est-expl-results}, we ran 5 different settings of $n$, increasing exponentially from $10^3$ to $10^5$. We ran all settings for $d=1$ and $d=16$. We ran 10 trials for each setting, with means and confidence intervals computed across these 10 trials.

We use a learning rate of $\eta = 2 \cdot 10^{-5}$, a batch size of 1000, and run for up to 3000 epochs with early stopping. We use $m=10^6$ Monte Carlo samples to evaluate the query from $\hM$ via Eq.~\ref{eq:ncm-ctf-eval} after training.

\subsection{MLE-NCM Experiments}

\subsubsection{Explicit Counterfactual Modeling} \label{app:explicit-ncm}

The GAN-NCM is an \emph{implicit} version of the NCM, where it outputs direct samples of the data. This is in contrast to an \emph{explicit} version of the NCM such as the maximum likelihood estimation and which explicitly outputs the probabilities and likelihood values \citep{xia:etal21}. An explicit approach requires modifications to the NCM in contrast to what is stated in Def.~\ref{def:gncm}. Notably, the functions must be adapted to output probabilities instead of values from the variable domains, and sampling is not as straightforward (e.g., a trick through Gumbel random variables may be needed, and samples will not have gradients). This means that the method outlined in Alg.~\ref{alg:ncm-learn-pv} is incompatible, and accomplishing the same task using an explicit model will require a significantly more complicated procedure. First, recall the construction of an MLE-NCM (generalized to high dimensions):

$\widehat{M}(\bm{\theta}) := $
\begin{eqnarray}\label{eq:mle-ncm}
    \begin{cases}
        \mathbf V &:= \mathbf V, \hspace{+0.05in}
        \widehat{\*U} := \{ U_{\mathbf C}: \mathbf C \in \bbC(\mathcal G) \} \cup \{ G_{V_i} : V_i \in \mathbf V \}, \\
        \widehat{\mathcal F} &:= \left\{ f_{V_i} := \arg \max_{v_i \in \cD_{V_i}} g_{V_i = v_i} +
            \log \sigma(\phi_{V_i}(\pai{V_i}, \ui{V_i}^c; \theta_{V_i}))_{v_i} : V_i \in \mathbf{V}\right\}, \\
        P(\widehat{\mathbf U}) &:= \{ U_{\mathbf C} \sim \mathrm{Unif}(0, 1) : U_{\mathbf C} \in \mathbf U \} \; \cup \\
         & \hspace{0.17in}  \{ G_{V_i = v_i} \sim \mathrm{Gumbel}(0, 1) : V_i \in \mathbf V, v_i \in \cD_{V_i} \},
    \end{cases}
\end{eqnarray}

where $\sigma: \mathbb R^d \to (0, 1)^d$ is the softmax activation function; $\bbC(\mathcal G)$ is the set of maximal confounded cliques of $\mathcal G$; each $G_{V_i = v_i}$ is a standard Gumbel random variable \citep{gumbel1954statistical}; each $\phi_{V_i}(\cdot; \theta_{V_i})$ is a neural net parameterized by $\theta_{V_i} \in \bm{\theta}$; $\pai{V_i}$ are the values of the parents of $V_i$; and $\ui{V_i}^c$ are the values of $\Ui{V_i}^c := \{ U_{\mathbf C} : U_{\mathbf C} \in \mathbf U \text{ s.t. } V_i \in \mathbf C \}$.

Let $P(\*y_*) = P(\*y_{1[\*x_1]}, \*y_{2[\*x_2]}, \dots)$ be an unconditional counterfactual quantity. Then, one can compute
\begin{equation}
    \label{eq:mle-estimate-marg}
    P^{\hM}(\*y_*) = \sum_{\*v_j = (\*w_j, \*y_j) \mid \forall j, \forall \*w_j \in \cD_{\*V \setminus \*Y_j}}P^{\hM} \left(\bigwedge_{j} \*V_{\*x_j} = \*v_j \right).
\end{equation}
Denoting $P^{\hM}(\*v_*) := P^{\hM} \left(\bigwedge_{j} \*V_{\*x_j} = \*v_j \right)$ and $\*v_j = (v_{1j}, v_{2j}, \dots, v_{|\*V|j})$, this can be computed as
\begin{equation}
    P^{\hM}(\*v_*) = \bbE_{P(\*U^c)} \left[\prod_{v_{ij} \mid \forall j, \forall V_i \in \*V \setminus \*X_j} \sigma_{v_{ij}}\right] \approx \frac{1}{m} \sum_{k=1}^m \prod_{v_{ij} \mid \forall j, \forall V_i \in \*V \setminus \*X_j} \sigma_{v_{ij}},
\end{equation}
where $\sigma_{v_{ij}}$ denotes $\sigma(\phi_{V_{i}}(\pai{V_i, j}, \ui{V_i, k}^c; \theta_{V_i}))_{v_i}$ ($\pai{V_i, j}$ refers to the corresponding values of $\Pai{V_i}$ in $\*v_j$ from Eq.~\ref{eq:mle-estimate-marg}), the likelihood output of the model, $P^{\hM}(V_i = v_{ij} \mid \pai{V_i, j}, \ui{V_i, k}^c)$.
The right hand side of the equation is a Monte-Carlo approximation of the expectation over $P(\*U^c)$ using $k \in \{1, \dots, m\}$ samples of $\*U^c$, the set of uniform random variables of $\widehat{\*U}$.

A conditional counterfactual can then be computed as
\begin{equation}
    \label{eq:mle-ncm-estimate}
    P^{\widehat{M}}(\*y_* \mid \*x_*) = \frac{P^{\widehat{M}}(\*y_*, \*x_*)}{P^{\widehat{M}}(\*x_*)}.
\end{equation}

We work in the log space for numerical stability. Denote $\{\hP^{\cM^*}(\*V_{\*z_k})\}_{k=1}^{\ell}$ as the data sets of $\bbZ = \{P^{\cM^*}(\*V_{\*z_k})\}_{k=1}^{\ell}$, where $\*v_{\*z_k, i}$ refers to the $i$th data point of $\hP^{\cM^*}(\*V_{\*z_k})$ out of $n_k$ total data points. Denote $Q = P(\*y \mid \*x)$ as the query of interest. The MLE-NCM can be trained to fit the data sets as well as maximize the query via the loss
\begin{equation}
    \label{eq:mle-ncm-loss}
    L \left(\widehat{M}, \{\hP^{\cM^*}(\*V_{\*z_k})\}_{k=1}^{\ell} \right) := \sum_{k = 1}^{\ell} \frac{1}{n_k}\sum_{i=1}^{n} - \log P^{\widehat M}(\mathbf v_{\*z_k, i}) - \lambda \log P^{\widehat{M}}(\*y_* \mid \*x_*)
\end{equation}
where $\lambda$ is set initially to a high value and decreases during training. The left term is the negative log-likelihood of each data point in each data set, and the right term is the negative log-likelihood of the query. To minimize the query, we instead subtract $\lambda \log \left( 1 - P^{\widehat{M}}(\*y_* \mid \*x_*)\right)$ in the right term. The entire training procedure is shown in Alg.~\ref{alg:mle-ncm-train}.

\begin{figure}

\IncMargin{1em}
\begin{algorithm}[H]

    \DontPrintSemicolon
    \SetKw{notsymbol}{not}
    \SetKwData{ncmdata}{$\widehat{M}$}
    \SetKwData{paramdata}{$\bm{\theta}$}
    \SetKwData{pdata}{$\hat{p}$}
    \SetKwData{qdata}{$\hat{q}$}
    \SetKwData{lossdata}{$L$}
    \SetKwFunction{ncmfunc}{NCM}
    \SetKwFunction{estimate}{Estimate}
    \SetKwFunction{consistent}{Consistent}
    \SetKwFunction{sample}{Sample}
    \SetKwInOut{Input}{Input}
    \SetKwInOut{Output}{Output}
    
    \Input{Data $\{\hP^{\cM^*}(\*V_{\*z_k}) = \{\*v_{\*z_k, i}\}_{i=1}^{n_k}\}_{k=1}^{\ell}$, query $Q = P(\*y_* | \*x_*)$, causal diagram $\cG$, number of Monte Carlo samples $m$, regularization constant $\lambda$, learning rate $\eta$, training epochs $T$}
    \BlankLine
    $\hM \gets$ \ncmfunc{$\*V, \cG$} \tcp*{from Eq.~\ref{eq:mle-ncm}}
    Initialize parameters $\bm{\theta}_{\min}$ and $\bm{\theta}_{\max}$\;
    \For{$t \gets 1$ \KwTo $T$}{
        \For{$k \gets 1$ \KwTo $\ell$}{
            \For{$i \gets 1$ \KwTo $n$}{
                \tcp{\estimate from Eq.~\ref{eq:mle-ncm-estimate}}
                $\pdata_{\min} \gets$ \estimate{$P^{\widehat{M}(\bm{\theta}_{\min})}(\*v_{\*z_k, i}), m$}\;
                $\pdata_{\max} \gets$ \estimate{$P^{\widehat{M}(\bm{\theta}_{\max})}(\*v_{\*z_k, i}), m$}\;
                $\qdata_{\min} \gets$ \estimate{$P^{\widehat{M}(\bm{\theta}_{\min})}(\*y_* \mid \*x_*), m$}\;
                $\qdata_{\max} \gets$ \estimate{$P^{\widehat{M}(\bm{\theta}_{\max})}(\*y_* \mid \*x_*), m$}\;
                \tcp{\lossdata from Eq.~\ref{eq:mle-ncm-loss}}
                $\lossdata_{\min} \gets -\log \pdata_{\min} - \lambda \log(1 - \qdata_{\min})$\;
                $\lossdata_{\max} \gets -\log \pdata_{\max} - \lambda \log \qdata_{\max}$\;
                $\paramdata_{\min} \gets \paramdata_{\min} - \eta \nabla \lossdata_{\min}$\;
                $\paramdata_{\max} \gets \paramdata_{\max} - \eta \nabla \lossdata_{\max}$\;
            }
        }
    }
    \caption{Training MLE-NCM}
    \label{alg:mle-ncm-train}
\end{algorithm}
\DecMargin{1em}
\end{figure}


Nested counterfactuals can be computed as described in Appendix \ref{app:nested-ctf}, but in training, all computations of nested terms will not be able to have gradients due to the sampling requirement. 

\subsubsection{Model Architecture}
Following the discussion in Sec.~\ref{app:explicit-ncm}, we construct each $\hat{f}_{V_i} \in \widehat{\cF}$ using a masked autoencoder for density estimation (MADE) \citep{pmlr-v37-germain15}. We use Andrej Karpathy's implementation of MADE \citep{karpathy-made}, which is provided with the MIT License. Each function has 2 hidden layers of size 64, and outputs are passed through a log sigmoid activation to be converted into a probability. Weights are initialized using Glorot initialization. The models are trained using the Adam optimizer.

\subsubsection{Identification Experiment Procedure}

The identification procedure largely matches that of the GAN-NCM. We ran 20 trials for each setting, each with 4 reruns, and $n=10^4$ data points from the generating RCM. Positivity is enforced.

For each rerun, one model is initialized and trained to fit the data while simultaneously maximizing the query, then a second model is initialized and trained to fit the data with simultaneously minimizing the query. Results are computed afterwards from the two models.

Each training epoch follows the procedure in Alg.~\ref{alg:mle-ncm-train}. For each dataset in $\bbZ$, a likelihood value for each datapoint type is evaluated from the NCM, $\hM$. These likelihood values are maximized, weighted by the frequency of the datapoint. For the query loss, the likelihood values of the query distribution are either minimized or maximized. The full batch is used each epoch.

We use a learning rate of $\eta = 4 \cdot 10^{-3}$ and run for 1000 epochs. We start $\lambda$ at $1$ and exponentially decrease to $10^{-3}$ by the end of training. We use $m=10^4$ Monte Carlo samples for estimating Eq.~\ref{eq:mle-ncm-estimate} during training.

\subsubsection{Estimation Experiment Procedure}

As with the GAN-NCM, for estimation experiments, only one model is trained, the query loss $\bbD_Q$ is not computed, and optimization is purely to fit the datasets in $\bbZ$. For each setting in Fig.~\ref{fig:gan-est-expl-results}, we ran 5 different settings of $n$, increasing exponentially from $10^3$ to $10^5$, but only for $d=1$ due to tractability issues with higher dimensions. We ran 10 trials for each setting, with means and confidence intervals computed across these 10 trials.

We use the full batch per epoch. We use a learning rate of $\eta = 4 \cdot 10^{-3}$, and run for up to 3000 epochs with early stopping. We use $m=10^6$ Monte Carlo samples to evaluate the query from $\hM$ via the Gumbel-max trick.

\subsection{Runtime Experiment Procedure}
For the runtime experiments procedure, we conducted NCM-GAN and NCM-MLE on the same dataset generated in the first graph of Fig.~\ref{fig:gan-id-expl-results} as the dimensionality $d$ of $Z$ scales from 1 to 7. Similar to estimation experiments, only one model is trained, the query loss is not computed, and optimization is purely to fit the datasets. We record the first 100 epochs of training time for both methods. For each setting in Fig.~\ref{fig:runtime-results}, positivity is enforced when generating data and $n=10000$. We ran 4 trials for each setting, with means and confidence intervals computed across these 4 trials.

For GAN-NCM, we use a learning rate of $\eta = 2 \cdot 10^{-5}$ and a batch size of 1000. For MLE-NCM, we use a learning rate of $\eta = 4 \cdot 10^{-3}$ and the full batch per epoch.
\subsection{Hardware}

The MLE-NCM was trained on NVIDIA Tesla V100 GPUs provided by Amazon Web Services. All other experiments are run via jobs on a cluster, and each thread may be assigned either an NVIDIA K80 or NVIDIA P100 GPU. In total, approximately 10000 GPU hours were used for experiments.

%% file: section/C_related_works.tex
\section{Connection to Related Works} \label{app:related-works}

This paper evaluates causal queries directly through a parameterized model in the form of an SCM trained on available data from the true SCM $\cM^*$, which we call a proxy model-based approach. In this section, we discuss the differences between the NCM presented in this paper and the models used in other works that also implement a proxy approach.

First, we note that many of these works implicitly benefited from the constraints of enforcing a causal diagram into the model structure without proving any properties about these constraints. One contribution of our work is to provide some explicit results about these constraints via Thms.~\ref{thm:gl3-consistency} and \ref{thm:l3-g-expressiveness}. Second, and more importantly, most of these works operated under different assumptions and problem settings, to be discussed in individual sections. When applied to our problem setting, these methods may not be as general as the method discussed in this paper. Some of these losses of generality are highlighted in the following propositions and illustrated through the further examples provided in App. \ref{app:examples}.

\begin{proposition}
    \label{prop:markov-expressive}
    Any model in which the Markovianity assumption (also known as causal sufficiency) is enforced is not $\cL_3$-$\cG$ expressive (i.e. an equivalent statement of Thm.~\ref{thm:l3-g-expressiveness} cannot be proven about such a model class).
    \hfill $\blacksquare$
    \begin{proof}
        See counterexample provided in Example \ref{ex:markov-expressive} in Sec.~\ref{app:examples-properties}.
    \end{proof}
\end{proposition}

\begin{proposition}
    \label{prop:gid}
    There exist settings in which an $\cL_3$ query that is not identifiable from $\cL_1$ data ($P(\*V)$) alone may be identifiable given data sets from $\cL_2$.
    \hfill $\blacksquare$
    \begin{proof}
        Examples include ID settings 9, 10, 11, and 12 from Fig.~\ref{fig:gan-id-expl-results}, in which their counterparts without the $P(\*V_{x})$ datasets are not identifiable. Another example is shown in Example~\ref{ex:id} in Sec.~\ref{app:examples-properties}.
    \end{proof}
\end{proposition}

\begin{proposition}
    \label{prop:markov-nonid}
    There exist settings in which an $\cL_3$ query is not identifiable even in the Markovian case and given both $\cG$ and all datasets from $\cL_2$.
    \hfill $\blacksquare$
    \begin{proof}
        Example is shown in Example \ref{ex:non-id} from Sec.~\ref{app:examples-properties}. Further, this is an example of a case where the query cannot even be bounded.
    \end{proof}
\end{proposition}

We discuss more specifically the differences with each paper in the following subsections. Whenever notation differs between the two works, we use our notation to maintain consistency and make the comparison more transparent.

\subsection{\citet{kocaoglu2018causalgan}}

Similar to our work, \citet{kocaoglu2018causalgan} used an implicit approach to fit a proxy SCM on the data via generative adversarial networks. The method was specifically trained on observational ($\cL_1)$ data and used to evaluate interventional ($\cL_2$) queries in the Markovian setting. Like for a $\cG$-NCM, their approach used a proxy model $\hM = \langle \widehat{\*U}, \*V, \widehat{\cF}, P(\widehat{\*U}) \rangle$, where $\widehat{\*U}$ consists of independent fixed noise variables and each $\hat{f}_{V_i} \in \widehat{\cF}$ is defined such that $\hat{f}_{V_i}$ takes $\Pai{V_i}$ and $\widehat{\*U}_{V_i}$ as input. Here, $\Pai{V_i}$ is the set of parents of $V_i$ in the graph $\cG$, and $\widehat{\*U}_{V_i}$ is a set of noise variables specific to variable $V_i$.

This work introduced two novel GAN architectures for $\cL_2$ inference, CausalGAN and CausalBEGAN, focused on applications in image generation. For this challenging task, the CausalGAN architecture was designed in a unique fashion in which the model was separated into two parts: one for generating the labels, and another for generating the image conditioned on these labels. This was in contrast to having one generator learning the entire joint distribution of labels and images. The approach for learning the image generator is similar to a conditional GAN in that the objective of the generator includes an additional term to maximize the prediction accuracy of a labeling network. However, this often results in a problem they described as \emph{label-conditioned mode collapse}, where the generator may have been complacent in generating a small number of fixed samples that are easy to label. Hence, in addition to this labeler, CausalGAN also introduced an anti-labeler network. The anti-labeler also attempts to label the generated images, but the objective of CausalGAN simultaneously maximizes the accuracy of the labeler while minimizing the accuracy of the anti-labeler. This prevents label-conditioned mode collapse because images that are too easily labeled will fail to minimize the accuracy of the anti-labeler. More formally, for a binary label $l$ with probability $\rho = P(l = 1)$, the objective of the CausalGAN generator is:
\begin{align}
    & \min_G \bbE_{x \sim p_g(x)} \left[ \log \left( \frac{1-D(x)}{D(x)} \right) \right] \nonumber \\
    &- \rho \bbE_{x \sim p_g^1(x)}[\log(D_{LR}(x)] - (1 - \rho)\bbE_{x \sim p_g^0 (x)}[log(1 - D_{LR}(x))] \nonumber \\
    &+ \rho \bbE_{x \sim p_g^1(x)}[\log(D_{LG}(x)] + (1 - \rho)\bbE_{x \sim p_g^0 (x)}[log(1 - D_{LG}(x))],
\end{align}
where $G$ is the generator, $D$ is the discriminator, $p_g(x)$ is the distribution induced by the generator, $p_g^i(x)$ is the distribution conditioned on $l=i$, $D_{LR}$ is the labeler, and $D_{LG}$ is the anti-labeler. The loss is composed of the loss from the discriminator subtracted with the loss from the labeler and adding the loss from the anti-labeler. The paper proved that the generator successfully matched the true conditional image distribution when $G$ was the optimum of this objective \citep[Thm.~2]{kocaoglu2018causalgan}. CausalBEGAN was another architecture introduced where optimization considered image quality margin and label margin. In both cases, interventions were computed using the mutilation procedure on the generating model $\hM$.

CausalGAN and CausalBEGAN were empirically shown to successfully learn the data distribution in several image generation tasks. Experiments demonstrated several settings in which intervening on a label resulted in a different image compared to conditioning on the label.

The reference produced impressive results in generating images under causal interventions using novel implicit GAN architectures. While the approach is similar to our work in that a GAN proxy model is used, there are a number of differences with our work.
\begin{enumerate}
    \item The paper worked in the space of inferring $\cL_2$ quantities given observational data from $\cL_1$. Our work generalizes this setting to inferring any counterfactual query from $\cL_3$ given any combination of datasets from $\cL_1$ and $\cL_2$.
    
    \item Unlike this work, we do not make the Markovianity assumption, meaning we allow unobserved confounding between variables. Many real world settings often have unobserved confounding, so the Markovianity assumption can be limiting. In a situation where the true SCM has unobserved confounding, a model class which enforces Markovianity will not be expressive enough to model the true SCM (Prop.~\ref{prop:markov-expressive}).
    
    \item By \citep[Def.~1]{kocaoglu2018causalgan}, the architecture of CausalGAN took the form of a causal diagram $\cG$, similar to how the $\cG$-constrained NCM is constructed in Def.~\ref{def:gncm}. In our work, we discuss the theoretical properties of the implied constraints via Thms.~\ref{thm:gl3-consistency} and \ref{thm:l3-g-expressiveness}. While $\cG^{(\cL_2)}$-consistency likely held for CausalGAN due to the nature of its construction, this property was not explicitly proven or discussed.
    
    \item The fact that all $\cL_2$ quantities are identifiable from $\cL_1$ data
     was acknowledged \citep[Prop.~1]{kocaoglu2018causalgan}. The identification question constitutes an additional challenge in our paper since we consider layer 3 queries and allow for unobserved confounding. It's understood that without Markovianity, even inferences about $\cL_2$ queries are not always identifiable from $\cL_1$ data. 
\end{enumerate}

\subsection{\citet{pawlowski2020deep}}

\citet{pawlowski2020deep} implemented a deep proxy SCM approach in which each mechanism $\hat{f}_{V_i}$ models a conditional distribution $P(V_i \mid \Pai{V_i})$ under the Markovianity assumption (lack of unobserved confounders). Under this assumption, the joint observational distribution $P(\*V)$ can be factorized as
\begin{equation}
    \label{eq:markov-factorization}
    P(\*V) = \prod_{V_i \in \*V} P(V_i \mid \Pai{V_i}),
\end{equation}
so each mechanism represents one conditional distribution in this factorization. They proposed three novel approaches to modeling this conditional distribution.
\begin{enumerate}
    \item The first approach involved modeling $P(X \mid \Pai{X})$ using normalizing flows. $\widehat{\*U}_{X}$ consists of a noise variable $\epsilon$, and $\hat{f}_{X}$ is an invertible function such that $x = \hat{f}_{X}(\epsilon; \pai{X})$. Then the explicit likelihood could be computed as $p(x \mid \pai{x}) = p(\epsilon) \cdot |\det \nabla_{\epsilon} \hat{f}_X (\epsilon; \pai{x})|^{-1}$, where $\epsilon$ could be computed as $\hat{f}_X^{-1}(x; \pai{x})$.
    
    \item The second approach extended the first approach to high dimensional settings by splitting $\epsilon$ into two independent noise variables, $z$ and $u$, and splitting $\hat{f}_X$ into two components: a non-invertible function $g$ and an invertible function $h$. Then $\hat{f}_X(\epsilon, \pai{x}) = h(u; g(z; \pai{x}), \pai{x})$, where $g$ is expected to capture the high-level aspects of the data. This is helpful in high dimensional settings as the size of the invertible component $h$ can be reduced because working with the outputs of $g$ is easier than directly working with the data. Since $p(x \mid \pai{x})$ is no longer tractable, the model is instead trained using variational inference by maximizing the evidence lower bound (ELBO) with variational distribution $Q(z \mid x, \pai{x})$,
    \begin{equation}
        \log p(x \mid \pai{x}) \geq \bbE_Q(z \mid x, \pai{x}) [\log p(x \mid z, \pai{x})] - D_{\kl}[Q(z \mid x, \pai{x}) || P(z)].
    \end{equation}
    
    \item Although not demonstrated in the paper, a third approach was presented suggesting the use of a GAN-based approach to implicitly model the likelihood.
\end{enumerate}

The method primarily used in the paper was the 2nd approach. They evaluated counterfactual quantities using the 3-step procedure from \citet[Thm.~7.1.7]{pearl:2k}:
\begin{enumerate}
    \item \textbf{Abduction} -- The observations $\*v$ of all variables $\*V$ are used to predict $P(\*U \mid \*v)$, where $P(\*U \mid \*v)$ decomposes as $\prod_{V_i \in V} P(\epsilon_{V_i} \mid v_i, \pai{V_i})$ due to the Markovianity assumption. This is then computed as
    \begin{align}
    P(\epsilon_{V_i} \mid v_i, \pai{V_i}) &= P(z_{V_i} \mid v_i, \pai{V_i})P(u_{V_i} \mid z_{V_i}, v_i, \pai{V_i}) \nonumber \\
    &\approx Q(z_{V_i} \mid e_{V_i}(v_i; \pai{V_i})) \delta(u_{V_i})
    \end{align}
    where $e(v_i; \pai{V_i})$ is a deterministic encoder mapping $v_i$ and $\pai{V_i}$ to $z_{V_i}$, and $\delta$ is a Dirac delta distribution centered at $h_{V_i}^{-1}(v_i; g_{V_i}(z_{V_i}; \pai{V_i}), \pai{V_i})$.
    
    \item \textbf{Action} -- For an intervention over sets of variables $\*X$, each mechanism $\hat{f}_X$ for $X \in \*X$ is replaced either with a fixed value $x$ (atomic) or a surrogate mechanism $\tilde{f}_{X}(\epsilon_X \mid \widetilde{\mathbf{pa}}_X)$ (stochastic).
    
    \item \textbf{Prediction} -- Under this distribution of $P(\*U \mid \*v)$, and with the modified mechanisms, the deep SCM is evaluated to produce a result for the counterfactual quantity $P(\*V_{\*x} \mid \*v)$.
\end{enumerate}

The approach was demonstrated empirically on a synthetic MNIST dataset and on a real world brain imaging dataset.

The paper presented a powerful approach to modeling a deep SCM using normalizing flows and variational inference, allowing for efficient and precise training even in high-dimensional cases. Still, there are a number of differences between this paper and our work.
\begin{enumerate}
    \item Unlike this work, we do not make the Markovianity assumption, meaning we allow unobserved confounding between variables. Many real world settings often have unobserved confounding, so the Markovianity assumption can be limiting. In a situation where the true SCM has unobserved confounding, a model class which enforces Markovianity will not be expressive enough to model the true SCM (Prop.~\ref{prop:markov-expressive}). Moreover, it is not a trivial task to extend the normalizing flow approach in this paper to a case with unobserved confounders. In such a case, Eq.~\ref{eq:markov-factorization} would no longer be valid, and it is not immediately clear how to train invertible mechanisms that can share noise.
    
    \item Like our work, this paper assumed that the causal graph $\cG$ is provided. However, in our work, we discuss the theoretical properties of the implied constraints via Thms.~\ref{thm:gl3-consistency} and \ref{thm:l3-g-expressiveness}, while this paper implicitly benefited from these constraints without discussing these properties. 
    
    \item The identification problem is not considered in this work while it is a major focus in our work. As shown in Prop.~\ref{prop:markov-nonid}, there exist counterfactual quantities that cannot be identified given the graph $\cG$, and all available data from $\cL_1$ and $\cL_2$, even under the Markovianity assumption. If these quantities were to be estimated from a model trained on such data, the result will likely be incorrect and/or misleading. For that reason, we make sure to identify all queries before attempting to estimate them.
    
    \item While this paper only considered input in the form of observational data from $\cL_1$, our work considers cases where additional experimental datasets from $\cL_2$ are provided (even cases where perhaps the observational distribution $P(\*V)$ is not provided). As shown in Prop.~\ref{prop:gid}, there do exist counterfactual queries that are not identifiable given observational data but are identifiable given datasets from $\cL_2$.
    
    \item This paper evaluated counterfactual quantities using the 3-step procedure of abduction, action, and prediction. While this is a sound algorithm for evaluating applicable quantities, it does not apply generally to all counterfactual quantities. Specifically, it only applies in quantities of the form $P(\*Y_{\*x} \mid \*z)$, in which only one intervention is applied before the conditioning bar, and observations are present after the conditioning bar. This excludes quantities such as $P(y_x, y'_{x'})$, which is a joint probability of two counterfactual terms in different worlds (one in $X = x$ and another in $X = x'$), as well as quantities such as $P(y_x \mid y'_{x'})$, in which a term in an intervened world $(X = x')$ is after the conditioning bar.
    
    \item Finally, although an implicit approach via GANs was discussed in  \citep{pawlowski2020deep}, it was not explored. Our paper realizes an implementation of GAN approach in the non-Markovian setting.
\end{enumerate}

\subsection{\citet{witty2021}}

\citet{witty2021} introduced simulation-based identifiability (SBI), which uses a Bayesian approach to solve the identification problem for $\cL_2$ and some $\cL_3$ queries given $\cL_1$ data. In their approach, the goal was to identify a particular query $Q$ given a fixed observational distribution $P(\*V)$ and a causal diagram $\cG$. An SCM $\cM^* = \langle \*U^*, \*V, \cF^*, P(\*U^*) \rangle$ is sampled from a prior over SCMs (specifically their functions and exogenous variables), $P(\cF, \*U)$. In each iteration, $n$ samples are drawn from $P^{\cM^*}(\*V)$, the observational dataset of $\cM^*$. Similar to our work, two parameterized SCMs, $\hM_1$ and $\hM_2$ are optimized to simultaneously fit $P^{\cM^*}(\*V)$ while maximizing the difference in the induced queries of the two SCMs, $\hQ_1$ and $\hQ_2$. This is done by maximizing the objective
\begin{equation}
    L(\hM_1, \hM_2, \{\*v_i\}_{i = 1}^n) = \sum_{i=1}^n \left(\log P^{\hM_1}(\*v_i) + \log P^{\hM_2}(\*v_i)\right) + \lambda |\hQ_1 - \hQ_2|,
\end{equation}
where $\{\*v_i\}_{i=1}^n$ are a collection of $n$ samples from $P^{\cM^*}(\*V)$. Maximizing this objective is equivalent to jointly maximizing the likelihood of the data for both $\hM_1$ and $\hM_2$ as well as maximizing the gap between $\hQ_1$ and $\hQ_2$. Identifiability is concluded if, following a hypothesis testing procedure, $|\hQ_1 - \hQ_2|$ is under a certain threshold. Non-identifiability is concluded otherwise. They proved that $Q$ is identifiable w.r.t. $\cM^*$ if and only if there does not exist any parameterized SCM $\hM$ fitted to the data from $P^{\cM^*}(\*V)$ such that $P^{\hM}(\*V)$ converges to $P^{\cM^*}(\*V)$ as $n \rightarrow \infty$, $Q^{\hM} \neq Q^{\cM^*}$, and $\hM$ is in the support of the prior \citep[Thm.~3.2]{witty2021}.

In non-parametric settings, there exist sound and complete symbolic methods such as do-calculus that could efficiently and correctly identify any $\cL_2$ query \cite{pearl:2k}. Hence, as we understand, the paper primarily focused its attention on parametric settings, where the model class of SCMs was further restricted to specific function classes (e.g. linear), and symbolic approaches to identification under this constrained class were not as well-studied.

For the identification task, SBI shares many similarities with our approach in Alg.~\ref{alg:ncm-solve-ctfid}, notably that they also trained two proxy SCMs to minimize/maximize the query while simultaneously fitting the data. Still, there are many differences between the two works.
\begin{enumerate}
    \item The paper considered a general form of the query $Q(\*V_{\*X = \*x}, \*V_{\*X = \*x'})$, which was a function of the counterfactual terms of $\*V$ in the two worlds of $\*X = \*x$ and $\*X = \*x'$. For example, average treatment effect (ATE) given a binary treatment $X$ could be computed as $Q(\*V_{\*X = \*x}, \*V_{\*X = \*x'}) = \bbE[Y_{X = 1}] - \bbE[Y_{X = 0}]$, which was what was used in their experiments. While in theory, this may have included certain counterfactual queries, these queries were not explored, so the work largely focused on the space of $\cL_2$ queries. Moreover, they only considered data from $\cL_1$. Our work generalizes this setting to inferring any counterfactual query from $\cL_3$ given any combination of datasets from $\cL_1$ and $\cL_2$.
    
    \item Like our work, this paper assumed that the causal graph $\cG$ is provided. However, in our work, we discuss the theoretical properties of the implied constraints via Thms.~\ref{thm:gl3-consistency} and \ref{thm:l3-g-expressiveness}, while this paper implicitly benefited from these constraints without discussing these properties.
    
    \item The paper was largely concerned with identifying quantities in parametric cases (and functions of SCMs were fixed to certain parametric families), while our work is largely focused on the non-parametric case. We make no claims about expressivity or $\cG$-consistency in parametric settings.
    
    \item Our work assumes that there exists a true SCM $\cM^*$, which generates a specific dataset from which we aim to identify and estimate our query. In contrast, \citep{witty2021} specifically aimed to identify a query without a true SCM in mind and instead matched the distribution with an SCM sampled from a prior distribution over the space of SCMs. While estimation could theoretically be performed after concluding identifiability with SBI, this was not a focus of \citep{witty2021}, and the estimation would not be meaningful unless $\cM^*$ corresponded to some true SCM in a real application.
    
    \item In terms of the approach, \citep{witty2021} differs from our work in that the two parametrized models, $\hM_1$ and $\hM_2$ are optimized jointly, while in our case, the two models are optimized separately. The goal in SBI is to maximize the gap between the two models in the query, while Alg.~\ref{alg:ncm-solve-ctfid} does this implicitly when maximizing or minimizing the query in the two models separately. Further, our implementation in practice utilizes GANs, which is not explored in this paper.
\end{enumerate}


\subsection{\citet{xia:etal21}}
\citet{xia:etal21} introduced the NCM, which is the model class considered in our paper. They proved that NCMs are expressive on all three layers \citep[Thm.~1]{xia:etal21} but then subsequently proved the Neural Causal Hierarchy Theorem \citep[Corol.~1]{xia:etal21} which states that despite this expressivity, an NCM trained only on lower layer data would almost never match the true SCM on higher layers. To proceed, they introduced a graphical inductive bias in the form of a causal diagram $\cG$ derived from the true SCM. They defined a $\cG$-constrained version of the NCM (Def.~\ref{def:gncm} in our paper).

They proved that the class of $\cG$-constrained NCMs is both $\cG$-consistent \citep[Thm.~2]{xia:etal21} and $\cL_2$-$\cG$ expressive \citep[Thm.~3]{xia:etal21}. $\cG$-consistency ($\cG^{(\cL_2)}$-consistency in the context of our work) constrained the class to a smaller set of models which may have potentially agreed on the causal query, while $\cL_2$-$\cG$ expressiveness guaranteed that $\cG$-NCMs could still model any SCM within this constrained space. Still, these results only considered up to layer 2 of the causal hierarchy, and made no claims about layer 3. In contrast, Thms.~\ref{thm:gl3-consistency} and \ref{thm:l3-g-expressiveness} in our work are much stronger claims, proving these results on layer 3. That is, $\cG$-NCMs have all of the constraints implied on the third layer while still being able to express any true SCM on all three layers within this constrained setting.

The paper then defined the neural identification problem specifically for $\cL_2$ queries. They proved a duality stating that neural identification for $\cL_2$ quantities is equivalent to graphical identification \citep[Thm.~4]{xia:etal21} and showed that the mutilation procedure applied to a proxy $\cG$-NCM can be used to estimate identifiable queries \citep[Corol.~2]{xia:etal21}. This definition of neural identification did not consider $\cL_3$ quantities or $\cL_2$ input data. In our case, we define neural counterfactual identification (Def.~\ref{def:ncm-l3-id}) generalizing the neural identification problem to handle $\cL_3$ queries given arbitrary datasets from $\cL_1$ and $\cL_2$. We prove a stronger duality in this setting (Thm.~\ref{thm:ncm-ctfid-equivalence}) and that counterfactual quantities can be evaluated directly from the proxy $\cG$-NCM, akin to the mutilation procedure for $\cL_2$ quantities (Corol.~\ref{cor:op-id}).

Following these identification results, \citep{xia:etal21} introduced an algorithm \citep[Alg.~1]{xia:etal21} which is both sound and complete for identifying and estimating $\cL_2$ queries \citep[Corol.~4]{xia:etal21}. The algorithm was implemented in practice through a maximum likelihood approach (MLE-NCM) through \citep[Alg.~2]{xia:etal21}. This involved training two $\cG$-NCM parameterizations, where one was trained to maximize the query, while the other was trained to minimize it, and both models were trained to fit the observational data $P(\*V)$. The algorithm in our counterfactual setting (Alg.~\ref{alg:ncm-solve-ctfid}) is similar on the surface, but there are many additional challenges in solving this more general problem. First, while an interventional query might be easy to compute and optimize, it is not as clear how to compute a counterfactual quantity with gradients such that it can be maximized or minimized. After all, counterfactual quantities often have complicated forms such as nesting (e.g. $P(Y_{Z_X})$) or having the same variable in multiple worlds (e.g. $P(Y_{X=0}, Y_{X=1})$). Second, fitting a single dataset is standard practice, which is what needs to be done when only observational data is available. It is not as clear how to simultaneously fit several datasets. Our work solves both of these problems in practice using the GAN-NCM approach in Sec.~\ref{sec:ncm-estimation}, but also explains how one might solve this problem using the MLE-NCM in Sec.~\ref{app:explicit-ncm}.

To demonstrate the ability of the MLE-NCM to identify and estimate $\cL_2$ queries in practice, the paper showed results from identifying average treatment effect (ATE) from 8 typical graphs encountered in the literature. They also provided estimation results for the identifiable cases. While our experiments mainly focused on identifying $\cL_3$ queries, we reproduced the experiments of \citep{xia:etal21} in Sec.~\ref{app:more-experiments} using the GAN-NCM to demonstrate that our approach is general and can be applied to the $\cL_2$ case as well. Further, although the reason why one might prefer to use an NCM for causal inference is due to the ability of neural models to scale inferences to high dimensions, this ability was not emphasized in the experimental results of this paper. We explicitly showed how the GAN-NCM can perform nearly as well even in 16-dimensional cases while showing that the MLE-NCM scales poorly due to tractability issues (see Fig.~\ref{fig:runtime-results} in Sec.~\ref{sec:experiments}). One additional issue of the MLE-NCM is its bias towards returning ID results even in non-ID cases. This is evident in Fig.~\ref{fig:gan-id-expl-results} and can be potentially problematic if used in practice since attempting to estimate a non-ID query will likely produce an incorrect and misleading result. For more details, see the discussion surrounding Fig.~\ref{fig:bdm-id-results} in Sec.~\ref{app:more-experiments}.

To summarize, this paper introduced the NCM and proved several properties about it, including results on expressiveness, constraints, and identification. 
Our work also uses the NCM and proves many analogous results, but there are still several differences:
\begin{enumerate}
    \item The paper primarily worked in the space of inferring $\cL_2$ quantities given observational data from $\cL_1$. Our work generalizes this setting to inferring any counterfactual query from $\cL_3$ given any combination of datasets from $\cL_1$ and $\cL_2$. This is reflected in the theoretical results, the algorithms, and the practical implementation. We show that the $\cG$-NCM has all $\cL_3$ constraints (Thm.~\ref{thm:gl3-consistency}) but can still express any model on all three layers in this constrained setting (Thm.~\ref{thm:l3-g-expressiveness}). We show that the $\cG$-NCM can be used to solve the identification and estimation problem in the more general counterfactual setting (Thm.~\ref{thm:ncm-ctfid-equivalence}, Corol.~\ref{cor:op-id}, Alg.~\ref{alg:ncm-solve-ctfid}, Corol.~\ref{thm:ncm-ctfid-correctness}). These results drastically generalize the capabilities of the NCM to $\cL_3$.
    
    \item While the MLE-NCM implemented in this paper is an option for counterfactual identification and estimation (as shown in  Sec.~\ref{app:explicit-ncm}), we develop a new type of architecture for the NCM using a generative adversarial approach. Our approach, detailed in Sec.~\ref{sec:ncm-estimation}, has many advantages over the MLE-NCM, notably its design simplicity, its ability to scale to high dimensions, and its robustness in solving the ID problem. We even show that the GAN-NCM can solve all of the experimental settings originally presented in \citep{xia:etal21} in Sec.~\ref{app:more-experiments}.
\end{enumerate}

From all of these results, both theoretical and empirical, it is evident that our approach, in fact,  subsumes the work in \citep{xia:etal21}.

\subsection{\citet{zecevic2021relating}}

\citet{zecevic2021relating} related the class of graph neural networks (GNNs) to the class of Markovian SCMs. Given a DAG $\cG$ with variables $\*V$, they defined GNN layers using the functions
\begin{equation}
    \label{eq:gnn-ncm}
    h_{V_i} = \phi \left( d_{V_i}, \bigoplus_{V_j \in \cN(V_i)} \psi (d_{V_i}, d_{V_j}) \right),
\end{equation}
where $d_{V_i}$ is either data associated with $V_i$ or a feature computed from a previous layer, $\cN(V_i)$ is the set of neighbors of $V_i$ in $\cG$, $\bigoplus$ refers to an aggregating function such as sum, and $h_{V_i}$ is a feature of $V_i$ for the current layer. They proved \citep[Thm.~1]{zecevic2021relating}, which states that for any Markovian SCM $\cM = \langle \*U, \*V, \cF, P(\*U) \rangle$, there exists a GNN construction such that $h_{V_i} = f_{V_i}$ for all $f_{V_i} \in \cF$, and such a construction (but without shared $\phi$ and $\psi$) is therefore a special case of the NCM introduced in \citep{xia:etal21} (\citep[Corol.~1]{zecevic2021relating}). Using a GNN interpretation of an SCM, one could compute quantities from an intervention $\*X = \*x$ through the mutilation approach by evaluating Eq.~\ref{eq:gnn-ncm} but with edges into variables of $\*X$ removed from $\cN$.

Leveraging results from variational graph autoencoders (VGAEs), \citep{zecevic2021relating} demonstrated an approach to learning the observational distribution with a GNN architecture. They introduced interventional VGAEs (iVGAEs) which defined inferences from VGAEs on layer 2 quantities via the mutilation procedure on the GNN layers. They proved a wealth of results related to this class of models following similar results proven in \citep{xia:etal21}: $\cL_2$-expressivity \citep[Thm.~3]{zecevic2021relating}, a causal hierarchy theorem result between layers 1 and 2 \citep[Corol.~2]{zecevic2021relating}, and equivalence between neural-identification from iVGAEs and classical identification of $\cL_2$ quantities \citep[Thm.~4]{zecevic2021relating}.

\citep{zecevic2021relating} introduced an interesting discussion on the relationship between GNNs, SCMs, and NCMs. The work differs from our work in a variety of ways.
\begin{enumerate}
    \item Like \citep{xia:etal21}, \citep{zecevic2021relating} primarily worked in the space of inferring $\cL_2$ quantities given observational data from $\cL_1$. Our work generalizes this setting to inferring any counterfactual query from $\cL_3$ given any combination of datasets from $\cL_1$ and $\cL_2$.
    
    \item Unlike this work, we do not make the Markovianity assumption, meaning we allow unobserved confounding between variables. Many real world settings often have unobserved confounding, so the Markovianity assumption can be limiting. In a situation where the true SCM has unobserved confounding, a model class which enforces Markovianity will not be expressive enough to model the true SCM (Prop.~\ref{prop:markov-expressive}). The results proven in this paper do not extend trivially to cases with unobserved confounders. Traditional GNNs are typically defined on adjacency matrices, which cannot describe bidirected edges in causal diagrams.
    
    \item Like our work, this paper assumed that the causal graph $\cG$ was provided. However, in our work, we discuss the theoretical properties of the implied constraints via Thms.~\ref{thm:gl3-consistency} and \ref{thm:l3-g-expressiveness}. There does not appear to be a result in the paper  demonstrating that $\cG$-GNNs or iVGAEs are $\cG^{(L_2)}$-consistent. It is not immediately obvious that $\cG$-constrained GNNs have all of the $\cL_2$ constraints of $\cG$ given the unique architecture of GNNs.
    
    \item The identification problem is one of the main focuses of our work. While the paper discussed the identification problem of $\cL_2$ queries given $\cL_1$ data, they did not provide a method of solving the problem using iVGAEs. Rather, they focused on the estimation problem under the assumption of identifiability. In fact, all $\cL_2$ quantities are identifiable from $\cL_1$ data under the Markovianity assumption \citep[Corol.~2]{bareinboim:etal20}, so an approach for solving the identification problem in this setting may not be meaningful.
    
    \item Finally, we note that while this paper explored some novel aspects of using GNN and VGAE architectures for causal inference, our paper focuses on GANs, an entirely different class of models.
\end{enumerate}

\subsection{\citet{sanchezmartin2022}}

\citet{sanchezmartin2022} introduced Variational Causal Autoencoders (VACA), a model for estimating interventional and counterfactual quantities under the Markovianity assumption using a variational graph autoencoder architecture (VGAE) and graph neural network (GNN) architecture like \citep{zecevic2021relating}. VACA is defined as a tuple $\langle A, p(\*Z), p_{\theta}, q_{\phi}\rangle$, where
\begin{itemize}
    \item $A$ is an adjacency matrix representing the causal DAG $\cG$;
    \item $p(\*Z) = \prod_{V_i \in \*V} p(Z_{V_i})$ is a prior distribution over independent latent variables $\*Z = \{Z_{V_i} : V_i \in \*V\}$;
    \item The decoder $p_{\theta}(\*V \mid \*Z, A)$ is a GNN parameterized by $\theta$ mapping $\*Z$ and $A$ to the likelihood of $\*V$ given these values;
    \item The encoder $q_{\phi}(\*Z \mid \*V, A)$ is a GNN parameterized by $\phi$ mapping $\*V$ and $A$ to the posterior approximation of $\*Z$ given these values.
\end{itemize}

The model was trained to maximize the likelihood of observational data $P(\*V)$, which can be expressed as
\begin{equation}
    P(\*V) \approx \int p_{\theta}(\*V \mid \*Z, A)p(\*Z)d\*Z.
\end{equation}

Here, $\*Z$ is similar to the role of $\*U$ in an SCM. Although $\*Z$ does not have to be equal to $\*U$ from the true SCM, each $Z_{V_i}$ provides a source of noise for variable $V_i$ that captures the same information that $\*U_{V_i}$ contributes to $V_i$ in the true SCM. The difference between $\*Z$ and $\*U$ does not matter as long as the VACA model can match the true observational distribution.

During evaluation, interventional ($\cL_2$) quantities can be estimated via
\begin{equation}
    \label{eq:vaca-l2}
    p(\*V_{\*X = \*x}) \approx \int \int p_{\theta}(\*V \mid \tilde{\*Z}, \tilde{\*Z}^{\cI}, A^{\cI})p(\tilde{\*Z})q_{\phi}(\tilde{\*Z}^{\cI} \mid A^{\cI}, \*X) d\tilde{\*Z} d\tilde{\*Z}^{\cI},
\end{equation}
where $A^{\cI}$ is the adjacency matrix with edges into variables of $\*X$ removed, $\tilde{\*Z}^{\cI}$ is the set of representations of $\*X$, and $\tilde{\*Z}$ is the subset of $\*Z$ corresponding to $\*V \setminus \*X$. In words, intervened variables are assigned a representation from the encoder under a graph with the edges into $\*X$ mutilated, while the remaining variables are assigned the same representation.

Further, the following counterfactual ($\cL_3$) quantity can also be estimated:
\begin{align}
    & P(\*V_{\*X = \*x} = \*v^{CF} \mid \*v^{F}) \nonumber \\
    &\approx \int \int p_{\theta}(\*V = \*v^{CF} \mid \tilde{\*Z}^{F}, \tilde{\*Z}^{\cI}, A^{\cI})q_{\phi}(\tilde{\*Z}^{\cI} \mid \*x, A^{\cI})q_{\phi}(\tilde{\*Z}^{F} \mid \*v^{F}, A) d\tilde{\*Z}^{\cI} d\tilde{\*Z}^{F},
\end{align}
where $\*v^{F}$ is the observed factual value of $\*V$, $\*v^{CF}$ is the counterfactual value of $\*V$ to be queried, and $\tilde{\*Z}^F$ is the set of representations associated with $\*v^F$ except for the variables in $\*X$. In words, the 3-step procedure for evaluating counterfactuals \citep[Thm.~7.1.7]{pearl:2k} is applied here. The abduction step is performed by computing the factual representations $\tilde{\*Z}^F$ through the encoder $q_{\phi}(\tilde{\*Z}^F \mid \*v^F, A)$ applied to the factual observations $\*v^F$. The action step is performed by assigning the intervened representations $\tilde{\*Z}^{\cI}$ as discussed in Eq.~\ref{eq:vaca-l2}. Finally, the prediction step is performed by using the interventional representations $\tilde{\*Z}^{\cI}$ from $\*X$ and the factual representations $\tilde{\*Z}^{F}$ from $\*V \setminus \*X$ in the encoder $p_{\theta}(\*V \mid \tilde{\*Z}^F, \tilde{\*Z}, A^{\cI})$ with the mutilated graph $A^{\cI}$ to produce a counterfactual outcome $\*v^{CF}$ for $\*V$ in this setting.

VACA presented interesting ideas for evaluating interventional and counterfactual quantities in the Markovian case using GNNs and VGAEs. Still, they differ from our work in several ways.
\begin{enumerate}
    \item Unlike this work, we do not make the Markovianity assumption, meaning we allow unobserved confounding between variables. Many real world settings often have unobserved confounding, so the Markovianity assumption can be limiting. In a situation where the true SCM has unobserved confounding, a model class which enforces Markovianity will not be expressive enough to model the true SCM (Prop.~\ref{prop:markov-expressive}). VACA does not extend trivially to cases with unobserved confounders. Traditional GNNs are typically defined on adjacency matrices, which cannot describe bidirected edges in causal diagrams.
    
    \item Like our work, this paper assumed that the causal graph $\cG$ was provided. However, in our work, we discuss the theoretical properties of the implied constraints via Thms.~\ref{thm:gl3-consistency} and \ref{thm:l3-g-expressiveness}. This paper possibly benefited from these constraints implicitly without discussing these properties, but it is not actually clear if VACA is a $\cG^{(\cL_3)}$-consistent model.
    
    \item The identification problem was not considered in this work while it is a major focus in our work. As shown in Prop.~\ref{prop:markov-nonid} above, there exist counterfactual quantities that cannot be identified given the graph $\cG$, and all available data from $\cL_1$ and $\cL_2$, even under the Markovianity assumption. If these quantities were to be estimated from a model trained on such data, the result will likely be incorrect and misleading. For that reason, we make sure to identify all queries before attempting to estimate them.
    
    \item While this paper only considered input in the form of observational data from $\cL_1$, our work considers cases where additional experimental datasets from $\cL_2$ are provided (even cases where perhaps the observational distribution $P(\*V)$ is not provided). As shown in Prop.~\ref{prop:gid}, there do exist counterfactual queries that are not identifiable given observational data but are identifiable given datasets from $\cL_2$.
    
    \item This paper evaluated counterfactual quantities using the 3-step procedure of abduction, action, and prediction. While this is a sound algorithm for evaluating applicable quantities, it does not apply generally to all counterfactual quantities. Specifically, it only applies in quantities of the form $P(\*Y_{\*x} \mid \*z)$, in which only one intervention is applied before the conditioning bar, and observations are present after the conditioning bar. This excludes quantities such as $P(y_x, y'_{x'})$, which is a joint probability of two counterfactual terms in different worlds (one in $X = x$ and another in $X = x'$), as well as quantities such as $P(y_x \mid y'_{x'})$, in which a term in an intervened world $(X = x')$ is after the conditioning bar.
    
    \item Finally, we note that while this paper explored some novel aspects of using GNN and VGAE architectures for causal inference, our paper focuses on GANs, a different class of models.
\end{enumerate}

%% file: section/D_examples.tex
\section{Further Discussions and Examples} \label{app:examples}

In this section, we supplement the main points of the paper with further discussions and examples. In Section \ref{sec:cg-assumption}, we expand on the impossibility of performing inferences across layers of the PCH in general and why assumptions are necessary. In particular, we discuss the implications of the causal diagram assumption and compare it with alternatives. In Section \ref{app:cli-discussion}, we expand on the reasoning behind introducing the discussion of constraints separately for each layer (as in Def.~\ref{def:gli-consistency}). In Section \ref{app:examples-properties}, we ground the theoretical results of Secs.~\ref{sec:data-structure} and \ref{sec:ncm-ctf-id} by walking through examples with actual constructions of NCMs. In Section \ref{app:markovianity}, we discuss the Markovianity assumption, why it simplifies causal inference, whether it's plausible, and why the approach proposed in this paper is more general for \emph{not} making this assumption. 

\subsection{Necessity of Causal Assumptions}
\label{sec:cg-assumption}

As emphasized at the beginning of Sec.~\ref{sec:data-structure}, the entire discussion on constraints in this paper is due to the Neural Causal Hierarchy Theorem (N-CHT) \citep[Corol.~1]{xia:etal21}, following the Causal Hierarchy Theorem from \citet{bareinboim:etal20} that states that distributions from lower layers of the hierarchy almost surely underdetermine higher layers in a measure-theoretic sense. In other words, it is not possible to make a statement about a higher layer (e.g. interventions, counterfactuals) with only knowledge from lower layers (e.g. observations). 
This result is well-understood in the causal inference field, and Pearl eloquently states in the Book of Why that 
``without the causal model, we could not go from rung [layer] one to rung [layer] two. This is why deep-learning systems (as long as they only use rung-one data and do not have a causal model) will never be able to answer questions about interventions (...)'' \citep[p.~32]{pearl:mackenzie2018}. 

To further ground this discussion, consider the following examples.

\begin{example}
\label{ex:l1-l2-nocollapse}
Consider a scenario where the board of education of a city is deciding on policies for school curricula. Data is collected from schools across the city, noting the textbooks ($X$) used in the schools and the GPA of students from that school ($Y$). For simplicity, let $X$ and $Y$ be binary variables, where $X=1$ represents the use of a new textbook, and $Y=1$ represents a high average GPA. Suppose the true model $\cM^*$ describing the relationship between $X$ and $Y$ is specified as follows:
\begin{equation}
\cM^* = 
\begin{cases}
    \*U &= \{U_{XY}, U_{Y1}, U_{Y2}, U_{Y3}\} \\
    \*V &= \{X, Y\} \\
    \cF &=
    \begin{cases}
        f_X(u_{XY}) &= u_{XY} \\
        f_Y(x, u_{XY}, u_{Y1}, u_{Y2}, u_{Y3}) &=
        \begin{cases}
            \neg x \wedge u_{Y1} & \text{if } u_{XY} = 0 \\
            (x \oplus u_{Y2}) \vee u_{Y3} & \text{if } u_{XY} = 1
        \end{cases}
    \end{cases} \\
    P(\*U) &= U_{XY}, U_{Y1} \sim Bern(0.2), U_{Y2} \sim Bern(0.5), U_{Y3} \sim Bern(0.6)
\end{cases}
\end{equation}
Here, $\neg$, $\wedge$, $\vee$, and $\oplus$ indicate bitwise negation, AND, OR, and XOR operators, respectively. The board of education would like to see if they should implement a policy to switch all schools to the new textbook. Formally, they will decide to do this if 
\begin{equation}
P(Y_{X=1} = 1) > P(Y_{X=0} = 1) 
\end{equation}
that is, if the  effect of incorporating the new textbook on high GPA is higher than not incorporating it.

\begin{table}[h]
    \centering
    \begin{tabular}{l|l|l}
    \hline
    $X$ & $Y$ & $P(X, Y)$ \\ \hline \hline
    0   & 0   & 0.64      \\ \hline
    0   & 1   & 0.16      \\ \hline
    1   & 0   & 0.04      \\ \hline
    1   & 1   & 0.16      \\ \hline
    \end{tabular}
    \caption{Observational data for Example \ref{ex:l1-l2-nocollapse}.}
    \label{tab:ex-l1-l2-pv}
\end{table}

However, the board does not have access to the true SCM $\cM^*$. Instead, they have only collected observational data, shown in terms of probabilities in Table \ref{tab:ex-l1-l2-pv}. One can verify that these are the true observational probabilities as induced by $\cM^*$. Given this data, one may be tempted to train a model to match this observational distribution. For example, suppose the board's data science team learns the following model: 

\begin{equation*}
\hM = 
\begin{cases}
    \widehat{\*U} &= \{\widehat{U}_X, \widehat{U}_Y\} \\
    \*V &= \{X, Y\} \\
    \widehat{\cF} &=
    \begin{cases}
        \hat{f}_X(\widehat{u}_X) &= \mathbbm{1}\{\widehat{u}_X < 0.2\} \\
        \hat{f}_Y(x, \widehat{u}_Y) &=
        \begin{cases}
            \mathbbm{1}\{\widehat{u}_Y < 0.2\} & \text{if } x = 0 \\
            \mathbbm{1}\{\widehat{u}_Y < 0.8\} & \text{if } x = 1
        \end{cases}
    \end{cases} \\
    P(\widehat{\*U}) &= \widehat{U}_X, \widehat{U}_Y \sim Unif(0, 1)
\end{cases}
\end{equation*}

One can verify that $\widehat{M}$ induces the same $P(\*V)$ and therefore matches $\cM^*$ in $\cL_1$. Computing the interventional effect from this model leads to: 
\begin{eqnarray} 
P^{\hM}(Y_{X=0} = 1) = P(\widehat{U}_Y < 0.2) = 0.2, \\
P^{\hM}(Y_{X=1} = 1) = P(\widehat{U}_Y < 0.8) = 0.8
\end{eqnarray}
 Since $P^{\hM}(Y_{X=1} = 1) > P^{\hM}(Y_{X=0} = 1)$, this seems to suggest that implementing the new textbook is beneficial to the students' population. However, computing the same values from  $\cM^*$ reveals that
 \begin{eqnarray}
P^{\cM^*}(Y_{X=0} = 1) = 0.32, \\ 
P^{\hM}(Y_{X=1} = 1) = 0.16.  	
 \end{eqnarray}
In other words, implementing the new textbook would actually harm students' performance on average.

The implemented model $\hM$ assigns $P(y_x) = P(y \mid x)$, failing to account for confounding between $X$ and $Y$. Still, even if confounding is considered, it is unclear to what extent $X$ and $Y$ are confounded. In fact, even the causal direction is not known. For example, perhaps schools with better-performing students opted to use the new textbook to challenge the students, resulting in a flipped causal relationship. Without further information about the generating model, one cannot infer anything about the interventional distribution $P(y_x)$, even if the trained model matches the true model on the observational level.,

\hfill $\blacksquare$
\end{example}

\begin{example}
\label{ex:l2-l3-nocollapse}
Consider again the board of education situation from Example \ref{ex:l1-l2-nocollapse}. Suppose, after conducting further experimental studies, the board discovers that (1) $Y$ has no causal effect on $X$ ($P(x_y) = P(x)$), and (2) $P(Y_{X=0}=1) = 0.32$ and $P(Y_{X=1} = 1) = 0.16$ as discovered in the previous example. Now, all $\cL_2$ information from $\cM^*$ is provided as well. The board realizes that implementing the new textbook may not be the best idea and is now wondering whether schools that have already implemented the textbook and have poorly performing students would see improvement if the new textbook was no longer used. Formally, they are wondering about the counterfactual probability $P(Y_{X=0} = 1 \mid X=1, Y=0)$. They construct the following model:

\begin{equation}
\hM = 
\begin{cases}
    \widehat{\*U} &= \{\widehat{U}_{XY}\} \\
    \*V &= \{X, Y\} \\
    \widehat{\cF} &=
    \begin{cases}
        \hat{f}_X(\widehat{u}_{XY}) &= \mathbbm{1}\{\widehat{u}_{XY} < 0.2\} \\
        \hat{f}_Y(x, \widehat{u}_{XY}) &=
        \begin{cases}
            1 & \text{if } \widehat{u}_{XY} < 0.16  \\
            0 & \text{if } 0.16 \leq \widehat{u}_{XY} < 0.84 \\
            \neg x & \text{otherwise}
        \end{cases}
    \end{cases} \\
    P(\widehat{\*U}) &= \widehat{U}_{XY} \sim Unif(0, 1)
\end{cases}
\end{equation}

Note that $\hM$ induces the same $\cL_1$ and $\cL_2$ distributions as $\cM^*$. When querying $\hM$ with the query of interest, one obtains 
\begin{equation}
P^{\hM}(Y_{X=0} = 1 \mid X=1, Y=0) = 0, 	
\end{equation}
indicating that removing the book from the curriculum would have no changes on student performance. However, in reality,
\begin{equation}
P^{\cM^*}(Y_{X=0} = 1 \mid X=1, Y=0) = 1,
\end{equation}
 meaning that removing the book would guarantee an improvement. 
 
 Once again, despite fully matching $\cL_1$ and $\cL_2$ of the true model, $\hM$ is not able to provide any meaningful information about an $\cL_3$ query.
\hfill $\blacksquare$
\end{example}

These examples are written in SCM notation, but equivalent NCM constructions can be found with the same distributions (following from \citep[Thm.~1]{xia:etal21}). These examples illustrate that there can be many models which match on a lower layer but fail to match on a higher layer. Hence, a model (such as an NCM) trained on data from a lower layer cannot be expected to produce correct inferences from higher layers (i.e. crossing between layers is impossible). For more examples, see Examples 7-9 from \citet{bareinboim:etal20}. Also, Example 2 from \citet{xia:etal21} shows a case where four different NCMs all match in the given observational data but have drastically differing outputs for an $\cL_2$ query, and none of them are correct. Evidently, NCMs cannot circumvent the impossibility described by the CHT despite their expressiveness, as expressivity does not guarantee that the learned content will match that of the true generating model.

Given this impossibility, the natural course of action to proceed in problems of cross-layer inference is to study under what assumptions inferences could potentially be made. Obviously, the problem could be solved if the mechanisms and noise of the true generating SCM are known, but this is almost never the case in practice. In this work, we assume that we are given partial knowledge of the true generating SCM in the form of a causal diagram, a qualitative description of the underlying mechanisms and noise. This assumption is commonly used in the causal inference literature starting from \citet{pearl:95a}, where the field started in its modern form, and later shown in several classic causal inference applications such as in the identification of causal effects with do-calculus \citep[Ch.~4]{pearl:2k} or counterfactual identification and their axioms \citep[Ch.~7]{pearl:2k}. There are compelling cases in which a data analyst may have a qualitative understanding of the causal relationships of the environment to form a causal graph \citep{bottou2013counterfactual,pearl:2k}. These relationships are explored in several previous counterfactual studies using symbolic approaches such as \citet{galles:pea98}, \citet{shpitser:pea07}, and \citet{correa:etal21}, to cite a few.

The reason why the causal diagram assumption can be utilized to solve counterfactual inference problems is due to the constraints implied by it. While many SCMs may agree on lower layer quantities but disagree on higher layer quantities, the pool of plausible models shrinks dramatically after falsifying all SCMs that do not match the given constraints. Specifically these constraints are the $\cL_i$ equality constraints discussed in Appendix \ref{app:proof-G-cons}, which led to the notion of $\cG^{(\cL_i)}$-consistency discussed in Sec.~\ref{sec:data-structure}. Still, even with these assumptions, inferring counterfactual quantities is a non-trivial task. Despite having a constrained setting, it is still not guaranteed that a counterfactual quantity is computable from observational and interventional data, nor is it obvious how to compute such a quantity even if it is computable. This results in the counterfactual identification and estimation problems discussed in Secs.~\ref{sec:ncm-ctf-id} and \ref{sec:ncm-estimation}.

\begin{figure}[h]
    \centering
    \begin{tikzpicture}[xscale=1.5, yscale=1]
		\node[draw, circle] (X) at (-1, 0) {$X$};
		\node[draw, circle] (Z) at (0, 1) {$\mathbf{Z}$};
		\node[draw, circle] (Y) at (1, 0) {$Y$};
		
		\path [-{Latex[length=2mm]}] (X) edge (Y);
		\path [-{Latex[length=2mm]}] (Z) edge (X);
		\path [-{Latex[length=2mm]}] (Z) edge (Y);
		\path [{Latex[length=2mm]}-{Latex[length=2mm]}, dashed, bend left] (Z) edge (Y);
	\end{tikzpicture}
	\caption{Graph with backdoor set $\mathbf{Z}$.}
	\label{fig:bd-graph}
\end{figure}
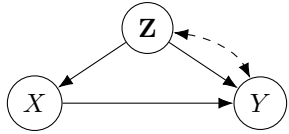

While assumptions are necessary to perform counterfactual inference, other forms of assumptions aside from the causal diagram may also lead to positive results. For example, many works in the intersection of neural models and causal inference have achieved success under the assumption of conditional ignorability (see references cited in Footnote \ref{ft:bd-models}). This assumption states that, under treatment $X$, outcome $Y$, and co-variates $\mathbf{Z}$, we have
\begin{equation}\label{eq:ignorability}
    X \indep Y_{x} \mid \mathbf{Z},
\end{equation}
or in words, $X$ is independent of the counterfactual variable $Y_x$ when conditioned on $\*Z$. This assumption is also known as back-door admissibility and is encoded in causal diagrams like Fig.~\ref{fig:bd-graph}. 
In this paper, we take the graphical language to validate assumptions, which may appear in different forms than Eq.~\ref{eq:ignorability} above. 

Another possibility is to assume the availability of all $\cL_2$ data. This data, combined with the faithfulness assumption, can be used to learn the true causal diagram \citep{kocaoglu:etal17}. If this is the approach used, then the work in this paper can be applied on top of this learned causal diagram. There are examples of counterfactual queries that are still not identifiable even when given the graph and all $\cL_2$ information (see Example \ref{ex:non-id}). Still, this relies on the user having an abundance of experimental data, which is not the most common setting in practice. The work in this paper is agnostic to the approach for acquiring the causal diagram, so it can still be used whether the graph is obtained from experiments or from other means such as through expert knowledge.

Finally, although it is out of the scope of this work, one may envision ways to relax the causal diagram assumption. One example is to work with an equivalence class (EC) of causal diagrams which can be learned from purely observational ($\cL_1$) data \citep{jaber2020cd}. 
Performing causal inferences using the EC comes with its own merits and faces a different set of challenges. In terms of merits, the work relies on fewer assumptions as it is entirely data-driven. In terms of challenges, for example, many interventional and counterfactual quantities are not identifiable since the equivalence class is strictly weaker than the causal diagram itself (for obvious reasons). Such works are complementary to this work since the data scientist can choose for themselves whether they would prefer to make more assumptions for stronger results. 

Ultimately, the goal of this work is not to promote one type of assumption over another but to offer a toolbox to solve counterfactual inference problems once the assumptions are made. The results of this paper are simply the first steps for solving counterfactual inferences in a well-understood setting. A promising direction for future work is to further relax and generalize these approaches when the diagram is not fully known.

\subsection{Discussion on $\cL_i$ Constraints}
\label{app:cli-discussion}

\begin{example}
\begin{figure}[h]
	\centering
	\begin{subfigure}{0.4\textwidth}
	    \centering
	    \begin{tikzpicture}[xscale=1.5, yscale=1]
    		\node[draw, circle] (X) at (-1, 0) {$X$};
    		\node[draw, circle] (Z) at (0, 0) {$Z$};
    		\node[draw, circle] (Y) at (1, 0) {$Y$};
    		
    		\path [-{Latex[length=2mm]}] (X) edge (Z);
    		\path [-{Latex[length=2mm]}] (Z) edge (Y);
    	\end{tikzpicture}
    	\caption{True graph.}
	\end{subfigure}
	\begin{subfigure}{0.4\textwidth}
	    \centering
	    \begin{tikzpicture}[xscale=1.5, yscale=1]
    		\node[draw, circle] (X) at (-1, 0) {$X$};
    		\node[draw, circle] (Z) at (0, 0) {$Z$};
    		\node[draw, circle] (Y) at (1, 0) {$Y$};
    		
    		\path [-{Latex[length=2mm]}] (Z) edge (X);
    		\path [-{Latex[length=2mm]}] (Z) edge (Y);
    	\end{tikzpicture}
    	\caption{Alternate graph.}
	\end{subfigure}
	\caption{Graphs for illustrating equality constraints.}
	\label{fig:ex-constraint-graph}
\end{figure}
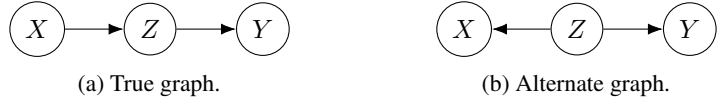

To illustrate this concept of graphical constraints on different layers with an example, consider the graphs in Fig.~\ref{fig:ex-constraint-graph}. Suppose graph (a) is the causal diagram of the true SCM $\cM^*$. Graph (a) implies different constraints across all three layers. For example, on layer 1, we know by d-separation that $P^*(y \mid z, x) = P^*(y \mid z)$. One could easily train a model $\cM$ given data on $P^*(X, Y, Z)$ such that the model outputs $P^{\cM}(X, Y, Z) = P^*(X, Y, Z)$, and hence, the equality constraint $P^{\cM}(y \mid z, x) = P^{\cM}(y \mid z)$ holds in $\cM$ as well. In this case $\cM$ would be considered $\cG^{(\cL_1)}$-consistent with $\cM^*$.

However, this is not to say that $\cM$ takes the constraints of $\cG$ on higher layers. For example, $\cM$ could be another SCM with causal diagram (b), which also implies the constraint $P(y \mid z, x) = P(y \mid z)$. The similarities break down on higher layer constraints. For example, on layer 2, graph (a) implies that $P^*(x_z) = P^*(x)$ (i.e. $Z$ has no causal effect on $X$) and $P^*(z_x) = P^*(z \mid x)$ (i.e. the causal effect of $X$ on $Z$ is equal to the association). If graph (b) truly reflects the mechanisms of $\cM$, then neither of these two constraints would hold for $\cM$ since the causal direction is reversed. $\cM$ would \emph{not} be $\cG^{(\cL_2)}$-consistent with $\cM^*$, greatly reducing its inferential capabilities on layer 2 quantities induced by $\cM^*$.

Layer 1 and layer 2 constraints are formalized through the concepts of bayesian networks and causal bayesian networks, respectively. Although not yet formally studied, there are further constraints encoded by $\cG$ corresponding to layer 3 distributions. For example, graph (a) implies that $P^*(y_z \mid z') = P^*(y_z)$ and that $P^*(y_{x'}, z_x) = P^*(y_z, z_x)$. These constraints are vital for performing inference on layer 3, which is why $\cG^{(\cL_3)}$-consistency is an important property.

\hfill $\blacksquare$
\end{example}

\subsection{Properties of the $\cG$-NCM}
\label{app:examples-properties}

Sec.~\ref{sec:data-structure} emphasizes that the $\cG$-NCM has the properties of $\cG^{(\cL_3)}$-consistency and $\cL_3$-$\cG$ expressiveness. In this section, we expand on why these two properties are nontrivial.

The proof that the $\cG$-NCM is $\cG^{(\cL_3)}$-consistent arises naturally from its structure and the fact that it is a type of SCM. We first note that any other type of structure may struggle to achieve $\cG^{(\cL_3)}$-consistency. To illustrate this point, consider the following example.

\begin{example}
    \label{ex:l3-expression-wrapper}
    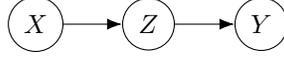
\begin{figure}[h]
	    \centering
	    \begin{tikzpicture}[xscale=1.5, yscale=1]
    		\node[draw, circle] (X) at (-1, 0) {$X$};
    		\node[draw, circle] (Z) at (0, 0) {$Z$};
    		\node[draw, circle] (Y) at (1, 0) {$Y$};
    		
    		\path [-{Latex[length=2mm]}] (X) edge (Z);
    		\path [-{Latex[length=2mm]}] (Z) edge (Y);
    	\end{tikzpicture}
    	\caption{Graph for following examples.}
    	\label{fig:cg-ex-l3-wrapper}
	\end{figure}
    
    When representing a collection of distributions, perhaps the most simple model that comes to mind is an expression wrapper.
    \begin{definition}[Expression Wrapper]
        Given a collection of distributions $\bbP$, an \emph{expression wrapper} $M^{\bbP}$ defined over $\bbP$ is a collection of probability models $M_{\bbP} = \{M_P : P \in \bbP\}$, where each $M_P$ is an $\cL_1$-expressive model such as a normalizing flow that can express any distribution $P$.
        \hfill $\blacksquare$
    \end{definition}
    It is obvious by construction that regardless of $\bbP$, there exists an expression wrapper defined over $\bbP$ that can express all of the distributions in $\bbP$. If we are interested in an application modeled by true SCM $\cM^*$, and $\bbP = \cL_3(\cM^*)$, that implies that there exists an expression wrapper that is $\cL_3$-consistent with $\cM^*$. The trouble arises when we aim to enforce $G^{(\cL_3)}$-consistency. Consider the simple 3 variable case where $\cG$ is shown in Fig.~\ref{fig:cg-ex-l3-wrapper}. Even on layer 2, $\cG$ encodes several constraints, including but not limited to:
    \begin{enumerate}
        \item $P(x) = P(x_z) = P(x_y) = P(x_{zy})$
        \item $P(z) = P(z_y)$
        \item $P(y_{xz}) = P(y_z) = P(y \mid z) = P(y \mid x, z)$
        \item $P(z_{xy}) = P(z_x) = P(z \mid x)$
        \item $P(y_x) = P(y \mid x)$
        \item $P(\{x, z\}_y) = P(x, z)$
        \item $P(\{y, z\}_x) = P(y, z \mid x)$
    \end{enumerate}
    There are even more constraints on Layer 3, such as \\
    $P(\{x, z\}_y) = P(x, z) = P(x, z_y) = P(x_y, z) = P(x_y, z_{y'}) = P(x_{y'}, z_y) = P(x_{y'}, z_{y'}) = P(x_{y'}, z_{y''}) = P(x_z, z) = P(x_z, z_y) = P(x_z, z_{y'}) = P(x_{z'}, z) = P(x_{z'}, z_y) = P(x_{z'}, z_{y'}) = P(x_{z, y}, z) = P(x_{z, y}, z_y) = P(x_{z, y}, z_{y'}) = P(x_{z', y}, z) = P(x_{z', y}, z_y) = P(x_{z', y}, z_{y'}) = P(x_{z, y'}, z) = P(x_{z, y'}, z_y) = P(x_{z, y'}, z_{y'}) = P(x_{z, y'}, z_{y''}) = P(x_{z', y'}, z) = P(x_{z', y'}, z_y) = P(x_{z', y'}, z_{y'}) = P(x_{z', y'}, z_{y''}) = P(x, z_x) = P(x, z_{x, y}) = P(x, z_{x, y'}) = P(x_y, z_x) = P(x_y, z_{x, y}) = P(x_y, z_{x, y'}) = P(x_{y'}, z_x) = P(x_{y'}, z_{x, y}) = P(x_{y'}, z_{x, y'}) = P(x_{y'}, z_{x, y''}) = P(x_z, z_x) = P(x_z, z_{x, y}) = P(x_z, z_{x, y'}) = P(x_{z'}, z_x) = P(x_{z'}, z_{x, y}) = P(x_{z'}, z_{x, y'}) = P(x_{z, y}, z_x) = P(x_{z, y}, z_{x, y}) = P(x_{z, y}, z_{x, y'}) = P(x_{z', y}, z_x) = P(x_{z', y}, z_{x, y}) = P(x_{z', y}, z_{x, y'}) = P(x_{z, y'}, z_x) = P(x_{z, y'}, z_{x, y}) = P(x_{z, y'}, z_{x, y'}) = P(x_{z, y'}, z_{x, y''}) = P(x_{z', y'}, z_x) = P(x_{z', y'}, z_{x, y}) = P(x_{z', y'}, z_{x, y'}) = P(x_{z', y'}, z_{x, y''})$, \\
    among several others. Furthermore, each of these constraints must hold for every value of $x, y, z, x', y', z'$, etc., in their respective domains.
    
    The sheer number of constraints even in a 3-variable case is so large that it is difficult to list them all, let alone train a model to actually enforce all of these constraints. The causal diagram parsimoniously encodes all of these constraints implicitly, and therefore, the NCM enforces these constraints by construction when it is used as an inductive bias. However, these constraints must be enforced manually when using a model such as the expression wrapper. Not only must every distribution $P \in \bbP$ be trained, which is already a difficult task due to the size of $\bbP$, all of these exponentially many constraints must be simultaneously enforced if $\cG^{(\cL_3)}$-consistency is desired. This is simply not feasible.
    \hfill $\blacksquare$
\end{example}

There are also models that achieve $\cG^{(\cL_2)}$-consistency but not $\cG^{(\cL_3)}$-consistency. Consider the following example.

\begin{example}
    \label{ex:cbn-expressive}
    Recall the definition of Causal Bayesian networks.
    \begin{definition}[Causal Bayesian Network (CBN) {\citep[Def.~16]{bareinboim:etal20}}]
        \label{def:cbn}
        Given observed variables $\mathbf{V}$, let $\mathbf{P}^*$ be the collection of all interventional distributions $P(\mathbf{V} \mid do(\mathbf{x}))$, $\mathbf{X} \subseteq \mathbf{V}$, $\mathbf{x} \in \cD_{\mathbf{X}}$. A causal diagram $\cG$ is a Causal Bayesian Network for $\mathbf{P}^*$ if for every intervention $do(\mathbf{X} = \mathbf{x})$ and every topological ordering $<$ of $\cG_{\overline{X}}$ through its directed edges,
        
        \begin{enumerate}[label=(\roman*)]
            \item $P(\mathbf{V} \mid do(\mathbf{X} = \mathbf{x}))$ is semi-Markov relative to $\cG_{\overline{\mathbf{X}}}$.
            
            \item For every $V_i \in \mathbf{V} \setminus \mathbf{X}$, $\mathbf{W} \subseteq \mathbf{V} \setminus (\Pai{i}^{\mathbf{X}+} \cup \mathbf{X} \cup \{V_i\})$:
            $$P(v_i \mid do(\mathbf{x}), \pai{i}^{\mathbf{x}+}, do(\mathbf{w})) = P(v_i \mid do(\mathbf{x}), \pai{i}^{\mathbf{x}+})$$,
            
            \item For every $V_i \in \mathbf{V} \setminus \mathbf{X}$, let $\Pai{i}^{\mathbf{X}+}$ be partitioned into two sets of confounded and unconfounded parents, $\Pai{i}^c$ and $\Pai{i}^u$ in $\cG_{\overline{\mathbf{X}}}$. Then
            \begin{align*}
                & P(v_i \mid do(\mathbf{x}), \pai{i}^c, do(\pai{i}^u)) \\
                &= P(v_i \mid do(\mathbf{x}), \pai{i}^c, \pai{i}^u)
            \end{align*}
        \end{enumerate}
        
        Here, $\Pai{V_i}^{\*x+} = \Pai{}(\{V \in \*C_{\overline{\*X}}(V_i) : V \leq V_i\})$, with $\*C_{\overline{\*X}}$ referring to the corresponding C-component in $\cG_{\overline{\*X}}$ and $\leq$ referring to the topological ordering.
        \hfill $\blacksquare$
    \end{definition}
    To clarify notation used in this definition, $P(\*y \mid do(\*x), \*z) = P(\*y_{\*x} \mid \*z_{\*x})$ for any $\*X, \*Y, \*Z \subseteq \*V$, and $\cG_{\overline{\*X}}$ refers to the graph $\cG$ with the edges into $\*X$ removed.
    
    In words, a causal diagram $\cG$ is a CBN for a collection of distributions $\*P^*$ if $\*P^*$ satisfies the $\cL_2$ equality constraints induced by $\cG$. The CBN explicitly states these constraints, and more details can be found in \citet{bareinboim:etal20}.
    
    Let $\cG$ be some causal diagram induced by true SCM $\cM^*$. Let $M_{\cbn} = \langle \cG, \*P^* \rangle$ be some CBN defined over distributions $\*P^*$. Then, by definition, $M_{\cbn}$ is $\cG^{(\cL_2)}$-consistent with $\cM^*$. However, $M_{\cbn}$ would not be $\cG^{(\cL_3)}$-consistent with $\cM^*$ since the constraints in the definition do not include any distributions outside of layer 2. In fact, $\*P^*$ does not even have to be defined with layer 3 distributions.
    
    As an example, suppose $\cG$ is the graph in Fig.~\ref{fig:cg-ex-l3-wrapper}, and that $\cL_1(\cM^*) = \cL_1(M_{\cbn})$, that is, $\*P^*$ matches the layer 1 distribution of $\cM^*$. Then, since $M_{\cbn}$ is $\cG^{(\cL_2)}$-consistent, one can derive $P(y_x) = P(y \mid x)$ and conclude that $P^{M_{\cbn}}(y_x) = P^{\cM^*}(y_x)$, successfully identifying a layer 2 quantity. However, if a model were $\cG^{(\cL_3)}$-consistent with $\cM^*$, one could derive quantities such as $P(y_x, z_x) = P(y_{xz}, z_x) = P(y_z, z_x) = P(y_z)P(z_x) = P(y \mid z)P(z \mid x)$, also computable from observational quantities. This is not guaranteed to be true with the CBN, since counterfactual quantities like $P(y_z, z_x)$ are not guaranteed to be defined in $\*P^*$, and even if they were, they could take any values since no part of the definition of CBNs constrains them.
    \hfill $\blacksquare$
\end{example}

There also exist models that may be $\cG^{(\cL_3)}$-consistent but are not $\cL_3$-$\cG$ expressive. Any model class in which Markovianity is enforced is an example of such a case, as shown in the following.

\begin{example}
    \label{ex:markov-expressive}
    
    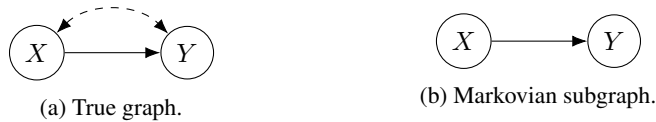
\begin{figure}[h]
    	\centering
    	\begin{subfigure}{0.4\textwidth}
    	    \centering
        	\begin{tikzpicture}[xscale=1, yscale=1]
        		\node[draw, circle] (X) at (-1, 0) {$X$};
        		\node[draw, circle] (Y) at (1, 0) {$Y$};
        		
        		\path [-{Latex[length=2mm]}] (X) edge (Y);
        		\path [{Latex[length=2mm]}-{Latex[length=2mm]}, dashed, bend left, out=45, in=135] (X) edge (Y);
        	\end{tikzpicture}
        	\caption{True graph.}
        	\label{fig:cg-ex-markov-true}
    	\end{subfigure}
    	\begin{subfigure}{0.4\textwidth}
    	    \centering
    	    \begin{tikzpicture}[xscale=1, yscale=1]
        		\node[draw, circle] (X) at (-1, 0) {$X$};
        		\node[draw, circle] (Y) at (1, 0) {$Y$};
        		
        		\path [-{Latex[length=2mm]}] (X) edge (Y);
        	\end{tikzpicture}
        	\caption{Markovian subgraph.}
        	\label{fig:cg-ex-markov-alt}
    	\end{subfigure}
    	\caption{Graphs for Example \ref{ex:markov-expressive}.}
    	\label{fig:cg-ex-markov}
    \end{figure}
    
    Given the causal diagram $\cG$ from the true SCM $\cM^*$, consider $\widetilde{\Omega}(\cG)$, the set of all Markovian SCMs that induce causal diagrams with a subset of the edges of $\cG$. By construction, $\widetilde{\Omega}(\cG)$ is $\cG^{(\cL_3)}$-consistent with $\cM^*$ since no constraint can be removed by removing edges, so any SCM in $\widetilde{\Omega}(\cG)$ has at least as many $\cL_3$ constraints as those induced by $\cG$.
    
    Suppose $\cG$ is the bow graph as shown in Fig.~\ref{fig:cg-ex-markov}, and the true SCM $\cM^*$ is defined as
    \begin{equation*}
        \cM^* =
        \begin{cases}
            \*U & = \{U_{XY}, U_Y\}, \text{ both binary}\\
            \*V &= \{X, Y\}, \text{ both binary}\\
            \cF &= \begin{cases}
                f_X(u_{XY}) &= u_{XY} \\
                f_Y(x, u_{XY}, u_{Y}) &= u_{XY} \oplus u_Y
            \end{cases} \\
            P(\*U) & \text{ defined such that } U_{XY} \sim \bern(0.5), U_{Y} \sim \bern(0.1)
        \end{cases}
    \end{equation*}
    Evidently, $X$ and $Y$ are highly correlated, as $P(Y = 1 \mid X = 1) = 0.9$, but $P(Y = 1) = 0.5$. However, this correlation is not causal, as the input $X$ remains unused by $f_Y$, implying that $P(Y_{X = 1} = 1) = P(Y = 1) = 0.5$. It is clear that no model in $\widetilde{\Omega}(\cG)$ will be able to represent $\cM^*$ on all three layers. For example, one may construct $\cM' \in \widetilde{\Omega}(\cG)$ as follows
    \begin{equation*}
        \cM' =
        \begin{cases}
            \*U' & = \{U_{X}, U_Y\}, \text{ both binary}\\
            \*V &= \{X, Y\}, \text{ both binary}\\
            \cF' &= \begin{cases}
                f'_X(u_{X}) &= u_{X} \\
                f'_Y(x, u_{Y}) &= x \oplus u_Y
            \end{cases} \\
            P(\*U') & \text{ defined such that } U_{X} \sim \bern(0.5), U_{Y} \sim \bern(0.1)
        \end{cases}
    \end{equation*}
    This model induces the graph $\cG'$ as shown in Fig.~\ref{fig:cg-ex-markov-alt}, so it is a valid model from $\widetilde{\Omega}(\cG)$. Moreover, it can even reproduce $\cL_1$ quantities from $\cM^*$, as it is clear that $P^{\cM'}(Y = 1 \mid X = 1) = 0.9$, and $P^{\cM'}(Y = 1) = 0.5$. This makes $\cM'$ $\cL_1$-consistent with $\cM^*$. However, on layer 2, the similarities break down. While in the true model, $X$ has no effect on $Y$, in this case, the Markovian nature of $\cM'$ means that $P^{\cM'}(Y_{X=1} = 1) = P^{\cM'}(Y = 1 \mid X = 1) = 0.9$, resulting in a different value.
    
    Clearly $\cM'$ is not $\cL_2$-consistent with $\cM^*$, let alone $\cL_3$-consistent. However, no other model from $\widetilde{\Omega}(\cG)$ would be able to achieve $\cL_2$-consistency either. In a Markovian model, the correlation between $X$ and $Y$ must be modeled through a causal effect from either $X$ to $Y$ or from $Y$ to $X$, and both of these options would violate $\cL_2$-consistency with $\cM^*$.
    \hfill $\blacksquare$
\end{example}

Other examples of model classes that are $\cG^{(\cL_3)}$-consistent but not $\cL_3$-$\cG$ expressive include parametric SCMs in which the parameterization is not fully expressive, such as linear SCMs.

As shown in Thm.~\ref{thm:l3-g-expressiveness}, the class of $\cG$-NCMs is $\cL_3$-$\cG$ expressive. It may not be immediately clear why this expressivity holds on all three layers instead of only on layers 1 or 2. The intuition is that with the expressivity of neural networks, a $\cG$-NCM can directly model the functions of the true SCM $\cM^*$ instead of indirectly matching its distributions. By matching the functions and utilizing the probability integral transform on uniform random variables to match $P(\*U)$, it is guaranteed that one can construct a $\cG$-NCM to directly mimic the behavior of $\cM^*$ on the functional level, implying equivalence on all three layers. Consider the following example.

\begin{example}
    \label{ex:l3-g}
    \begin{figure}[h]
    	\centering
	    \begin{tikzpicture}[xscale=1, yscale=1]
    		\node[draw, circle] (X) at (-1, 0) {$X$};
    		\node[draw, circle] (Y) at (1, 0) {$Y$};
    		
    		\path [-{Latex[length=2mm]}] (X) edge (Y);
    	\end{tikzpicture}
    	\caption{Graph for Example \ref{ex:l3-g}.}
    	\label{fig:cg-ex-l3-g}
    \end{figure}
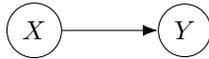

    Suppose we are studying the effects of a drug ($X$) on the recovery of a disease ($Y$). For simplicity, $X$ and $Y$ are both binary, $X=1$ represents the drug being taken, and $Y=1$ represents a full recovery. There are two hypothesized models. In $\cM_1$, taking the drug helps people without a certain gene ($U_Y = 1$) but actively harms people with the gene ($U_Y = 0$). In $\cM_2$, the drug has no effect. The two SCMs are clarified as follows:
    \begin{equation*}
        \cM_1 =
        \begin{cases}
            \*U_1 & = \{U_X, U_Y\}, \text{ both binary}\\
            \*V &= \{X, Y\}, \text{ both binary}\\
            \cF_1 &= \begin{cases}
                f^1_X(u_{X}) &= u_{X} \\
                f^1_Y(x, u_{Y}) &= x \oplus u_Y
            \end{cases} \\
            P(\*U_1) & \text{ defined such that } U_{X}, U_Y \sim \bern(0.5)
        \end{cases}
    \end{equation*}
    \begin{equation*}
        \cM_2 =
        \begin{cases}
            \*U_2 & = \{U_X, U_Y\}, \text{ both binary}\\
            \*V &= \{X, Y\}, \text{ both binary}\\
            \cF_2 &= \begin{cases}
                f^2_X(u_{X}) &= u_{X} \\
                f^2_Y(x, u_{Y}) &= u_Y
            \end{cases} \\
            P(\*U_2) & \text{ defined such that } U_{X}, U_Y \sim \bern(0.5)
        \end{cases}
    \end{equation*}
    It turns out that these two hypothesized models match in both $\cL_1$ and $\cL_2$. For example, $P^{\cM_1}(Y_{X = 1} = 1) = P^{\cM_2}(Y_{X = 1} = 1) = 0.5$. The difference lies in their induced distributions for $\cL_3$. For example, $P^{\cM_1}(Y_{X = 1} = 1 \mid X = 0, Y = 1) = 0$. In words, in $\cM_1$, if it is observed that someone did not take the drug and recovered from the disease, then it is guaranteed that they would fail to recover if they had taken the drug. On the other hand, $P^{\cM_2}(Y_{X=1} = 1 \mid X = 0, Y = 1) = P^{\cM_2}(Y = 1 \mid Y = 1) = 1$ since $X$ simply has no effect on the disease.
    
    The reasoning for this difference at the functional level is that both $f^1_Y$ and $f^2_Y$ have two different behaviors with respect to $X$ depending on the setting of $U_Y$. When $U_Y = 0$, $f^1_Y(x) = x$, while $f^2_Y(x) = 0$. When $U_Y = 1$, $f^1_Y(x) = 1 - x$, while $f^2_Y(x) = 1$. In both cases, the functions are different between $\cM_1$ and $\cM_2$. Notably, the output of $f^1_Y$ depends on $X$ while the output of $f^2_Y$ does not.
    
    We can reproduce this difference in behavior when constructing $\cG$-NCMs to represent these SCMs.
    
    Recall the binary step activation function,
    \begin{equation}
        \sigma(z) =
        \begin{cases}
            0 & z < 0 \\
            1 & z \geq 0.
        \end{cases}
    \end{equation}
    Using this activation function, we can create various useful neural network functions like the following:
    \begin{align*}
        \hat{f}_{OR}(z_1, \dots, z_n) &= \sigma \left( \sum_{i = 1}^{n} z_i - 1 \right) & \text{bitwise OR} \\
        \hat{f}_{AND}(z_1, \dots, z_n) &= \sigma \left( \sum_{i = 1}^{n} z_i - n \right) & \text{bitwise AND} \\
        \hat{f}_{NOT}(z) &= \sigma(-z) & \text{bitwise NOT} \\
        \hat{f}_{\geq p}(z) &= \sigma(z - p) & \text{1 if } z \geq p \\
        \hat{f}_{<p}(z) &= \hat{f}_{NOT}(\hat{f}_{\geq p}(z)) & \text{1 if } z < p
    \end{align*}
    
    With these functions, consider the following $\cG$-NCMs, $\hM_1$ and $\hM_2$ that replicate the behavior of $\cM_1$ and $\cM_2$ respectively.
    
    \begin{equation*}
        \hM_1 =
        \begin{cases}
            \widehat{\*U}_1 & = \{\widehat{U}_X, \widehat{U}_Y\} \\
            \*V &= \{X, Y\}, \text{ both binary}\\
            \widehat{\cF}_1 &= \begin{cases}
                \hat{f}^1_X(\widehat{u}_{X}) &= \hat{f}_{<0.5}(\widehat{u}_{X}) \\
                \hat{f}^1_Y(x, \widehat{u}_{Y}) &= \hat{f}_{OR}
                \begin{cases}
                    \hat{f}_{AND}(x, \hat{f}_{NOT}(\hat{f}_{<0.5}(\widehat{u}_Y))) \\
                    \hat{f}_{AND}(\hat{f}_{NOT}(x), \hat{f}_{<0.5}(\widehat{u}_Y))
                \end{cases}
            \end{cases} \\
            P(\widehat{\*U}_1) & \text{ defined such that } \widehat{U}_{X}, \widehat{U}_Y \sim \unif(0, 1)
        \end{cases}
    \end{equation*}
    \begin{equation*}
        \hM_2 =
        \begin{cases}
            \widehat{\*U}_2 & = \{\widehat{U}_X, \widehat{U}_Y\} \\
            \*V &= \{X, Y\}, \text{ both binary}\\
            \cF_2 &= \begin{cases}
                \hat{f}^2_X(\widehat{u}_{X}) &= \hat{f}_{<0.5}(\widehat{u}_{X}) \\
                \hat{f}^2_Y(x, \widehat{u}_{Y}) &= \hat{f}_{<0.5}(\widehat{u}_Y)
            \end{cases} \\
            P(\widehat{\*U}_2) & \text{ defined such that } \widehat{U}_{X}, \widehat{U}_Y \sim \unif(0, 1)
        \end{cases}
    \end{equation*}
    
    Note that $\hM_1$ and $\hM_2$ match $\cM_1$ and $\cM_2$, respectively, on all three layers. For example, $P^{\hM_1}(Y_{X = 1} = 1 \mid X = 0, Y = 1) = 0$ and $P^{\hM_2}(Y_{X = 1} = 1 \mid X = 0, Y = 1) = 1$. This is because the construction of $\hM_1$ and $\hM_2$ were designed to directly match the functions of $\cM_1$ and $\cM_2$.
    \hfill $\blacksquare$
\end{example}

As a disclaimer, note that $\cL_3$-$\cG$ expressiveness means the existence of a model that matches on all three layers, but it does not necessarily mean that such a model can be found. In the case of Example \ref{ex:l3-g}, we could create $\cG$-NCMs to match the hypothesized SCMs because the full specifications of the SCMs were provided. In practice, this is rarely the case, and it may not be possible to determine the correct NCM from training on data alone, which is why the identification problem is important to determine when inferences can be made from a trained NCM.


Determining whether a counterfactual quantity is identifiable from the given data and constraints is an important problem even in the Markovian case. Consider the following simple example.

\begin{example}
    \label{ex:non-id}
    \begin{figure}[h]
        \centering
        \begin{tikzpicture}[xscale=1.5, yscale=1.0]
    		\node[draw, circle] (X) at (-1, 0) {$X$};
    		\node[draw, circle] (Y) at (1, 0) {$Y$};
    		
    		\path [-{Latex[length=2mm]}] (X) edge (Y);
    	\end{tikzpicture}
    	\caption{Graph for Example \ref{ex:non-id}.}
    	\label{fig:ex-non-id}
    \end{figure}
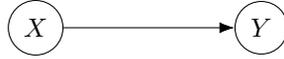
    Consider a simple study on the effect of a new drug ($X$) on survival rate of a sick patient ($Y$), modeled by the following SCM $\cM$:
    \begin{equation*}
        \cM^* =
        \begin{cases}
            \*U & = \{U_X, U_{Y1}, U_{Y2}\}, \text{ all binary}\\
            \*V &= \{X, Y\}, \text{ both binary}\\
            \cF &= \begin{cases}
                f_X(u_X) &= u_X \\
                f_Y(x, u_{Y1}, u_{Y2}) &= 
                \begin{cases}
                    x \oplus u_{Y1} & u_{Y2} = 0 \\
                    u_{Y1} & u_{Y2} = 1
                \end{cases}
            \end{cases} \\
            P(\*U) & \text{ defined such that } U_X, U_{Y1}, U_{Y2} \sim \bern(0.5)
        \end{cases}
    \end{equation*}
    This SCM induces the graph $\cG$ in Fig.~\ref{fig:ex-non-id}. Evidently from observing $\cM^*$, this drug has some effect on patients with this disease. However, $\cM^*$ is unobserved, and we only have data. Suppose we have both observational data and experimental data from testing on $X$. Experiments on $Y$ are not needed since $Y$ has no causal effect on $X$, so this essentially means we have data on all of layer 1 and layer 2. This data is shown in terms of probabilities in Table.~\ref{tab:ex-non-id}.
    
    \begin{table}[h]
        \centering
        \begin{tabular}{l|l|l}
        \hline \hline
        $X$ & $Y$ & $P(X, Y)$ \\ \hline \hline
        0 & 0 & 0.25    \\ \hline
        0 & 1 & 0.25    \\ \hline
        1 & 0 & 0.25    \\ \hline
        1 & 1 & 0.25    \\ \hline
        \end{tabular}
        \quad
        \begin{tabular}{l|l|l}
        \hline \hline
        $x$ & $Y_{X=x}$ & $P(Y_{X=x})$ \\ \hline \hline
        0 & 0 & 0.5    \\ \hline
        0 & 1 & 0.5    \\ \hline
        1 & 0 & 0.5    \\ \hline
        1 & 1 & 0.5    \\ \hline
        \end{tabular}
        \caption{Data for Example \ref{ex:non-id}}
        \label{tab:ex-non-id}
    \end{table}
    
    The data seems to suggest that the drug has no effect on the disease since $P(y_x) = P(y \mid x) = P(y)$ for all values of $y$ and $x$. However, just to be sure, we are interested in evaluating the query $P(Y_{X = 1} = 1 \mid X = 0, Y = 0)$, the probability that someone who did not take the drug and did not recover \emph{would have} recovered had they taken the drug. This quantity is called the probability of sufficiency \citep[Def.~9.2.2]{pearl:2k}.
    
    Running the procedure in Alg.~\ref{alg:ncm-solve-ctfid} on $\cG$ and $\bbZ = \{P(\*V), P(\*V_x)\}$, we may achieve the following two $\cG$-NCMs (with the notation from Example \ref{ex:l3-g})
    \begin{equation*}
        \hM_1 =
        \begin{cases}
            \widehat{\*U} & = \{\widehat{U}_X, \widehat{U}_Y\} \\
            \*V &= \{X, Y\} \\
            \widehat{\cF}_1 &= \begin{cases}
                \hat{f}_X^{(1)}(\widehat{u}_X) &= \hat{f}_{< 0.5} (\widehat{u}_X) \\
                \hat{f}_Y^{(1)}(x, \widehat{u}_Y) &= \hat{f}_{< 0.5}(\widehat{u}_Y)
            \end{cases} \\
            P(\widehat{\*U}) & \text{ defined such that } \hU_X, \hU_Y \sim \unif(0, 1)
        \end{cases}
    \end{equation*}
    
    \begin{equation*}
        \hM_2 =
        \begin{cases}
            \widehat{\*U} & = \{\widehat{U}_X, \widehat{U}_Y\} \\
            \*V &= \{X, Y\} \\
            \widehat{\cF}_2 &= \begin{cases}
                \hat{f}_X^{(2)}(\widehat{u}_X) &= \hat{f}_{< 0.5} (\widehat{u}_X) \\
                \hat{f}_Y^{(2)}(x, \widehat{u}_Y) &= \hat{f}_{OR}(\hat{f}_{AND}(x, \hat{f}_{< 0.5}(\widehat{u}_Y)), \hat{f}_{AND}(\hat{f}_{NOT}(x), \hat{f}_{\geq 0.5}(\widehat{u}_Y)))
            \end{cases} \\
            P(\widehat{\*U}) & \text{ defined such that } \hU_X, \hU_Y \sim \unif(0, 1)
        \end{cases}
    \end{equation*}
    Note that $\hat{f}_Y^{(2)}$ simply resembles the XOR function. Evaluating $P^{\hM_1}(X, Y)$, $P^{\hM_2}(X, Y)$, $P^{\hM_1}(Y_x)$, and $P^{\hM_2}(Y_x)$, one can verify that they match the results in Table \ref{tab:ex-non-id}. However, for our query, observe that
    \begin{align*}
        P^{\hM_1}(Y_{X = 1} = 1 \mid X = 0, Y = 0) &= \frac{P^{\hM_1}(Y_{X = 1} = 1, X = 0, Y = 0)}{P^{\hM_1}(X = 0, Y = 0)} \\
        &= \frac{P(\hU_Y < 0.5, \hU_X \geq 0.5, \hU_Y \geq 0.5)}{P(\hU_X \geq 0.5, \hU_Y \geq 0.5)} \\
        &= \frac{0}{0.25} = 0,
    \end{align*}
    and
    \begin{align*}
        P^{\hM_2}(Y_{X = 1} = 1 \mid X = 0, Y = 0) &= \frac{P^{\hM_2}(Y_{X = 1} = 1, X = 0, Y = 0)}{P^{\hM_2}(X = 0, Y = 0)} \\
        &= \frac{P(\hU_Y < 0.5, \hU_X \geq 0.5, \hU_Y < 0.5)}{P(\hU_X \geq 0.5, \hU_Y < 0.5)} \\
        &= \frac{0.25}{0.25} = 1.
    \end{align*}
    In words, according to $\hM_1$, someone who did not take the drug or recover is guaranteed \emph{not} to recover even if they had taken the drug. On the other hand, according to $\hM_2$, the opposite is true, and that person would be guaranteed to recover if they had taken the drug. Remarkably, despite the fact that we have the graph $\cG$ and \emph{all} data from $\cL_1$ and $\cL_2$, we cannot narrow down the value of our query at all, as it can take any possible value between 0 and 1.
    
    Moreover, when we evaluate the true result of the query from $\cM^*$, we get
    \begin{align*}
        P^{\cM^*}(Y_{X = 1} = 1 \mid X = 0, Y = 0) &= \frac{P^{\cM^*}(Y_{X = 1} = 1, X = 0, Y = 0)}{P^{\cM^*}(X = 0, Y = 0)} \\
        &= \frac{P(U_X = 0, U_{Y1} = 0, U_{Y2} = 0)}{P(U_X = 0, U_{Y1} = 0)} \\
        &= \frac{0.125}{0.25} = \frac{1}{2}.
    \end{align*}
    
    In other words, neither $\hM_1$ nor $\hM_2$ captured the true value of the query. This query is not identifiable from $\bbZ$ and $\cG$. Arbitrarily training a model to fit $\bbZ$ and $\cG$ is almost guaranteed to fail to capture the true query in this case.
    
    \hfill $\blacksquare$
\end{example}

Although Example \ref{ex:non-id} shows a compelling case where one would not want to estimate the query due to non-identifiability, the NCM serves as a robust proxy model in cases where the query is identifiable.

\begin{example}
    \label{ex:id}

    \begin{figure}[h]
        \centering
        \begin{tikzpicture}[xscale=1.5, yscale=1.0]
    		\node[draw, circle] (R) at (-1, 0) {$R$};
    		\node[draw, circle] (J) at (1, 0) {$J$};
    		\node[draw, circle] (E) at (0, -1) {$E$};
    		
    		\path [-{Latex[length=2mm]}] (R) edge (J);
    		\path [-{Latex[length=2mm]}] (R) edge (E);
    		\path [-{Latex[length=2mm]}] (E) edge (J);
    		\path [{Latex[length=2mm]}-{Latex[length=2mm]}, dashed, bend right] (R) edge (E);
    	\end{tikzpicture}
    	\caption{Graph for Example \ref{ex:id}.}
    	\label{fig:ex-id}
    \end{figure}
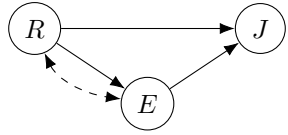
    Suppose we are conducting a study on whether race ($R$) affects an applicant's chances of getting a certain job ($J$). We also collect data on the applicant's education background ($E$). For simplicity, suppose all three are binary variables, $R=1$ implies that the applicant is from a protected class (e.g. black, hispanic), $E=1$ implies that the applicant graduated from a prestigious university, and $J=1$ implies that the applicant got the job. Experts suggest that race may affect application results both directly and indirectly through education background, but race and education may share confounding factors such as location, as shown in Fig.~\ref{fig:ex-id}. The full relationship between variables is modeled through SCM $\cM^*$:
    \begin{equation*}
        \cM^* =
        \begin{cases}
            \*U & = \{U_{RE}, U_E, U_J\}, \text{ all binary}\\
            \*V &= \{R, E, J\} \\
            \cF &= \begin{cases}
                f_R(u_{RE}) &= u_{RE} \\
                f_E(r, u_{RE}, u_{E}) &= (\neg u_{RE} \wedge \neg r) \oplus u_E \\
                f_J(r, e, u_J) &= (\neg r \vee e) \oplus u_J
            \end{cases} \\
            P(\*U) & \text{ defined such that } U_{RE} \sim \bern(0.5), U_E, U_J \sim \bern(0.25)
        \end{cases}
    \end{equation*}
    
    Relevant quantities and their corresponding probability for each setting of $\*U$ is shown in Table \ref{tab:ex-id-truth}.
    
    \begin{table}[h]
        \begin{tabular}{l|l|l|l|l|l|l|l|l|l|l|l}
        \hline \hline
        $P(\*U)$ & $U_{RE}$ & $U_E$ & $U_J$ & $R$ & $E$ & $J$ & $E_{R=0}$ & $E_{R=1}$ & $J_{R=0}$ & $J_{R=1}$ & $J_{R=1, E_{R=0}}$ \\ \hline \hline
        $p_0 = 9/32$     & 0        & 0     & 0     & 0   & 1   & 1   & 1         & 0         & 1         & 0         & 1                  \\ \hline
        $p_1 = 3/32$     & 0        & 0     & 1     & 0   & 1   & 0   & 1         & 0         & 0         & 1         & 0                  \\ \hline
        $p_2 = 3/32$     & 0        & 1     & 0     & 0   & 0   & 1   & 0         & 1         & 1         & 1         & 0                  \\ \hline
        $p_3 = 1/32$     & 0        & 1     & 1     & 0   & 0   & 0   & 0         & 1         & 0         & 0         & 1                  \\ \hline
        $p_4 = 9/32$     & 1        & 0     & 0     & 1   & 0   & 0   & 0         & 0         & 1         & 0         & 0                  \\ \hline
        $p_5 = 3/32$     & 1        & 0     & 1     & 1   & 0   & 1   & 0         & 0         & 0         & 1         & 1                  \\ \hline
        $p_6 = 3/32$     & 1        & 1     & 0     & 1   & 1   & 1   & 1         & 1         & 1         & 1         & 1                  \\ \hline
        $p_7 = 1/32$     & 1        & 1     & 1     & 1   & 1   & 0   & 1         & 1         & 0         & 0         & 0                  \\ \hline
        \end{tabular}
        \caption{Relevant quantities calculated for each setting of $\*U$ in $\cM^*$ for Example \ref{ex:id}.}
        \label{tab:ex-id-truth}
    \end{table}
    
    We are given the causal diagram $\cG$ from Fig.~\ref{fig:ex-id} as well as the datasets $\bbZ = \{P(\*V), P(\*V_{R=0}), P(\*V_{R=1})\}$. The observational data from $P(\*V)$ was collected by simply asking random participants to list their race, education level, and application results. The experimental data from $P(\*V_{R=0})$ and $P(\*V_{R=1})$ is obtained similarly, except participants are asked to write their race as $R=0$ or $R=1$ respectively in their school and job applications regardless of their actual race. This data, computed from Table \ref{tab:ex-id-truth}, is shown in Table \ref{tab:ex-id-data}.
    
    \begin{table}[h]
        \centering
        \begin{tabular}{l|l|l|l}
        \hline \hline
        $R$ & $E$ & $J$ & $P(\*V)$ \\ \hline \hline
        0 & 0 & 0 & $p_3 = 1/32$    \\ \hline
        0 & 0 & 1 & $p_2 = 3/32$    \\ \hline
        0 & 1 & 0 & $p_1 = 3/32$    \\ \hline
        0 & 1 & 1 & $p_0 = 9/32$    \\ \hline
        1 & 0 & 0 & $p_4 = 9/32$    \\ \hline
        1 & 0 & 1 & $p_5 = 3/32$    \\ \hline
        1 & 1 & 0 & $p_7 = 1/32$    \\ \hline
        1 & 1 & 1 & $p_6 = 3/32$    \\ \hline
        \end{tabular}
        \quad
        \begin{tabular}{l|l|l|l}
        \hline \hline
        $r$ & $E_{R=r}$ & $J_{R=r}$ & $P(\*V_{R=r})$ \\ \hline \hline
        0 & 0 & 0 & $p_3 + p_5 = 4/32$    \\ \hline
        0 & 0 & 1 & $p_2 + p_4 = 12/32$    \\ \hline
        0 & 1 & 0 & $p_1 + p_7 = 4/32$    \\ \hline
        0 & 1 & 1 & $p_0 + p_6 = 12/32$    \\ \hline
        1 & 0 & 0 & $p_0 + p_4 = 18/32$    \\ \hline
        1 & 0 & 1 & $p_1 + p_5 = 6/32$    \\ \hline
        1 & 1 & 0 & $p_3 + p_7 = 2/32$    \\ \hline
        1 & 1 & 1 & $p_2 + p_6 = 6/32$    \\ \hline
        \end{tabular}
        \caption{Data for Example \ref{ex:id}}
        \label{tab:ex-id-data}
    \end{table}
    
    When looking at the average treatment effect of race on job success from the data, we see that
    \begin{equation*}
        P(J_{R=1} = 1) - P(J_{R=0} = 1) = \frac{12}{32} - \frac{24}{32} = -\frac{12}{32} = -0.375
    \end{equation*}
    which is a rather large disparity between the success rates of protected classes compared to others. To avoid accusations of discrimination, the company claims that this disparity is not a result of the company's hiring practices. The company says that they only look at education levels, and the reason why protected classes are less likely to get the job is because the fault lies with the education system for discriminating against protected classes in their admissions process. To verify this claim, we may be interested in the \emph{natural direct effect} query, formally defined in this context as
    \begin{equation}
        Q = P(J_{R=1, E_{R=0}} = 1) - P(J_{R=0} = 1).
    \end{equation}
    In words, the first term is the probability that someone will get the job if we change their race to a protected class, but change their education level as if their race was not from a protected class. The second term is simply their natural chance of getting the job if their race was fixed to a non-protected class. In other words, this query calculates the direct effect of race on the hiring process, irrespective of education. We do not have data on the first term of $Q$, but it turns out that $Q$ is identifiable from $\cG$ and $\bbZ$. Running Alg.~1 would return two NCMs that induce the same result for $Q$. For example, the first term can be derived symbolically as
    \begin{align*}
        & P(J_{R=1, E_{R=0}} = 1) \\
        &= P(J_{R=1, E=0} = 1, E_{R=0} = 0) + P(J_{R=1, E=1} = 1, E_{R=0} = 1) & \text{Unnesting} \\
        &= P(J_{R=1, E=0} = 1)P(E_{R=0} = 0) + P(J_{R=1, E=1} = 1)P(E_{R=0} = 1) & \text{C-factor decomp.} \\
        &= P(J = 1 \mid R = 1, E = 0)P(E_{R=0} = 0) & \text{Rule 2} \\
        &+ P(J = 1 \mid R = 1, E = 1)P(E_{R=0} = 1) \\
        &= \left(\frac{3}{12} \right) \left(\frac{16}{32}\right) + \left(\frac{3}{4} \right) \left(\frac{16}{32}\right) = \frac{1}{8} + \frac{3}{8} = \frac{1}{2} & \text{From data}
    \end{align*}
    From the data, we also have $P(J_{R=0} = 1) = \frac{24}{32} = \frac{3}{4}$. Hence, we can compute
    \begin{align*}
        Q = P(J_{R=1, E_{R=0}} = 1) - P(J_{R=0} = 1) = \frac{1}{2} - \frac{3}{4} = -\frac{1}{4} = -0.25,
    \end{align*}
    revealing that while the direct effect is slightly lower than the average treatment effect, the company still shares some responsibility in the discrimination of protected classes in their hiring practices. Indeed, one could construct a $\cG$-NCM that matches $\cM^*$ on the given data sets, for example:
    \begin{equation*}
        \hM =
        \begin{cases}
            \widehat{\*U} & = \{\widehat{U}_{RE}, \widehat{U}_J\} \\
            \*V &= \{R, E, J\} \\
            \widehat{\cF} &= \begin{cases}
                \hat{f}_R(\widehat{u}_{RE}) &= \hat{f}_{\geq 0.5}(\widehat{u}_{RE}) \\
                \hat{f}_E(r, \widehat{u}_{RE}) &= \hat{f}_{OR}
                \begin{cases}
                    \hat{f}_{AND} \left(\hat{f}_{<0.375}(\widehat{u}_{RE}), \hat{f}_{NOT}(r) \right) \\
                    \hat{f}_{AND} \left(\hat{f}_{<0.125}(\widehat{u}_{RE}), r \right) \\
                    \hat{f}_{\geq 0.875}(\widehat{u}_{RE})
                \end{cases} \\
                \hat{f}_J(r, e, \widehat{u}_J) &= \hat{f}_{OR}
                \begin{cases}
                    \hat{f}_{AND} \left(\hat{f}_{NOT}(r), \hat{f}_{< 0.75}(\widehat{u}_J)\right) \\
                    \hat{f}_{AND} \left(r, \hat{f}_{NOT}(e), \hat{f}_{< 0.25}(\widehat{u}_J)\right) \\
                    \hat{f}_{AND} \left(r, e, \hat{f}_{< 0.75}(\widehat{u}_J)\right) \\
                \end{cases}
            \end{cases} \\
            P(\widehat{\*U}) & \text{ defined such that } \widehat{U}_{RE}, \widehat{U}_J \sim \unif(0, 1)
        \end{cases}
    \end{equation*}
    One can verify that the quantities induced by $\hM$ match those in Table ~\ref{tab:ex-id-data}. Interestingly, the functions of $\hM$ are quite different from those of $\cM^*$, and in fact, they will likely not match in non-identifiable quantities. However, if we evaluate $\hM$ on our query of interest $Q$, we see that
    \begin{enumerate}
        \item $E_{R=0} = 1$ when $\hU_{RE} < 0.375$ or $\hU_{RE} \geq 0.875$, 0 otherwise;
        \item When $E_{R=0} = 0$, $J_{R=1, E_{R=0}} = 1$ when $\hU_{J} < 0.25$, 0 otherwise;
        \item When $E_{R=0} = 1$, $J_{R=1, E_{R=0}} = 1$ when $\hU_{J} < 0.75$, 0 otherwise.
    \end{enumerate}
    Hence, we can compute
    \begin{align*}
        & P^{\hM}(J_{R=1, E_{R=0}} = 1) \\
        &= P(0.375 \leq \hU_{RE} < 0.875)P(\hU_J < 0.25) + P(\hU_{RE} < 0.375 \vee \hU_{RE} \geq 0.875)P(\hU_J < 0.75) \\
        &= (0.5)(0.25) + (0.5)(0.75) = 0.5,
    \end{align*}
    which matches the symbolic result. Additionally, we see that $P^{\hM}(J_{R=0} = 1) = P(\hU_J < 0.75) = 0.75$, so we see that the result induced by $\hM$ for $Q$ is indeed $0.5 - 0.75 = -0.25$.
    
    Computing the true value of the query from $\cM^*$ using values from Table \ref{tab:ex-id-truth}, we see that
    \begin{align*}
        & P^{\cM^*}(J_{R=1, E_{R=0}} = 1) - P^{\cM^*}(J_{R=0} = 1) \\
        &= p_0 + p_3 + p_5 + p_6 - p_0 - p_2 - p_4 - p_6 = p_3 + p_5 - p_2 - p_4 \\
        &= \frac{1}{32} + \frac{3}{32} - \frac{3}{32} - \frac{9}{32} = -\frac{8}{32} = -0.25.
    \end{align*}
    
    Remarkably, by simply fitting the $\cG$-NCM on the data and querying it as if it were $\cM^*$ itself, we achieve the correct value of $Q$ matching both the true value from $\cM^*$ and the result from the symbolic derivation. As long as the training succeeds in matching the data, it does not matter what functions are learned, and the evaluation requires no knowledge of the symbolic approach, demonstrating the strength of using the NCM for evaluating identifiable queries.
    \hfill $\blacksquare$
\end{example}

\subsection{Discussion on Markovianity}
\label{app:markovianity}

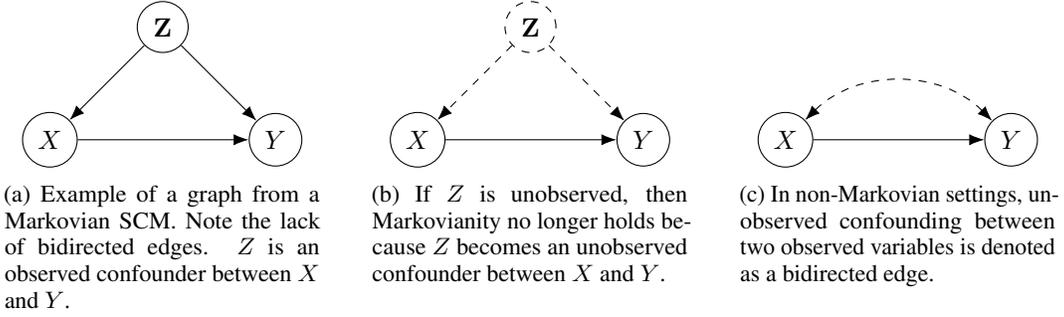
\begin{figure}
    	\centering
    	\begin{subfigure}[t]{0.3\textwidth}
    	    \centering
            \begin{tikzpicture}[xscale=1.5, yscale=1.5]
        		\node[draw, circle] (X) at (-1, 0) {$X$};
        		\node[draw, circle] (Z) at (0, 1) {$\mathbf{Z}$};
        		\node[draw, circle] (Y) at (1, 0) {$Y$};
        		
        		\path [-{Latex[length=2mm]}] (X) edge (Y);
        		\path [-{Latex[length=2mm]}] (Z) edge (X);
        		\path [-{Latex[length=2mm]}] (Z) edge (Y);
        	\end{tikzpicture}
        	\caption{Example of a graph from a Markovian SCM. Note the lack of bidirected edges. $Z$ is an observed confounder between $X$ and $Y$.}
        	\label{fig:bd-markov-graph}
    	\end{subfigure}
    	\hfill
    	\begin{subfigure}[t]{0.3\textwidth}
    	    \centering
            \begin{tikzpicture}[xscale=1.5, yscale=1.5]
        		\node[draw, circle] (X) at (-1, 0) {$X$};
        		\node[draw, circle, dashed] (Z) at (0, 1) {$\mathbf{Z}$};
        		\node[draw, circle] (Y) at (1, 0) {$Y$};
        		
        		\path [-{Latex[length=2mm]}] (X) edge (Y);
        		\path [-{Latex[length=2mm]}, dashed] (Z) edge (X);
        		\path [-{Latex[length=2mm]}, dashed] (Z) edge (Y);
        	\end{tikzpicture}
        	\caption{If $Z$ is unobserved, then Markovianity no longer holds because $Z$ becomes an unobserved confounder between $X$ and $Y$.}
        	\label{fig:bd-unobserved}
    	\end{subfigure}
    	\hfill
    	\begin{subfigure}[t]{0.3\textwidth}
    	    \centering
            \begin{tikzpicture}[xscale=1.5, yscale=1.5]
        		\node[draw, circle] (X) at (-1, 0) {$X$};
        		\node[draw, circle] (Y) at (1, 0) {$Y$};
        		
        		\path [-{Latex[length=2mm]}] (X) edge (Y);
        		\path [{Latex[length=2mm]}-{Latex[length=2mm]}, dashed, bend left=45] (X) edge (Y);
        	\end{tikzpicture}
        	\caption{In non-Markovian settings, unobserved confounding between two observed variables is denoted as a bidirected edge.}
        	\label{fig:bow-non-markov}
    	\end{subfigure}
    	\caption{Graphs for Markovianity discussion.}
    	\label{fig:markov-graphs}
    \end{figure}

One of the major contributions emphasized in this work is the generality of the approach. In particular, Alg.~\ref{alg:ncm-solve-ctfid} solves the counterfactual identification and estimation problems for any counterfactual query with any given input in any causal graph setting. Notably, we do not make the \emph{Markovianity} assumption about the generating SCM. 
Markovianity restricts the set of unobserved variables $\*U$, requiring them to be independent and having each $U \in \*U$ as an input to at most one function in $\cF$. This translates to the lack of bidirected edges in any causal diagram generated by a Markovian SCM (Fig.~\ref{fig:bd-markov-graph}, for example). In other words, since no unobserved variable can be a common cause of two observed variables, Markovianity encodes the fact that there is no unobserved confounding.

Markovianity is frequently assumed in many works (including several related works in Appendix \ref{app:related-works}) because it drastically simplifies the problem. When attempting to infer $\cL_2$ queries from $\cL_1$ data and the causal diagram $\cG$, it turns out that any interventional query $P(\*y_{\*x})$ can trivially be identified as
\begin{equation}
    \label{eq:markov-l2}
    P(\*y_{\*x}) = \sum_{\*z} P(\*y \mid \*z, \*x)P(\*z),
\end{equation}
where $\*Z$ is the set of non-descendants of $\*X$; for details, see \citep[Corol.~2]{bareinboim:etal20}.

Several observations follow from this observation. First, this implies that works that attempt to infer $\cL_2$ queries from $\cL_1$ data under the Markovianity assumption do not even need to consider the identification problem, as they are all identifiable given the graph. \citet{kocaoglu2018causalgan} is an example of a work that can directly estimate the query without having to be concerned about the identification problem due to the Markovianity assumption. 

Second, when attempting to estimate the $\cL_2$ query, the estimand in Eq.~\ref{eq:markov-l2} is the same for every query, so only one form of the query needs to be considered when designing an estimator. For example, although not explicitly assuming Markovianity, the works in Footnote \ref{ft:bd-models} assume conditional ignorability, which implies the lack of unobserved confounding between the treatment $X$ and the outcome $Y$. This allows the causal effect $P(\*y_{\*x})$ to be identified as stated in Eq.~\ref{eq:markov-l2}, so these works can focus directly on estimating the specific equation in Eq.~\ref{eq:markov-l2}. On the counterfactual level, the identification problem may still need to be considered as there are non-identifiable $\cL_3$ queries even in the Markovian case. Still, the problem is drastically simplified as confounding is no longer an issue. In fact, works like \citet{pawlowski2020deep} and \citet{sanchezmartin2022} rely on this simplification for counterfactual estimation in the Markovian case and do not consider the identification problem despite potential the lack of identifiability of various queries.

Third, another notable simplification introduced by Markovianity is that distributions from the SCM can often be factorized into a convenient form. For example, the observational distribution $P(\*v)$ can be factorized as the product: 
\begin{equation*}
    P(\*v) = \prod_{V \in \*V} P(v \mid \pai{V}),
\end{equation*}
where $\Pai{V}$ refer to the set of variables with an edge into $V$ in $\cG$, and $v$ and $\pai{V}$ are consistent with $\*v$. Similarly, an interventional distribution $P(\*v_{\*x})$ can be factorized as
\begin{equation*}
    P(\*v_{\*x}) = \prod_{V \in \*V \setminus \*X} P(v \mid \pai{V}),
\end{equation*}
where $v$ and $\pai{V}$ are consistent with $\*v$ and $\*x$. Counterfactual distributions follow a similar pattern. For example, when considering the graph $\cG$ in Fig.~\ref{fig:bd-markov-graph}, $P(y_x, x'_z)$ factorizes as
\begin{equation*}
    P(y_x, x'_z) = \sum_{z'} P(y \mid x, z')P(x' \mid z)P(z'), 
\end{equation*}
where $x \neq x'$. The pattern is that in every case, each term in the estimand can be written as the probability of a variable given its parents. This allows for models to encode the mechanism for each variable $V$ as simply the conditional distribution $P(V \mid \Pai{V})$, abstracting away the functions $\cF$ and exogenous noise $P(\*U)$ from the SCM. Architectures that benefit from this include DeepSCM from \citet{pawlowski2020deep}, which utilizes normalizing flows to model each conditional probability $P(V \mid \Pai{V})$, and VACA from \citet{sanchezmartin2022}, which fits a variational graph autoencoder to the factorized distributions. In contrast, in non-Markovian cases, these factorizations typically do not hold. For example, if the true SCM induces the graph in Fig.~\ref{fig:bow-non-markov}, then $P(y_x) \neq P(y \mid x)$ despite $X$ being the only observed parent of $Y$. Works that rely on Markovian factorizations often do not have a straightforward generalization to non-Markovian cases.

Fourth, the works cited in this section that leverage the Markovianity assumption have still provided important contributions to the understanding and practical approaches to causal and counterfactual inference, and they are widely praised for their successes. However, Markovianity is a very restricting assumption and is unrealistic in many practical settings. For example, consider a simple study on the effects of exercise $X$ on cholesterol $Y$, with age $Z$ as a confounder. The relationship between variables can be modeled as shown in Fig.~\ref{fig:bd-markov-graph}. If, in the study, data on age is simply not collected, then $Z$ becomes an unobserved confounder (see Fig.~\ref{fig:bd-unobserved}), and Markovianity is immediately violated. The corresponding graph would look like Fig.~\ref{fig:bow-non-markov} with $Z$ omitted. In many cases, collecting the data may not even be an option. For example, if $X$ is smoking, $Y$ is cancer, and $Z$ is the expressiveness of a gene, perhaps collecting data on gene expressivity is too difficult or even impossible. In general, one cannot expect to be able to collect data on every relevant variable in the study, so unobserved confounding is to be expected in realistic settings.

The solution presented in this paper is therefore important to understand the more complex, realistic, and general setting of potentially having unobserved confounding. Since the lack of the Markovianity assumption greatly complicates the problem of counterfactual inference, many of the theoretical results in this work (e.g. Thms.~\ref{thm:l3-g-expressiveness} and \ref{thm:ncm-ctfid-equivalence}) required more involved proofs. The result is that using Alg.~\ref{alg:ncm-solve-ctfid}, $\cG$-NCMs can solve any instance of counterfactual identification and estimation even under unobserved confounding. This is demonstrated in the experiments, which consider several settings in which Markovianity is violated (e.g. the second, third, and fourth rows in Fig.~\ref{fig:gan-id-expl-results}, where the unobserved confounding is shown visually in the causal diagrams as red dashed bidirected arrows).

%% file: section/E_FAQ.tex
\section{Frequently Asked Questions}

\begin{enumerate}[label=Q\arabic*.]
    \item Why do you assume that the causal diagram exists instead of learning it? Is it reasonable to expect that the causal diagram is available?
    
    \textbf{Answer}. For a detailed discussion on this topic, see Sec.~\ref{sec:cg-assumption}. In summary, the assumption of the causal diagram is made out of necessity. The Neural Causal Hierarchy Theorem (N-CHT) \citep[Corol.~1]{xia:etal21} states that distributions from lower layers almost always underdetermine higher layers in a measure-theoretic sense \citep[Thm.~1]{bareinboim:etal20}. This means that despite its universal expressivity, an NCM trained only on lower layer data (e.g., observational) cannot be reliably used for inferences on higher layers (e.g., interventional or counterfactual). Therefore, assumptions are necessary to make inferences across the PCH's layers. The causal diagram is a flexible assumption that is used throughout the literature and encodes a qualitative description of the generating model, which is often much easier to obtain than the actual mechanisms of the generating SCM \citep{pearl:2k, spirtes:etal00,peters:17}. Given the N-CHT, any attempts to learn the causal diagram will be incorrect or not meaningful unless other assumptions are made in this process (e.g., faithfulness). While other types of assumptions can be made, all such assumptions have their own benefits and drawbacks. The goal of this paper is not to decide which set of assumptions is the best, but rather to provide the toolkit for data analysts to perform the inferences once the assumptions have already been made.
    \\
    
    \item How could the graphical assumption be relaxed?
    
    \textbf{Answer}. Relaxing the graphical assumption is out of the scope of this work since the goal of this paper is first to establish a principled approach of using neural models to perform counterfactual inferences under general well-understood assumptions. Still, finding a solution to the problem when only partial information about the graph is available is an interesting direction for future work. One such possible approach, discussed at the end of Sec.~\ref{sec:cg-assumption}, is to utilize equivalence classes of causal diagrams, learnable from data, to perform inferences \citep{jaber2018causal,jaber2019identification,jaber2022calc}. However, this approach has its own drawbacks, notably, that equivalence classes are not as descriptive as the true causal diagram, and therefore fewer queries will be identifiable. Still, this is one approach that can be considered in cases where assuming the availability of the full causal diagram is too stringent of a condition.
    \\
    
    \item Why do you assume non-Markovianity? Is this more general than solutions that assume Markovianity?
    
    \textbf{Answer}.
    That is an interesting question. We do not assume ``non-Markovianity'', but rather we do not make any assumption on Markovianity at all. Typically, Markovianity is the assumption and not the other way around. Markovianity requires that there exists no unobserved confounding, but works that do not assume Markovianity do not enforce such a requirement and are therefore more general because they work in both Markovian and non-Markovian cases. As mentioned earlier in the manuscript, it is hard to ascertain Markovianity in many practical settings, so not making the assumption is more reasonable in many natural settings. 
    Unlike many related works like those in Appendix \ref{app:related-works}, this paper considers cases with unobserved confounding, which is represented graphically through dashed bidirected edges. This approach is more general in the sense that Alg.~\ref{alg:ncm-solve-ctfid} can be used even in the presence of unobserved confounders, which is one of the major contributions of this paper. Markovianity is often assumed to simplify the problem setting, but it is often unrealistic to expect that data on all potential confounding variables are collected in a study. For a detailed discussion on this topic, see Appendix \ref{app:markovianity}.
    \\
    
    \item Why do you use the $\mathcal{G}$-NCM definition in Def.~\ref{def:gncm} with feedforward neural networks and uniform noise? Could the results in this paper also hold for other function classes and noise?
    
    \textbf{Answer}. The architectural choices made in Def.~\ref{def:gncm} are standard choices used to ground the discussion to a more concrete model. The same architecture is also used in the experiments. Still, the specific choice of architecture is not important to the theoretical results as long as key properties persist (e.g. $\cG^{(\cL_3)}$-consistency, $\cL_3$-$\cG$ expressiveness), which will often be the case, since the arguments to prove the theoretical results in this work are mostly independent of the architecture. In particular, Sec.~D.1 of \citet{xia:etal21}, which introduces the $\cG$-NCM, discusses generalizations of NCMs in terms of functions and noise which still guarantee the relevant theoretical results, so long as the functions are sufficiently expressive and the noise variables have well-defined probability density functions.
    \\
    
    \item Given that neural networks are universal approximators, it seems unsurprising that $\cG$-NCMs can perform counterfactual inference. What makes $\cG$-NCMs so special compared to other model classes?
    
    \textbf{Answer}. The structure of this paper follows the delicate balance between expressiveness and constraint. As emphasized by the Causal Hierarchy Theorem \citep[Thm.~1]{bareinboim:etal20} and \citep[Corol.~1]{xia:etal21} , expressivity alone is not enough to infer higher layer quantities from lower layer data. Without assumptions, any model, including NCMs, trained only on $\cL_1$ or $\cL_2$ data will not contain any meaningful information about $\cL_3$ quantities. Hence, assumptions in the form of constraints are discussed in Sec.~\ref{sec:data-structure}. One important question is how to properly incorporate such constraints into a model. Constraining a model too much may hinder its expressiveness, and it may no longer be able to express the true model. Therefore, a proper model class for counterfactual inference must be able to incorporate the constraints without losing expressivity within the constrained class.
    
    The $\cG$-NCM model class has both the critical properties of expressiveness and constraint (Thms.~\ref{thm:gl3-consistency} and \ref{thm:l3-g-expressiveness}) that are not enjoyed by many model classes. Examples of these properties are shown in Appendix \ref{app:examples-properties}. Example \ref{ex:l3-expression-wrapper} shows an example of a model class that, although expressive on all three layers, fails to properly incorporate the constraints of the graph, $\cG$. Examples \ref{ex:cbn-expressive} and \ref{ex:markov-expressive} show examples of model classes that incorporate the constraints of $\cG$ but fail to achieve expressiveness within the constrained class. In contrast to these examples, the $\cG$-NCM model class has a natural and intuitive approach to incorporating the constraints of $\cG$ without damaging its expressiveness. These properties empower the $\cG$-NCM the capability of solving the problems of counterfactual identification and estimation (see Thm.~\ref{thm:ncm-ctfid-equivalence}, Corols.~\ref{cor:op-id} and \ref{thm:ncm-ctfid-correctness}, and Alg.~\ref{alg:ncm-solve-ctfid}). Examples \ref{ex:non-id} and \ref{ex:id} concretely demonstrate this ability.
    
    When compared to similar works, like the ones in Appendix \ref{app:related-works}, these properties are not as clearly established. While many of the works do incorporate graphical constraints, these constraints are often taken for granted as they are not formally discussed. In contrast, the ideas of graphical constraints are developed formally in this work in Sec.~\ref{sec:data-structure} and Appendix \ref{app:proof-G-cons}. Many models in the related works assume Markovianity, which fails in terms of expressiveness for the same reason as Example \ref{ex:markov-expressive}. Some of the works only focus on the interventional ($\cL_2$) level rather than the counterfactual ($\cL_3$) level.
    \\
    
    \item Why would one prefer to identify counterfactual quantities using NCMs? Would it not be easier first to identify the query using symbolic approaches, then use the NCM for estimation?
    
    \textbf{Answer}. For a detailed discussion on the difference between symbolic and optimization approaches, we encourage the reader to read Appendix C.4 of \citet{xia:etal21}. In summary, the approach in this paper utilizes $\cG$-NCMs to serve as a proxy model for the true generating SCM. In contrast, existing symbolic approaches like \citet{correa:etal21} work directly with the graph and data, which is on a level of abstraction without the generating SCM. These works typically derive an exact answer for the target query by directly using the constraints of the graph. Indeed, one could use a hybrid approach of using a symbolic algorithm to identify a counterfactual query followed by training a $\cG$-NCM to obtain an estimate. This is certainly an option, and the goal of this paper is to provide another solution route, not to decide for the data analyst which approach is better.
    
    The proxy model approach presented in this paper is insightful to the problem of counterfactual inference because it shows that $\cG$-NCMs are capable of jointly solving both counterfactual identification and estimation, which typically appear together in practice. Solving both problems together indicates that the $\cG$-NCM is capable of both reasoning whether the query has any counterfactual meaning through identification and providing the practical benefits of statistical modeling through estimation. There are many reasons why one may prefer a proxy-based approach. First, the generative properties of the learned proxy model allow it to serve as an infinite source of samples for the user. Second, the trained model could provide quick estimation results for multiple queries without the need for retraining, provided the queries are identifiable. Third, the idea of minimizing and maximizing the query in Alg.~\ref{alg:ncm-solve-ctfid} is central to identification problems. Hence, a similar procedure can be used in similar identification problem settings without an existing symbolic solution. Fourth, this idea can potentially be used to bound non-identifiable queries, which is an exciting avenue for future work.
    \\
    
    \item Why would one want to learn an entire SCM to solve the estimation problem for a specific query instead of just directly estimating the query?
    
    \textbf{Answer}.
    Indeed, if the user simply wants to estimate one query and is knowledgeable about the estimand, it may be more efficient, statistically or computationally, to focus efforts on just estimating that query. However, there are many reasons why the proxy model approach used in this work may be beneficial. First, by learning an entire SCM and estimating queries using Eq.~\ref{eq:ncm-mc-sampling}, one can apply the same procedure to any potential query, reducing the complexity of the solution. This relieves the user of the need to derive an estimand. Even if the estimand is known, developing a specific approach to estimating this estimand can be difficult depending on the form of the estimand. Second, if the user wants to estimate multiple queries, they could potentially reuse the same $\cG$-NCM without retraining, provided that the queries are identifiable. Third, as a generative model, the $\cG$-NCM can be used as an infinite source of samples should the user desire this in addition to having an estimate for the query.
    \\
    
    \item Why is the GAN architecture used for the experiments?
    
    \textbf{Answer}. 
    We found that using a GAN approach was the most straightforward way to implement the $\cG$-NCM in practice. When using an implicit generative model approach, one can construct a $\cG$-NCM exactly as described in Def.~\ref{def:gncm} and train it as specified in Alg.~\ref{alg:ncm-learn-pv}. Using a discriminator network as the divergence function ($\bbD_P$) provides a flexible and practical approach to generating a loss for optimization, as understood in GAN literature. Alternative approaches to modeling the divergence functions all have practical drawbacks. For example, KL divergence is typically intractable to compute directly. Appendix \ref{app:explicit-ncm} explains why an explicit generative model approach like the MLE-NCM (used in \citet{xia:etal21}) is significantly more complicated. We also include experiments comparing against the MLE-NCM, which demonstrated some key flaws of the MLE approach, such as poor runtime scalability (see Fig.~\ref{fig:runtime-results}) or having a bias to output ``ID'' even in non-ID settings (see Fig.~\ref{fig:gan-id-expl-results}). For further results, see Appendix \ref{app:more-experiments}. Still, the key theoretical results in this work do not depend on the specific architecture used in practice, and the choice of using GANs is unimportant in this discussion. Users can use any architecture they see fit while still applying the results in this work.
    \\
    
    \item Given the training instability of GANs, is the GAN-NCM approach guaranteed to work?
    
    \textbf{Answer}. Indeed, the theoretical results described in Sec.~\ref{sec:ncm-ctf-id} assume that infinite data is provided and optimization is perfect, which is unrealistic in practice. Hence, a hypothesis testing procedure is required to decide identifiability, described in Appendix \ref{app:hyp-test}. This means that the results from solving the identification and estimation problems will have some degree of error in practice. While deriving bounds on the error could be an interesting direction for future work, the error is shown to be quite manageable empirically using GANs, as shown in the experiments in Sec.~\ref{sec:experiments}. In the estimation experiments, the error is shown to decrease as samples increase, demonstrating a level of statistical consistency.
    \\

    \item Does the work apply to cases with continuous variables?
    
    \textbf{Answer}. The theoretical results in this work assume that the variable domains are finite, and no claims are made about continuous variables. If extending to the continuous case, on the theoretical side, counterfactual distributions and queries would have to be properly defined and new machinery based on measure theory evoked. It is non-trivial to prove some of the theoretical results of this paper in continuous settings. 
    While the exogenous variables and structural functions of the true model can take any form, this paper relies on the discreteness of the endogenous variables to ensure that there exists a $\cG$-NCM that can exactly match the true model (see also \citet{zhang:bareinboim21b}). It is currently an open problem to generalize these results to the continuous case and in generality.
    
    On the experimental side, sampling or maximizing/minimizing the query from a $\cG$-NCM may have complications depending on the definition of a counterfactual query/distribution in a continuous setting. Intuitively, however, if the query can be sampled via Alg.~\ref{alg:ncm-ctf-sample}, then similar techniques applied in our paper can be extended without much change (although currently without theoretical guarantees). For example, if the query is $E[Y_{X=x} | X \geq 0]$ with both continuous $X$ and $Y$, one could  (1) sample several instances of $X$, (2) filter out negative samples, (3) sample $Y$ using the values of $\mathbf{U}$ from the remaining instances, and (4) compute the average out of these values of $Y$. Then, this Monte Carlo estimate can be used in Alg.~\ref{alg:ncm-learn-pv} for maximizing/minimizing the query and estimating the final result if identifiable. If the query is not so easily compatible with Alg.~\ref{alg:ncm-ctf-sample}, then the approach may need to be adjusted to accommodate continuous definitions.
    \\

    \item Why are there no empirical studies comparing other neural estimation methods (such as the ones in Appendix \ref{app:related-works})?
    
    \textbf{Answer}. Despite their notable contributions, as discussed earlier, related works like the one in Appendix \ref{app:related-works} are not as general as the approach described in this work. For example, some works assume Markovianity, some only solve queries from $\cL_2$ or specific queries from $\cL_3$, and some do not consider identifiability (see Appendix \ref{app:related-works} for a detailed discussion on each of these works). The works in Footnote \ref{ft:bd-models} assume conditional ignorability or backdoor admissibility, which only hold in a few graphs, such as graph (a) from Fig.~\ref{fig:id-r80} (it does not hold in any other graph in Fig.~\ref{fig:id-r80}). The approach presented in this paper is the first work to use neural models to identify any $\cL_3$ query given any combination of $\cL_1$ and $\cL_2$ data and any graph setting (even non-Markovian cases). It will also perform the estimation for identifiable cases, but it will return ``FAIL'' instead of estimating non-identifiable cases. Hence, comparisons with these specific related works are not appropriate since they cannot handle such general settings.
    \\

    \item How does the proposed approach scale to higher dimensional variables or with a greater number of variables?
    
    \textbf{Answer}. The training of the GAN-NCM scales tractably with higher dimensionality or variables, as shown in Fig.~\ref{fig:runtime-results}. While having higher dimensions or more variables may require a larger architecture, and therefore more training time, typically this also comes with a higher data requirement. Hence, the architecture can be scaled with the complexity of the problem, and although complex problems will require more training time, the user may be willing to spend more computational resources on a more complex problem. We hope this work can be a baseline and act as a catalyst to the study of different neural architectures with different computational and statistical properties. 

\end{enumerate}